\documentclass[opre,nonblindrev]{informs3}

\usepackage{algorithmicx}
\usepackage{algpseudocode}
\OneAndAHalfSpacedXI
\usepackage{endnotes}
\let\footnote=\endnote
\let\enotesize=\normalsize
\def\notesname{Endnotes}%
\def\makeenmark{\hbox to1.275em{\theenmark.\enskip\hss}}
\def\enoteformat{\rightskip0pt\leftskip0pt\parindent=1.275em
  \leavevmode\llap{\makeenmark}}

\usepackage{natbib}
\usepackage{multibib}
\newcites{EC}{References in Appendices} 

\usepackage{algorithm}
\usepackage{bm}
\usepackage{graphicx}
\usepackage{enumerate}
\usepackage{hyperref}
\usepackage{subcaption}
\usepackage{mathrsfs}
\usepackage{diagbox}
   \usepackage{booktabs} 
 \usepackage{multirow} 
 \usepackage[flushleft]{threeparttable}
\usepackage{threeparttable} 
\usepackage{caption}
\usepackage{natbib}
 \usepackage{multibib}
\newcites{EC}{References in Appendix}
 \bibpunct[, ]{(}{)}{,}{a}{}{,}%
 \def\bibfont{\small}%
\TheoremsNumberedThrough     
\ECRepeatTheorems
\EquationsNumberedThrough    

\begin{document}




\TITLE{Nonparametric Bayesian Optimization for General Rewards}

\ARTICLEAUTHORS{%
\AUTHOR{Zishi Zhang}

\AFF{Guanghua School of Management, Peking University,  zishizhang@stu.pku.edu.cn} 
\AUTHOR{Tao Ren}
\AFF{Guanghua School of Management, Peking University} 
\AUTHOR{Yijie Peng}
\AFF{Guanghua School of Management, Peking University,  pengyijie@pku.edu.cn} 
} 

\ABSTRACT{%
 This work focuses on Bayesian optimization (BO) under reward model uncertainty. 
We propose the first BO algorithm that achieves \textit{no-regret} guarantee in a general reward setting, requiring only Lipschitz continuity of the objective function and accommodating a broad class of measurement noise.
The core of our approach is a novel surrogate model, termed as infinite Gaussian process ($\infty$-GP).
It is a Bayesian nonparametric model that places a prior on the space of reward distributions, enabling it to represent a substantially broader class of reward models than classical Gaussian process (GP).
The $\infty$-GP is used in combination with Thompson Sampling (TS) to enable effective exploration and exploitation. Correspondingly, we develop a new TS regret analysis framework for general rewards, which relates the regret to the total variation distance between the surrogate model and the true reward distribution. 
Furthermore, with a truncated Gibbs sampling procedure, our method is computationally scalable, incurring minimal additional memory and computational complexities
compared to classical GP. Empirical results demonstrate state-of-the-art performance, particularly in settings with non-stationary, heavy-tailed, or other ill-conditioned rewards.
}

\maketitle

\section{Introduction}

Bayesian Optimization (BO) \citep{frazier2018tutorial} is a powerful framework for optimizing expensive-to-evaluate black-box functions and has found broad applications in hyperparameter tuning for machine learning models \citep{snoek2012practical}, robotics \citep{berkenkamp2023bayesian}, biology \citep{amin2024bayesian}, and reinforcement learning \citep{brochu2010tutorial}. To minimize the number of costly function evaluations, BO maintains a probabilistic model, known as a \emph{surrogate model}, over the unknown objective. This surrogate is updated after each evaluation and serves as a cheap proxy for the reward. Based on the surrogate, an \emph{acquisition policy} is used to determine the next point to evaluate. The policy aims to balance \emph{exploration} and \emph{exploitation}, enabling efficient search for the global optimum. Common choices include Expected Improvement (EI), Upper Confidence Bound (UCB), Knowledge Gradient (KG), and Thompson Sampling (TS).

The quality of the surrogate model plays a central role in the optimization process. In most BO methods, Gaussian processes (GPs) and their variants serve as the default surrogate models. GPs offer two key advantages: first, they admit easy-to-compute closed-form posterior updates; second, they provide not only point predictions but also \textit{value uncertainty} quantification, i.e., how confident we are in the predicted value, which is essential for balancing exploration and exploitation.

However, the success of GP surrogate model hinges on restrictive and often unrealistic assumptions, both in practice and in theory. Suppose the observed reward is composed of a deterministic objective $\mu^\ast(x)$ and a noise term $\epsilon^\ast(x)$. For the deterministic part, GP-based BO methods implicitly require the true objective to ``resemble" a GP. Specifically, $\mu^\ast(x)$ is often assumed either to be a sample path drawn from a GP \citep{russo2014learning, wang2025convergence}, or to lie in a reproducing kernel Hilbert space (RKHS) associated with a kernel $k(x,x')$ and have a known RHKS norm bound \citep{srinivas2010gaussian, srinivas2012information, chowdhury2017kernelized}, referred to as the \textit{RKHS assumption} in this paper. The connection between GP regression and RKHS is formally established in \cite{kanagawa2018gaussian}, which shows that the \textit{RKHS assumption} is closely related to assuming that the objective function is a GP sample path. Additionally, the aforementioned RKHS norm bound is assumed to be known in the literature, although it is difficult to compute in practice, as noted by \cite{lederer2019uniform}. 
For the noise term $\epsilon^\ast(x)$, it is commonly assumed to be independent and (sub-)Gaussian, or even noiseless in some formulations \citep{chen2023pseudo}.

These assumptions are often difficult to verify and are frequently violated in practical applications, under which GP-based BO methods may fail.
For instance, \citet{snoek2014input} point out that GP surrogates often struggle with non-stationary objectives $\mu^\ast$, and this is a widely recognized challenge in applying BO to hyperparameter tuning for machine learning models, where the objective function may exhibit varying smoothness, frequency, or scale across different regions of the input space.
GPs are also inadequate for modeling heavy-tailed or heteroscedastic noise, which frequently arises in financial applications and other real-world BO tasks \citep{ray2019bayesian, cakmak2020bayesian}.

Our objective is to pursue an efficient BO algorithm under unknown reward distributions. To this end, inspired by the Bayesian nonparametric literature,
we introduce a novel fully nonparametric surrogate model, the \emph{\( \infty \)-Gaussian Process} (\( \infty \)-GP), which models the reward function as an infinite mixture of GPs, constructed via a sequential spatial Dirichlet process mixture. Our model treats the reward distribution itself as random and places a prior over the space of distributions, enabling the model to capture \emph{reward model uncertainty}—that is, uncertainty over the distributional form of the reward—rather than merely the \textit{value uncertainty} within a fixed distribution class as in classical GP model.
When combined with TS as the acquisition policy, our method achieves \emph{no-regret} guarantees (i.e., the cumulative regret grows sublinearly) for a much broader class of reward distributions than GP-based methods. The main computational bottleneck of our algorithm lies in a Markov chain Monte Carlo (MCMC) step. To address this issue, we design a tailored truncated Gibbs sampler. As a result, although our approach involves an infinite mixture of GPs, it incurs only minimal additional computational overhead compared to standard GP, increasing the complexity from $\mathcal{O}(n^3)$ to $\mathcal{O}(n^3  \log n)$.

\vspace{-5pt}
\begin{table}[htbp]
\centering
\scriptsize
\setlength{\tabcolsep}{3pt}
\renewcommand{\arraystretch}{1.1}
\begin{threeparttable}
\caption{Comparison of modeling assumptions in no-regret BO methods}
\label{tab:bo-assumptions}
\begin{tabular}{p{4.2cm}p{4.2cm}p{4.2cm}p{3.9cm}}
\toprule
\textbf{Paper} & \textbf{Objective Function $\mu^\ast(x)$} & \textbf{Noise $\epsilon^\ast(x)$} & \textbf{Acquisition Policy} \\
\midrule
\cite{srinivas2010gaussian,srinivas2012information,UCBnips_bound} & RKHS Ass. or Sample from GP + Lip. Con. & i.i.d. Gaussian (known var.) & UCB \\
\cite{russo2014learning} & Sample from GP+ Lip. Con. & i.i.d. Gaussian (known var.) & TS (Bayes regret) \\
\cite{chowdhury2017kernelized} & RKHS Ass. & i.i.d. sub-Gaussian & UCB, TS \\
\cite{ray2019bayesian} & RKHS Ass.& i.i.d. heavy-tailed & Trunc. UCB \\
\cite{chen2023pseudo} & Continuous\tnote{1} & Noiseless & General \\
\cite{wang2025convergence} & Sample from GP + Lip. Con. & i.i.d. Gaussian & EI \\
\textbf{\normalsize This work}
     & \textbf{\normalsize Arbitrary (Lip. Con.)\tnote{2}} & \textbf{\normalsize Arbitrary}\tnote{3} & \textbf{\normalsize TS} \\
\bottomrule
\end{tabular}
\begin{tablenotes}
\scriptsize
\item[1] Requires additional assumptions on the interaction between $\mu^\ast$ and the acquisition function.
\item[2] For posterior consistency, $\mu^\ast(x)$ only needs to be continuous; Lipschitz continuity is additionally required to establish no-regret.
\item[3] Requires mild tail assumptions and smoothness assumptions
\end{tablenotes}
\end{threeparttable}
\vspace{-8pt}
\end{table}

To highlight the theoretical advancement of our work, Table~\ref{tab:bo-assumptions} summarizes the acquisition policies and imposed assumptions in several representative BO studies that establish provable \emph{no-regret} guarantees or consistency. It is worth noting that the extensive literature on multi-armed bandits (MAB), as well as the related ranking and selection (R\&S) problems, is not included here, despite their close connection to BO. As pointed out by \cite{chen2023pseudo}, the discrete or linear structures of MAB are fundamentally different from the continuous, black-box optimization setting of BO. As summarized in Table~\ref{tab:bo-assumptions}, a substantial body of theoretical results exists for UCB-type and EI acquisition policies in BO. These methods typically rely on the \emph{RKHS assumption} and (sub-)Gaussian noise, which together allow the use of information-theoretic concentration inequalities to establish no-regret guarantees \citep{srinivas2010gaussian, srinivas2012information, chowdhury2017kernelized, EI1, EI2}.
Theoretical results for the regret of TS acquisition policy remain relatively limited in the BO literature. For the frequentist setting, most analyses build on the seminal work of \cite{chowdhury2017kernelized}, which partitions the decision space into ``saturated" and ``unsaturated" regions. The probability of sampling from a saturated region is shown to be small, while the regret incurred in unsaturated regions is bounded through the posterior variance. However, this approach still fundamentally relies on RKHS assumptions to control the posterior variance via information gain.
For the Bayesian setting, most existing results follow the regret decomposition introduced by \cite{russo2014learning}, which relates the Bayesian regret of TS to that of UCB-type methods. As a result, these analyses inherit the restrictive assumptions required by UCB-type approaches.

In contrast, our work imposes the weakest assumptions in the literature, both on the deterministic component $\mu^\ast(x)$ and the noise term $\epsilon^\ast(x)$. For the objective function $\{\mu^\ast(x): x \in \mathcal{X}\}$, we assume only Lipschitz continuity, which is strictly weaker than the commonly adopted RKHS assumption (since RKHS assumption itself implies Lipschitz continuity; see \cite{de2012exponential}). For the noise, our method accommodates heavy-tailed and heteroscedastic distributions.
We note that while Table~\ref{tab:bo-assumptions} uses the term “arbitrary’’ noise for simplicity, our theory still requires mild regularity conditions on noise tails and moments.
Our analysis builds on a new regret framework for TS that accommodates a general class of reward distributions, and the resulting proof technique is intuitive. Because TS selects actions according to the posterior probability of being optimal under the surrogate model, the regret can be controlled directly through the discrepancy between the surrogate posterior and the true reward distribution. We quantify this discrepancy using the total variation distance. Since the $\infty$-GP can converge to a much broader class of reward distributions, this discrepancy vanishes under mild conditions, leading to no-regret performance.


\subsection{Related Works}
Surrogate modeling, particularly using GPs, plays a central role not only in BO but also in the simulation optimization (SO) literature \citep{kriging_hong}. GPs were initially employed to approximate response surfaces in stochastic simulations, a methodology commonly referred to as stochastic kriging \citep{kriging_ankenman}. Since then, GP-based models have been widely applied to other areas of SO, including random search \citep{wang2025gaussian, sun2014balancing} and R\&S problems \citep{cakmak2024contextual}. To address large-scale SO settings, various GP model variants have been developed, such as those in \cite{l2019gaussian, kriging_ding}.

Related to our work, \citet{bogunovic2021misspecified} also investigate the issue of model misspecification in BO, but their analysis assumes that the true objective function can be uniformly approximated by a well-behaved function that satisfies the RKHS assumption.
While many studies in BO attempt to relax the restrictive assumptions of GP models, each typically addresses a single limitation, such as non-stationarity or heavy-tailed noise.
For instance,
to address non-stationarity issue, prior works have proposed various modifications to the classical GP models. \cite{snoek2014input} warp the input space of standard GP kernels to model non-stationary functions. \cite{li2023study} replace the GP surrogate model with neural networks. \cite{seeger2004gaussian} and \cite{higdon2022non} employ non-stationary kernels in GP model.
To handle heavy-tailed noise, a substantial body of work has been developed in the context of MAB \citep{bubeck2013bandits,medina2016no}. However, for BO with a continuous decision space, available methods are far more limited. A notable exception is \cite{ray2019bayesian}, which proposes a truncation-based modification to the GP-UCB algorithm to handle heavy-tailed noises. 
Finally, similar to our $\infty$-GP model, mixture-of-experts GP models \citep{rasmussen2001infinite, meeds2005alternative} also construct GP mixtures using Dirichlet processes (DPs). However, they rely on a global DP to partition the input space, assigning each input to a single GP expert. In contrast, our model employs a sequential spatial DP prior, allowing each location to mix over infinitely many GP surfaces. This leads to a fundamentally different formulation that enables convergence over a broad class of reward distributions.

The DP is a foundational tool in the Bayesian nonparametric literature, which places a prior over the space of probability measures. A draw from a DP, denoted \( G \sim \text{DP}(\nu G_0) \), admits the  representation:
\(
G = \sum_{l=1}^\infty w_l \delta_{\xi^{(l)}},
\)
where \( \{w_l\} \) are ``stick-breaking" weights and \( \{\xi^{(l)}\} \) are i.i.d. atoms drawn from the base measure \( G_0 \).
The dependent DP (DDP) \citep{maceachern1999dependent,quintana2022dependent} extends the classical DP to model a collection of covariate-dependent random measures \( \{G_x\}_{x \in \mathcal{X}} \), by introducing dependence through the weights and the atoms as functions of the covariate \(x\).
The spatial DP (SDP) \citep{gelfand2005bayesian} is a spatial extension of the DDP. SDP assumes that observations are collected at a fixed, finite set of locations \( \{x_1, x_2, \dots, x_D\} \), with all locations sampled simultaneously in each replication. In each replication, observations across different locations are realized on a single GP surface.
However, SDP does not accommodate the sequential nature of BO, where the next location is selected adaptively based on past observations. Our proposed \(\infty\)-GP model allows observations to be sequentially collected at arbitrary locations and realized on different latent GP surfaces, making it well-suited for BO. Accordingly, both model fitting and predictive distribution in our model differ fundamentally from those in original SDP.

The contributions of this work are organized as follows. Section~\ref{SEC2} formulates the BO problem in its most general form and highlights the limitations of classical GP surrogates.  In Section~\ref{sec-3}, we propose a novel \(\infty\)-GP surrogate model for BO. Section~\ref{sec_4} derives the predictive distribution of the \(\infty\)-GP and introduces the corresponding TS acquisition policy. In Section~\ref{Section_6}, we propose a new TS regret analysis framework for general rewards and show that \(\infty\)-GP with TS achieves no-regret guarantees. Section~\ref{section_7} discusses computational considerations and shows that, with a tailored truncated Gibbs sampler, $\infty$-GP incurs only minimal additional cost compared with a standard GP. Finally, Section~\ref{sec_exp} provides comprehensive empirical evaluations.

\section{Problem Formulation and Background}\label{SEC2}

In its most general form, BO aims to solve the following stochastic optimization problem: \begin{equation}\label{general_formulation} \arg\max_{x\in\mathcal{X}} \mathbb{E}_{y\sim f^\ast(\cdot|x)}(y), \end{equation} where $\mathcal{X}\subset\mathbb{R}^d$ is a compact decision space, and $f^\ast(y|x)$ denotes the \textit{unknown} conditional distribution of the stochastic reward $y(x)\in\mathbb R$ observed at $x$. Our goal is to identify the decision $x$ that maximizes the expected reward.

At the $n$-th iteration, the decision maker observes the reward at the query point $x_n\in\mathcal{X}$ and uses the data collected so far, i.e., $\mathcal{H}_{n}=\{(x_1,y(x_1)),\cdots,(x_{n},y(x_{n}))\}$, to fit a \textit{surrogate model} that predicts the reward distribution over $\mathcal{X}$. Based on this computationally cheap surrogate, an acquisition policy is then applied to select the next promising point $x_{n+1}\in\mathcal{X}$ to evaluate, balancing exploration and exploitation.

    \subsection{Limitations of Classical BO Methods and GP Surrogate Model}\label{sec_limit_of_convential}  
The quality of the surrogate model is critical to the performance of BO.
Classical BO methods predominantly employ Gaussian processes (GPs) as the surrogate models. However, the effectiveness of these GP-based methods relies heavily on restrictive assumptions on the ground-truth reward.
Specifically, we rewrite the observed reward as \begin{equation}\label{eq_old_formulation}
     y(x)=\mu^\ast(x)+\epsilon^\ast(x),
 \end{equation}

     where $\mu^\ast(x):=\mathbb{E}_{y\sim f^\ast(\cdot|x)}(y)$ is the unknown deterministic objective function and $\epsilon^\ast(x):= y(x)-\mu^\ast(x)$ is the zero-mean noise. As listed in Table \ref{tab:bo-assumptions}, the assumptions imposed in classical GP-based BO can be summarized in the following Assumption \ref{ass_0}.
\begin{definition}[RKHS Assumption]
$\mu^\ast(x)$ lies in the RKHS $H_k$ associated with a kernel $k(x,x')$, with a bounded RKHS norm $\|\mu^\ast\|_{H_k}\le B_{H_k}$ for some known constant $B_{H_k}>0$.
\end{definition}

\begin{assumption}[Classical BO Assumptions]\label{ass_0}
    The observed reward is given by Eq.~\eqref{eq_old_formulation}.
$\mu^\ast(x)$ is assumed to be Lipschitz continuous, and is further required to be either a sample path from a GP or to satisfy the \emph{RKHS assumption}. The noise terms $\{\epsilon^\ast(x)\}_{x \in \mathcal{X}}$ are assumed to be independent and identically distributed, zero-mean, (sub-)Gaussian random variables.
\end{assumption}
Assumption~\ref{ass_0} imposes strong assumptions on both the objective $\mu^\ast(x)$ and the noise $\epsilon^\ast(x)$.
On the noise side, the i.i.d. sub-Gaussian requirement is essential for applying concentration inequalities that underpin the theoretical guarantees of GP-based BO.
On the objective side, even in the absence of noise, GP-based BO may still fail unless $\mu^\ast(x)$ satisfies highly restrictive structural conditions.
Two alternative assumptions are: (i) $\mu^\ast(x)$ is a sample path from a GP and is Lipschitz continuous; or (ii) the \textit{RKHS assumption}, which requires $\mu^\ast(x)$ to belong to the RKHS associated with a given kernel and to have a {known} upper bound on its RKHS norm. Notably, the RKHS assumption itself implies Lipschitz continuity \citep[Lemma 1]{de2012exponential}.

Several practically important reward structures violate Assumption~\ref{ass_0} or fall outside the modeling capacity of standard GPs.
 First, \textbf{non-stationary} rewards, whose covariance structure varies across the input domain, contradict the stationarity assumption of classical GP kernels such as the squared exponential, Matérn, and rational quadratic kernels, which require that the covariance between $x$ and $x'$ depends solely on their distance $\|x - x'\|$.
Second, rewards with \textbf{heavy-tailed or heteroscedastic} noise, where the noise distribution either has tails heavier than sub-Gaussian or varies with the point $x$, violate the i.i.d. sub-Gaussian noise assumption.


\subsection{Model Uncertainty vs.\ Value Uncertainty}

A key strength of GPs in BO is their ability to quantify uncertainty through the posterior variance. However, the uncertainty captured by a GP is inherently limited to \textit{value uncertainty},  i.e., uncertainty about the predicted value at a given input point, under a fixed reward distribution class. This overlooks a more fundamental source of uncertainty: whether the GP itself is an appropriate surrogate model for the underlying reward?

In this work, we highlight an additional and conceptually distinct form of uncertainty, which we call \textit{model uncertainty}. Model uncertainty represents the decision maker’s uncertainty about the correct {type} or {complexity} of the true reward distribution, rather than uncertainty in its realized values.
In the next section, we introduce a new surrogate model that is designed to explicitly capture this form of model uncertainty and to accommodate a much broader range of reward distributions than standard GPs.

\section{$\infty$-Gaussian-Process Surrogate Modeling}\label{sec-3}

In this section, we introduce the $\infty$-Gaussian Process ($\infty$-GP), a novel surrogate model for the observed reward $\{y(x): x \in \mathcal{X}\}$. It is constructed as a \textit{sequential spatial DP mixture with a GP baseline}.

\begin{figure}[h]
    \centering
    \includegraphics[width=1\textwidth]{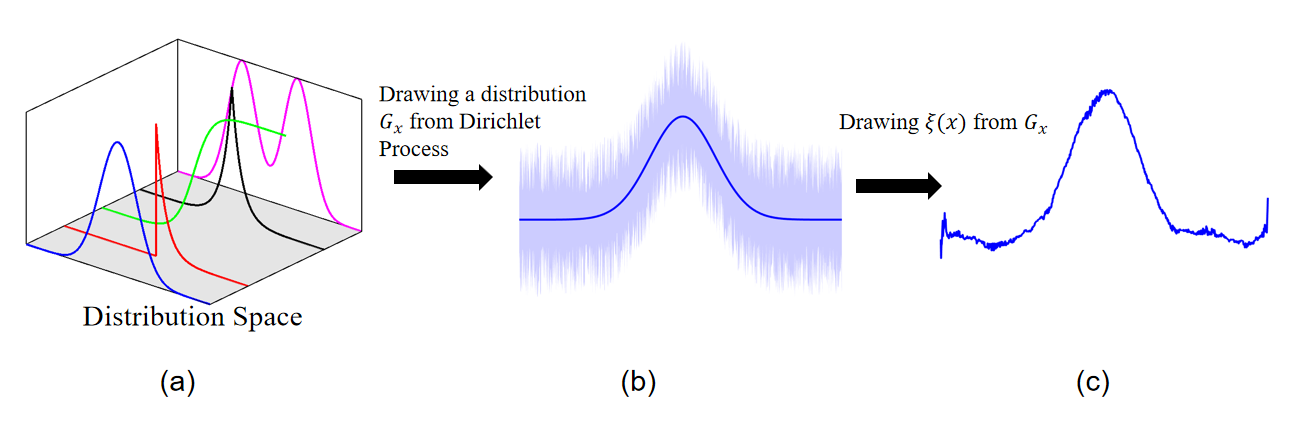} 
    \caption{(a)\textendash(b): Draw a stochastic process from the distribution space, representing model uncertainty. (b)\textendash(c): Draw a realization from the stochastic process, representing value uncertainty. $\infty$-GP model includes (a)\textendash(c), whereas the GP model includes only (b)\textendash(c).}
    \label{hierarchy}
\end{figure}

Suppose that the observations are collected at input points $x_1,\cdots,x_n\in\mathcal{X}$ sequentially. The observed reward is modeled as 
\begin{equation}\label{eq_y=m+xi+e}
    y(x_i)=\beta^\top x_i+\xi (x_i)+\epsilon_i,\ \epsilon_i\sim N(0,\tau^2), 
\end{equation}
\begin{equation}\label{eq_xi-gxi}
    \xi(x_i)\overset{ind}{\sim} G_{x_i},
\end{equation}
$\forall i=1,\cdots,n$, where $\beta^\top x$ represents the deterministic mean trend term with $\beta\in\mathbb{R}^d$. The main modeling component is the stochastic process $\xi:= \{\xi(x)\in\mathbb{R}:x\in\mathcal{X}\}$, which captures the spatial dependence across the decision space and follows distribution $\{G_x:x\in\mathcal{X}\}$.  In classical GP-based BO methods, as shown in Figure \ref{hierarchy}(b) and \ref{hierarchy}(c),
$\xi$ is modeled as a sample path of a GP,  i.e., a \textit{surface} over $\mathcal{X}$.
Because the distributional form is fixed to be a GP, this approach captures only \textit{value uncertainty} within a restrictive model class.

 To capture \textit{model uncertainty}, we do not fix the distribution type of \( \{ G_x :x\in\mathcal{X}\} \). Instead, as illustrated in Figure \ref{hierarchy}(a), we assume that \( \{ G_x :x\in\mathcal{X}\} \) is itself random and follows a (sequential) SDP, which places a prior over \textit{the space of distributions}:
\begin{equation}\label{equation_gx-sdp}
  \{G_{x} : x \in \mathcal{X}\} \sim \text{SDP}(\nu G_0),
\end{equation}
where $\nu>0$ is a scalar and $G_0$ is a specified baseline distribution of $\xi$. In this work, we use a stationary GP $\mathcal{GP}(0,\sigma^2 \rho_{\bm{\phi}}(\cdot,\cdot))$ over $\mathcal{X}$ as the baseline distribution, with kernel $\rho_{\bm{\phi}}(x,x')=\exp\left( -\sum_{k=1}^d \phi_k (x^{(k)} - x'^{(k)})^2 \right)$, where $x^{(k)}$ denotes the $k$-th coordinate of $x$ and the length-scale parameters are $\bm{\phi}=(\phi_1,\cdots,\phi_d)$.
Specifically, a distribution drawn from $\text{SDP}(\nu G_0)$ is almost surely discrete and admits the representation
\begin{equation}\label{equation_infi_surface}
G_{x}=\sum_{l=1}^{\infty}w_l\delta_{\xi^{(l)}(x)}, \forall x\in\mathcal{X},
\end{equation}
 where each surface $\xi^{(l)}:= \{\xi^{(l)}(x):x\in\mathcal{X}\}$ is a sample path independently drawn from the baseline $\mathcal{GP}(0,\sigma^2 \rho_{\bm{\phi}}(\cdot,\cdot))$ and the weights $\{w_l\}_{l=1}^\infty$ admit the well-known ``stick breaking" construction of the DP, i.e.,
\begin{equation}\label{eq_prior_v}
    w_1=V_1,\ \ w_l=V_l\prod_{r=1}^{l-1}(1-V_r), \ \  V_r\overset{iid}{\sim} Beta(1,\nu).
\end{equation}
Therefore, as shown in Figure \ref{fig:surface}, a draw $\{G_x:x\in\mathcal{X}\}$ from $\text{SDP}(\nu G_0)$ can be viewed as an infinite mixture of GP surfaces, and thus is termed as $\infty$-GP. 
It is important to note that, unlike classical BO methods that impose a surrogate model on the objective function \( \mu^\ast(x) \) and assume Gaussian observation noise, $\infty$-GP directly models the full distribution of the observed reward \( y(x) \), thereby accommodating a broader class of noise distributions. 
In this framework, the term \( \epsilon_i \sim N(0, \tau^2) \) in model~(\ref{eq_y=m+xi+e}) should not be interpreted as imposing an i.i.d. Gaussian assumption on the true measurement noise.
Instead, it serves as a modeling component that induces a prior over the space of probability distributions with continuous support.

\begin{figure}[htbp]
    \centering
    \includegraphics[width=0.6\textwidth]{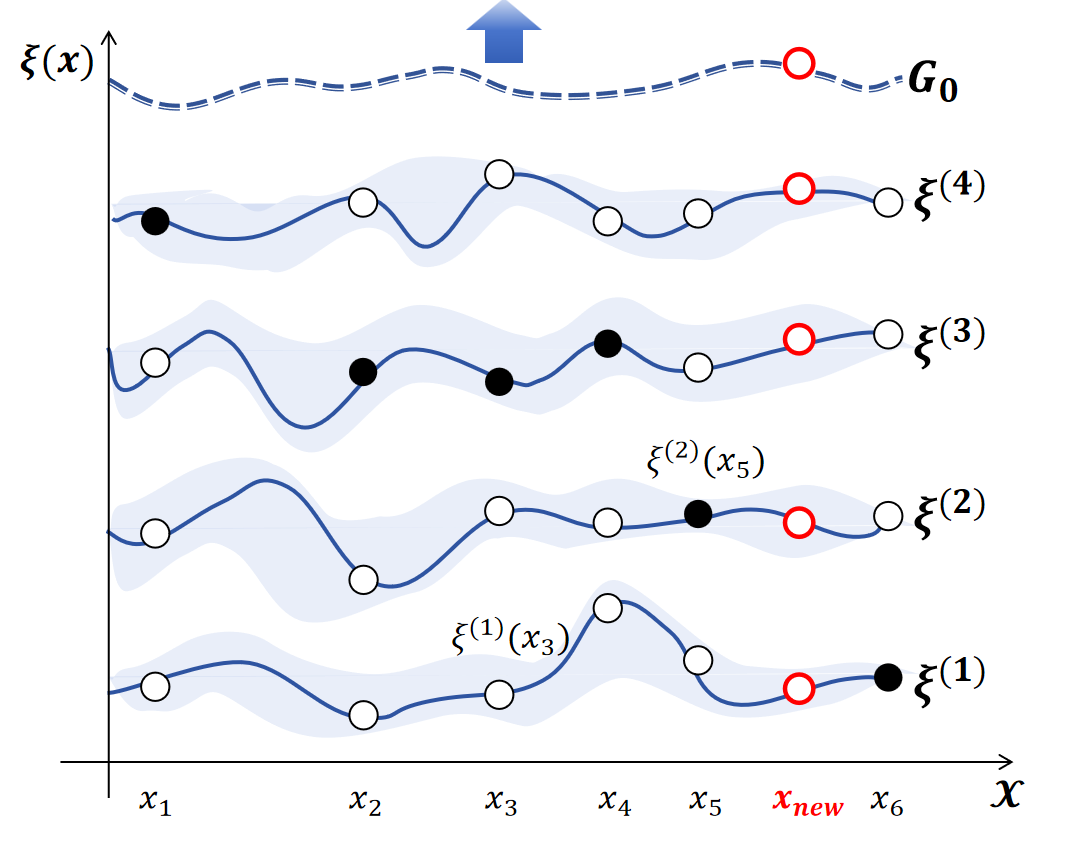} 
    \caption{$\infty$-GP is theoretically an infinite sum of GP surfaces. The solid dots means that the $\xi(x_i)$ is realized on that surface and the hollow ones represent that an unrealized latent variable. Red dots means predicting the new response at an unexplored $x_{new}$. Red dots on solid lines represent new observations being realized on an existing surface, while dashed lines indicate the observation lies on a new surface drawn from the base distribution.}
    \label{fig:surface}
\end{figure}


We now take a closer look at the internal structure of the $\infty$-GP model, which is critical for the posterior inference in the next section.
As defined in Eq.~(\ref{equation_infi_surface}), each distribution \( \{G_x : x \in \mathcal{X} \} \) drawn from the SDP is a discrete measure over an infinite collection of surfaces \( \{\xi^{(l)}\}_{l=1}^\infty \).
For any \( x_i \in\mathcal{X}\), a draw \( \xi(x_i) \sim G_{x_i} \) can be realized on any of the surfaces in the collection. To indicate on which surface the $\xi(x_i)$ is realized,
 we introduce auxiliary variables $\bm{z}_{1:n}=\{z_i\}_{i=1}^n$, where each $z_i\in\mathbb{N}$ denotes the index of the of the surface on which $\xi(x_i)$ is realized; that is,  $\xi(x_i)=\xi^{(z_i)}(x_i)$.
Note that for any surface \( \xi^{(j)} \) with $j\neq z_i$, on which the observation at \( x_i \) is not realized, the corresponding value \( \xi^{(j)}(x_i) \) should be regarded as latent and must also be inferred. As illustrated in Figure~\ref{fig:surface}, solid dots (e.g., \( \xi^{(2)}(x_5) \)) indicate that the observation at \( x_5 \) is realized on surface \( \xi^{(2)} \), while hollow dots (e.g., \( \xi^{(1)}(x_3) \)) indicate that the observation at \( x_3 \) is not realized on surface \( \xi^{(1)} \).
Since multiple observations may be realized on the same surface, the number of realized surfaces, denoted by \( K_n \), typically grows slowly with \( n \) (see Section~\ref{section_7}). Let \( \bm{\xi}_{1:n} = (\xi^{(1)}, \dots, \xi^{(K_n)}) \) denote the collection of realized surfaces up to iteration \( n \), and let \( n_j \) denote the number of points within \( \{x_1, \cdots, x_n\} \) that are realized on surface \( \xi^{(j)} \), i.e., \( n_j = \#\{i : z_i = j\} \), $j=1,\cdots, K_n$.

Similar to the GP model, the hyperparameters of the $\infty$-GP must be inferred during use.
We adopt a fully-Bayesian treatment to estimate the hyperparameters $\Theta:=\{\beta,\nu,\tau,\sigma^2,\phi\}=\{\Theta^{(1)},\Theta^{(2)}\}$, where the first-layer parameters $\Theta^{(1)}:=\{\beta,\tau\}$ are the hyperparameters that appear in Eq.~\eqref{eq_y=m+xi+e} and the second-layer parameters $\Theta^{(2)}:=\{\nu,\sigma^2,\phi\}$ are those that appear in Eq.~\eqref{eq_xi-gxi}. The fully Bayesian treatment involves placing priors on the hyperparameters and inferring them by updating their posterior, rather than estimating them in a frequentist manner using methods such as MLE. As emphasized in the seminal work of \citet{snoek2012practical}, fully Bayesian hyperparameter treatment is often more sensible for practical BO.
Specifically, the priors on hyperparameters $\Theta$ are given by
\begin{equation}\label{eq_prior_tau}
    \beta,\tau^2\sim N_p(\beta_0,\Sigma_\beta)\times IGamma(a_{\tau},b_{\tau}),
\end{equation}
\begin{equation}\label{eq_prior_sigma}
    \sigma^2\sim IGamma(a_\sigma,b_\sigma),
\end{equation}
\begin{equation}
    \bm{\phi}\sim U((0,{b_\phi}]^d),
\end{equation}
\begin{equation}\label{eq_end_of_model}
    \nu\sim Gamma(a_{\nu},b_{\nu}),
\end{equation}
where $\beta$ has a Gaussian prior, $\tau^2$ and $\sigma^2$ have inverse-Gamma priors, $\phi$ has a uniform prior on $(0,b_\phi]$ and $\nu$ has a Gamma prior. Practical guidelines for specifying these priors are provided in the Appendix~\ref{SEC:PRIOR SPECIFICATION}.

 
The advantages of the \( \infty \)-GP surrogate model, which will be explained in detail in the following sections, can be briefly summarized as follows. 
First, the model can adaptively adjust its complexity based on the observed data (e.g., the number of latent surfaces), enabling it to approximate and learn a broad class of reward distribution types.
 In Section~\ref{Section_6}, we show that this flexibility leads to improved regret performance, yielding the first \textit{no-regret} guarantee for general reward distributions under mild conditions.
Moreover, the \( \infty \)-GP is scalable as $n$ grows large, requiring minimal additional computation and memory compared to the classical single GP model, which will be explained in detail in Section~\ref{section_7}.

To conclude this section, it is important to understand the two quantities (or ``parameters") that characterize the SDP: $G_0$ and $\nu$. For any measurable set $B$ in the support of $G_0$ and if a random distribution $G\sim \text{SDP}(\nu G_0)$, then we have $\mathbb{E}(G(B))=G_0(B)$. This is to say, $G_0$ serves as a ``mean" measure for random measures drawn from $\text{SDP}(\nu G_0)$. The parameter $\nu$ can be interpreted as a prior precision parameter, controlling the variability of $G$ around $G_0$. Specifically $\mathrm{Var}(G(B))=\frac{G_0(B)(1-G_0(B))}{1+\nu}$. In practice, with new data becoming available, we update the posterior of $\nu$, which allows $\nu$ to quantify the \textit{model uncertainty} dynamically: the higher the $\nu$, the more confident we are in the assertion that the true reward distribution follows the baseline GP $G_0$. In particular, when $\nu\to\infty$, $\infty$-GP model will degenerate to a standard GP.  Moreover, it is important to note that despite using a stationary exponential kernel, \textit{the random process drawn from the SDP is non-stationary and has heterogeneous variance.}

\section{Predictive Distribution and Thompson Sampling}\label{sec_4}
Having introduced the novel $\infty$-GP surrogate model, we now turn to the next key component of BO: the acquisition policy, which determines the next candidate point $x_{n+1}\in\mathcal{X}$ to evaluate based on the surrogate model's prediction of its potential value.

\subsection{Predictive Distribution}
The first step in designing an acquisition policy is to derive the predictive distribution of the response $y(x_{n+1})$ at any unexplored point $x_{n+1} \in \mathcal{X}$ based on the surrogate model. Specifically, we seek to compute the posterior distribution
\begin{equation}\label{eq_predictive}
    [y(x_{n+1})|\mathcal{H}_n],
\end{equation}
 where $\mathcal{H}_n=\{(x_1,y(x_1)),\cdots,(x_n,y(x_n))\}$ represents the set of observed data collected during the past $n$ iterations. 
For notational convenience, we adopt the bracket notation system used in \cite{gelfand1990sampling}, where $[Y \mid X]$ denotes the conditional density of $Y$ given $X$, and $[Y]$ denotes the marginal density of $Y$. For example, let \( f(\cdot) \) denote a probability density function. Then, expressions like $\int [Y \mid X][X]  $
should be interpreted as shorthand for $\int f(Y \mid X)  f(X)  dX,$
 which yields the marginal density \( f(Y) \) by the law of total probability. That is, 
$\int [Y \mid X][X] = [Y].$
It is important to note that, since \( x_{n+1} \) is not fixed but varies over the decision space \( \mathcal{X} \), the predictive distribution~(\ref{eq_predictive}) should be understood as a predictive stochastic process over \( \mathcal{X} \), rather than a single-point distribution.

Using the surrogate model~\eqref{eq_y=m+xi+e}, the main challenge of computing \eqref{eq_predictive} lies in inferring the posterior distribution of $\xi(x_{n+1})$. Once this component is inferred, the posterior distributions of the remaining terms $\beta^\top x_{n+1}$ and $\epsilon_{n+1}$ can be directly obtained (see Algorithm~\ref{algo_gibbs_sampling}).
Let us begin with the classical BO approach, which places a GP prior on the $\{\xi(x): x \in \mathcal{X}\}$. In that case, the joint distribution of $(\xi(x_1), \dots, \xi(x_n), \xi(x_{n+1}))$ is multivariate Gaussian, i.e., $(\xi(x_1), \dots, \xi(x_{n+1})) \sim \mathcal{N}\left( \bm{0}, \Sigma_0(x_{1:{n+1}}, x_{1:{n+1}}) \right),$
where $\Sigma_0(x_{1:{n+1}}, x_{1:{n+1}}) \in \mathbb{R}^{(n+1)\times(n+1)}$ is the covariance matrix with entries $\Sigma_0(x_i, x_j) = \sigma^2 \rho_{\bm{\phi}}(x_i, x_j)$, $\forall i,j=1,\cdots,n+1$. Conditioning on the latent values at previously evaluated inputs $\xi(x_{1:n}) := (\xi(x_1), \dots, \xi(x_n))$, the posterior distribution of $\xi(x_{n+1})$ remains Gaussian.
Computing this conditional distribution corresponds to the well-known \textit{kriging} technique \citep{kriging_ankenman} and yields a predictive GP over $\mathcal{X}$:
\begin{equation} \label{eq_kriging_BO}
\underbrace{[\xi(x_{n+1})\mid \xi(x_{1:n}),\sigma^2,\bm{\phi}]}_{\text{kriging}} \sim \mathcal{GP}(\mu_n(x_{n+1}), \sigma_n^2(x_{n+1})),
\end{equation}
where the posterior mean and variance functions are given by
\begin{equation}
\mu_n(x_{n+1}) := \Sigma_0(x_{n+1}, x_{1:n}) \Sigma_0(x_{1:n}, x_{1:n})^{-1} \xi(x_{1:n}),
\end{equation}
\begin{equation} \label{kriging_variance}
\sigma_n^2(x_{n+1}) := \sigma^2 - \Sigma_0(x_{n+1}, x_{1:n}) \Sigma_0(x_{1:n}, x_{1:n})^{-1} \Sigma_0(x_{1:n}, x_{n+1}),
\end{equation}
with $\Sigma_0(x_{n+1}, x_{1:n}) := (\Sigma_0(x_{n+1}, x_1), \dots, \Sigma_0(x_{n+1}, x_n))\in\mathbb{R}^{1\times n}$.

In contrast, the $\infty$-GP model can be interpreted as a mixture of infinitely many GP surfaces $\{\xi^{(l)}\}_{l=1}^\infty$. 
Under this more flexible model, instead of assuming that all data lie on a single surface, the new $x_{n+1}$ may be realized on any of these surfaces. With $z_{n+1}$ representing the surface assignment of $x_{n+1}$, as illustrated by the red dots in Figure~\ref{fig:surface}, 
the prediction of $\xi(x_{n+1})=\xi^{(z_{n+1})}(x_{n+1})$ consists of two steps: 
\begin{itemize}
    \item \textbf{Step (a): Predicting the surface $\xi^{(z_{n+1})}$ on which $x_{n+1}$ is realized.}  
Conditioned on the previously realized surfaces $\bm{\xi}_{1:n} = (\xi^{(1)}, \dots, \xi^{(K_n)})$ and their surface assignments $\bm{z}_{1:n}$, and after {marginalizing out} the random measure $G_{x_{n+1}}$, the distribution of the surface assignment index $z_{n+1}$ follows the urn scheme of the DP~\citep{blackwell1973ferguson}. Specifically, the new observation either reuses one of the existing surfaces from $\bm{\xi}_{1:n}$ (solid lines in Figure~\ref{fig:surface}) or initiates a new surface drawn from the base measure $G_0 = \mathcal{GP}(0, \sigma^2 \rho_{\bm{\phi}}(\cdot, \cdot))$ (dashed lines in Figure~\ref{fig:surface}). Formally,
\begin{equation}\label{urn_scheme}
\begin{aligned}
    \underbrace{[\xi^{(z_{n+1})} \mid \bm{\xi}_{1:n}, \bm{z}_{1:n}, \Theta^{(2)}]}_{\text{Urn scheme}} 
&= \int [\xi^{(z_{n+1})} \mid \bm{\xi}_{1:n}, \bm{z}_{1:n}, \Theta^{(2)}, G_{x_{n+1}}]   [G_{x_{n+1}}]\\
&\sim \underbrace{\frac{\nu}{\nu + n} G_0}_{\mathrm{Opening\ a\ new\ surface}} + \sum_{j=1}^{K_n}\underbrace{ \frac{n_j}{\nu + n} \delta_{\xi^{(j)}}}_{\mathrm{Reusing\ old\ surfaces}}.
\end{aligned}
\end{equation}
Recall that $n_j$ denotes the number of previously evaluated inputs whose observations are realized on surface $\xi^{(j)}$. A surface with a larger $n_j$ is correspondingly more likely to be reused (with probability $\frac{n_j}{\nu+n}$). 
\item \textbf{Step (b): Performing kriging on the surface $\xi^{(z_{n+1})}$.}  
Given the surface index $z_{n+1}$ determined in Step (a), we predict $\xi^{(z_{n+1})}(x_{n+1})$ using standard GP kriging conditioned on the previously realized values $\xi^{(z_{n+1})}(x_{1:n})$ on that surface, following Eq.~\eqref{eq_kriging_BO}. 
\end{itemize}

Combining Step (a) and Step (b), the conditional distribution of 
$\xi^{(z_{n+1})}(x_{n+1})$ given the realized values 
$\bm{\xi}_{1:n}(x_{1:n}) := \left\{ \xi^{(j)}(x_i) \right\}_{\substack{i = 1,\dots,n \\ j = 1,\dots,K_n}}\in \mathbb{R}^{K_n \times n}$, the surface assignments 
$\bm{z}_{1:n}$, and the hyperparameters $\Theta^{(2)}$, admits the following closed form: \begin{equation}
\begin{aligned}\label{eq_urm_finite}
&{\textcolor{red}{   [\xi^{(z_{n+1})}(x_{n+1})|\bm{\xi}_{1:n}(x_{1:n}),\bm{z}_{1:n},\Theta^{(2)}]}}\\
&=\int \underbrace{[\xi^{(z_{n+1})}(x_{n+1})|\bm{\xi}_{1:n}(x_{n+1}),\bm{z}_{1:n},\Theta^{(2)}]}_{\text{Step (a): Urn scheme.}\ \text{Eq}.(\ref{urn_scheme})}\underbrace{[\bm{\xi}_{1:n}(x_{n+1})|\bm{\xi}_{1:n}(x_{1:n}),\Theta^{(2)}]}_{\text{Step (b): kriging. Eq.\eqref{eq_kriging_BO}}}\\
&\sim\frac{\nu}{\nu+n}G_0+\sum_{j=1}^{K_n}\frac{n_j}{\nu+n}\mathcal{GP}(\mu_n^{(j)}(x_{n+1}),\sigma_n^2(x_{n+1})),
\end{aligned}
\end{equation}
where $\mu_n^{(j)}(x_{n+1}) = \Sigma_0(x_{n+1}, x_{1:n}) \Sigma_0(x_{1:n}, x_{1:n})^{-1} \xi^{(j)}(x_{1:n})$, and $\sigma_n^2(x_{n+1})$ is defined in Eq.~\eqref{kriging_variance}. 
Then, the final step in obtaining the predictive distribution 
$[y(x_{n+1}) \mid \mathcal{H}_n]$ is to infer the posterior  
$[\bm{\xi}_{1:n}(x_{1:n}),\bm{z}_{1:n},\Theta^{(2)}|\mathcal{H}_n]$ using MCMC.
 We summarize the full predictive distribution decomposition in the following proposition. The proof is provided in Appendix~\ref{sec:ec:prop1}.

\begin{proposition}[Predictive Distribution]\label{prop_core_decompose}
    The predictive reward distribution for any unexplored solution $x_{n+1}\in\mathcal{X}$ is given by  
\begin{equation}
    \begin{aligned}\label{decompose2}
        &[y(x_{n+1})|\mathcal{H}_n]=\int\textcolor{blue}{\underbrace{[y(x_{n+1})|\Theta^{(1)},\xi^{(z_{n+1})}(x_{n+1})]}_{A:\ \text{Gaussian}}}\underbrace{[\Theta^{(1)},\xi^{(z_{n+1})}(x_{n+1})|\mathcal{H}_n]}_{B+C}\\
        &=\int\int\textcolor{blue}{\underbrace{[y(x_{n+1})|\Theta^{(1)},\xi^{(z_{n+1})}(x_{n+1})]}_{A:\ \text{Gaussian}}}\textcolor{red}{\underbrace{[\xi^{(z_{n+1})}(x_{n+1})|\bm{\xi}_{1:n}(x_{1:n}),\bm{z}_{1:n},\Theta^{(2)}]}_{B:\ Eq.\eqref{eq_urm_finite}}}\underbrace{[\Theta,\bm{\xi_{1:n}}(x_{1:n}),\bm{z}_{1:n}|\mathcal{H}_n]}_{C: \text{Simulation via MCMC}}.
    \end{aligned}
\end{equation}
The term A, $\textcolor{blue}{[y(x_{n+1}) \mid \Theta^{(1)}, \xi^{(z_{n+1})}(x_{n+1})]}$, follows $\mathcal{N}(x_{n+1}^\top \beta + \xi^{(z_{n+1})}(x_{n+1}), \tau^2)$.
\end{proposition}

In the first line of Eq.~(\ref{decompose2}), conditioned on the $\xi^{(z_{n+1})}(x_{n+1})$ and the first-layer hyperparameters $\Theta^{(1)} = (\beta, \tau^2)$, the distribution $\textcolor{blue}{[y(x_{n+1}) \mid \Theta^{(1)}, \xi^{(z_{n+1})}(x_{n+1})]}$, i.e., Term A, is Gaussian.
The second line decomposes the posterior $[\Theta^{(1)}, \xi^{(z_{n+1})}(x_{n+1}) \mid \mathcal{H}_n]$ into Term B and Term C. Here, Term B, i.e., $\textcolor{red}{[\xi^{(z_{n+1})}(x_{n+1})|\bm{\xi}_{1:n}(x_{1:n}),\bm{z}_{1:n},\Theta^{(2)}]}$, follows the $(K_n+1)$-GP mixture defined in Eq.~\eqref{eq_urm_finite}. Term C, $[\Theta,\bm{\xi_{1:n}}(x_{1:n}),\bm{z}_{1:n}|\mathcal{H}_n]$, denotes the joint posterior over the hyperparameters $\Theta$ and the latent variables $\bm{\xi}_{1:n}$ and $\bm{z}_{1:n}$, which can be efficiently inferred via a tailored Gibbs sampling procedure (see Section~\ref{section_7}).
We note that, strictly speaking, Term C should be written as
$
[\Theta^{(1)} \mid \bm{\xi}_{1:n}(x_{1:n}), \bm{z}_{1:n}, \Theta^{(2)}, \mathcal{H}_n]  [\Theta^{(2)}, \bm{\xi}_{1:n}(x_{1:n}), \bm{z}_{1:n} \mid \mathcal{H}_n].
$
However, the marginalization over $\Theta^{(1)}$ is implicitly embedded in the Gibbs sampling scheme. Therefore, we adopt the current compact representation of Term C for simplicity.

To compute the predictive distribution in Eq.~\eqref{decompose2}, one would, in principle, approximate the integral by drawing multiple samples from the posterior
$[\Theta,\bm{\xi}_{1:n}(x_{1:n}),\bm{z}_{1:n} \mid \mathcal{H}_n]$
via Gibbs sampling and averaging over them, which can be computationally intensive.
However, when TS is used as the acquisition strategy, we can bypass repeated posterior sampling, as TS requires only a single draw from
\(
[\Theta,\bm{\xi}_{1:n}(x_{1:n}),\bm{z}_{1:n}\mid\mathcal{H}_n].
\)

\subsection{$\infty$-GP Thompson Sampling}\label{sec-TS}
Therefore, we now turn to the TS acquisition policy and provide a detailed description of its implementation under the \(\infty\)-GP surrogate model, referred to as $\infty$-GP-TS.

It is important to emphasize that, although the surrogate modeling and predictive inference in the previous sections are performed directly on the noisy observed reward \( y(x) \), the ultimate goal in BO is to maximize the expected reward \( \mu^\ast(x)=\mathbb{E}(y(x)) \). Accordingly, the TS acquisition policy proceeds by drawing a sample from the posterior distribution of \( \mu^\ast(x) \) given the data \( \mathcal{H}_n \), and selecting its maximizer as the next point to evaluate,  i.e., 
\begin{equation}\label{eq_ts_initial}
    x_{n+1}=\arg\max_{x \in \mathcal{X}} \hat{\mu}^\ast(x), \quad \hat{\mu}^\ast \sim [\mu^\ast \mid \mathcal{H}_n].
\end{equation}

We now describe the posterior $[\mu^\ast \mid \mathcal{H}_n]$ and show how TS policy,  i.e., Eq.~(\ref{eq_ts_initial}), is implemented. We define $\Theta' := \{\xi^{(z_{n+1})}(x_{n+1}), \Theta^{(1)}\}$. According to the decomposition in Proposition~\ref{prop_core_decompose}, the posterior predictive distribution of $y(x_{n+1})$ can be expressed as
$$[y(x_{n+1}) \mid \mathcal{H}_n] 
= \int \textcolor{blue}{[y(x_{n+1}) \mid \Theta']}   [\Theta' \mid \mathcal{H}_n]
= \int \textcolor{blue}{\mathcal{N}(x_{n+1}^\top \beta + \xi^{(z_{n+1})}(x_{n+1}), \tau^2)}   [\Theta' \mid \mathcal{H}_n].$$
Therefore, conditioned on the latent $\Theta'$, the distribution $[y(x_{n+1})|\Theta']$ is Gaussian and the conditional expected reward is $[\mu^\ast(x_{n+1}) \mid \Theta'] = x_{n+1}^\top \beta + \xi^{(z_{n+1})}(x_{n+1}).$ The uncertainty is captured by the posterior distribution $[\Theta'|\mathcal{H}_n]$.
This allows TS under the $\infty$-GP model to be implemented as follows: First, draw a posterior sample $\hat{\Theta}' \sim [\Theta' \mid \mathcal{H}_n]$; Second, select the next evaluation point by $$x_{n+1} = \arg\max_{x \in \mathcal{X}}   x^\top \hat{\beta} + \hat{\xi}^{(z_{n+1})}(x).$$
Here, the hat notation (e.g., $\hat{\nu}, \hat{\beta}, \hat{K}_n, \{\hat{n}_j\}_{j=1}^{\hat{K}_n}$) denotes posterior samples.

The remaining question is how to generate samples from $[\Theta' \mid \mathcal{H}_n]$, particularly the trajectory $\hat{\xi}^{(z_{n+1})} = \{ \hat{\xi}^{(z_{n+1})}(x_{n+1}) : x_{n+1} \in \mathcal{X} \}$.
 Based on Proposition~\ref{prop_core_decompose}, we decompose
\begin{equation}
[\Theta' \mid \mathcal{H}_n] = \int 
\textcolor{red}{\underbrace{[\xi^{(z_{n+1})}(x_{n+1}) \mid \bm{\xi}_{1:n}(x_{1:n}), \bm{z}_{1:n}, \Theta^{(2)}]}_{\text{Term B}}}  
\underbrace{[\Theta, \bm{\xi}_{1:n}(x_{1:n}), \bm{z}_{1:n} \mid \mathcal{H}_n]}_{\text{Term C}}.
\end{equation}

Sampling from this hierarchical distribution can be carried out in two steps:
\begin{enumerate}
    \item[$\bullet$] Sample $(\hat{\Theta}, \bm{\hat{\xi}}_{1:n}(x_{1:n}), \bm{\hat{z}}_{1:n}) \sim [\Theta, \bm{\xi}_{1:n}(x_{1:n}), \bm{z}_{1:n} \mid \mathcal{H}_n]$ via a Gibbs sampler (see Algorithm \ref{algo_gibbs_sampling});
    \item[$\bullet$] Draw a sample path $\hat{\xi}^{(z_{n+1})}$ from $\textcolor{red}{[\xi^{(z_{n+1})} \mid \hat{\Theta}^{(2)}, \bm{\hat{\xi}}_{1:n}(x_{1:n}), \bm{\hat{z}}_{1:n}]}$,
    which follows the finite GP mixture specified in Eq.~\eqref{eq_urm_finite}: \begin{equation}\label{eq_urm_finite_3}
\begin{aligned}
\textcolor{red}{[\xi^{(z_{n+1})}(x_{n+1}) \mid \hat{\Theta}^{(2)}, \bm{\hat{\xi}}_{1:n}(x_{1:n}), \bm{\hat{z}}_{1:n}]
}\sim 
\underbrace{\frac{\hat{\nu}}{\hat{\nu} + n} G_0}_{\mathrm{exploration}}
+ 
\sum_{j=1}^{\hat{K}_n} \underbrace{\frac{\hat{n}_j}{\hat{\nu} + n}  
\mathcal{GP}(\hat{\mu}_n^{(j)}(x_{n+1}), \hat{\sigma}_n^2(x_{n+1}))}_{\mathrm{exploitation}},
\end{aligned}
\end{equation}
\end{enumerate}
where the first term on the RHS corresponds to ``exploration" by opening a new surface from the base $G_0$, while the second term corresponds to ``exploitation" by reusing empirical distributions $\mathcal{GP}(\hat{\mu}_n^{(j)}(\cdot), \hat{\sigma}_n^2(\cdot))$ constructed from the observed data.
Efficient sampling from posterior GP surfaces 
$\mathcal{GP}(\hat{\mu}_n^{(j)}(\cdot), \hat{\sigma}_n^2(\cdot))$, 
as required here, has been extensively studied; see 
\cite{wilson2020efficiently, lin2023sampling, zhou2025accelerating}. 
This step constitutes a primary source of computational complexity, typically scaling as $\mathcal{O}(n^3)$.
Another computational bottleneck arises from the Gibbs sampling procedure.
To mitigate this, we develop a tailored fast Gibbs sampler. We defer the discussion of computational issues to Section~\ref{section_7}.

To conclude, the pseudocode of the proposed $\infty$-GP-TS algorithm is provided in Algorithm~\ref{alg:thompson_sampling}. Rather than using the vanilla TS policy,  we adopt a $\zeta^n$-greedy variant, a commonly used strategy in TS implementations to avoid getting stuck in a suboptimal region. Specifically, at the $n$-th iteration, with probability \(1 - \zeta^n\), the original TS policy is executed, while with probability \( \zeta^n \), an action is selected uniformly at random from \( \mathcal{X} \).
To avoid incurring a constant additional regret, we adopt a decaying $\zeta^n = C_1 n^{-\lambda_1}$, where $C_1 > 0$ and $0 < \lambda_1 < 1$, following the recommendation in \citet{auer2002finite}. In the remainder of the paper, TS refers by default to this $\zeta^n$-greedy version. 

\begin{algorithm}[ht]
\caption{Thompson Sampling with $\infty$-GP Surrogate Model }
\label{alg:thompson_sampling}
\begin{algorithmic}[1]
\State \textbf{Input:}  \( \mathcal{H}_0 \), total budget \(N\), decay schedule \( \zeta^n = C_1 n^{-\lambda_1} \).
\For{\(n = 1\) to \(N\)}
    \State \textbf{Step 1:} Sample \( (\hat{\Theta}, \bm{\hat{\xi}}_{1:n}(x_{1:n}), \bm{\hat{z}}_{1:n}) \sim [\Theta, \bm{\xi}_{1:n}(x_{1:n}), \bm{z}_{1:n} \mid \mathcal{H}_n] \) via Gibbs sampling (see Algorithm~\ref{algo_gibbs_sampling}).
    \State \textbf{Step 2:} Sample \( \hat{\xi}^{(z_{n+1})} \sim [\xi^{(z_{n+1})} \mid \hat{\Theta}^{(2)}, \bm{\hat{\xi}}_{1:n}(x_{1:n}), \bm{\hat{z}}_{1:n}] \) using Eq.~\eqref{eq_urm_finite_3}.
    \State \textbf{Step 3:} With probability \( 1 - \zeta^n \), let 
        $$
        x_{n+1} = \arg\max_{x \in \mathcal{X}}   x^\top \hat{\beta} + \hat{\xi}^{(z_{n+1})}(x),
        $$
        otherwise, choose \( x_{n+1} \sim \text{Unif}(\mathcal{X}) \).
    \State \textbf{Step 4:} Evaluate $y(x_{n+1})$, update history: \( \mathcal{H}_{n+1} = \mathcal{H}_n \cup \{(x_{n+1}, y(x_{n+1}))\} \).
\EndFor
\end{algorithmic}
\end{algorithm}

\section{Theoretical Analysis: Understanding the Superiority of $\infty$-GP over GP}\label{Section_6}


In this section, we analyze the regret of TS when using either a classical GP or the proposed $\infty$-GP surrogate model, only assuming that the objective $\mu^\ast(x)$ is Lipschitz continuous and that the noise satisfies mild tail conditions.
To the best of our knowledge, no existing regret analysis framework extends to such general reward distributions, thus necessitating a new analytical approach.

Our key idea is intuitive. TS selects actions according to the posterior probability of being optimal under the surrogate model. Consequently, the regret can be controlled by the discrepancy between the surrogate posterior and the true reward distribution. We quantify this discrepancy using the total variation distance (TVD), defined for any two probability measures $p$ and $q$ on a measurable space $(\Omega, \mathcal{F})$ as
$$
\delta_{\mathrm{TV}}(p, q) := \sup_{A \in \mathcal{F}} |p(A) - q(A)|.
$$
Figure~\ref{fig:tvd_gap} provides an illustrative example of our idea. Suppose a classical GP surrogate is used. At a given input $x_0$, the posterior predictive distribution is Gaussian. But if the true distribution has heavier tails than a Gaussian (e.g., $t$-distribution) or is multi-modal, a Gaussian model cannot represent it.
As the figure shows, even with infinitely many data, the best Gaussian fit to 
the true distribution still exhibits a non-vanishing TVD gap (the blue shaded region). 
Under our TVD-based regret bound, this implies an irreducible regret. In contrast, by adaptively adjusting its complexity based on the observed data,
the $\infty$-GP surrogate can flexibly represent a wide range of reward distributions, thereby achieving vanishing regret.

\begin{figure}
    \centering
    \includegraphics[width=1\linewidth]{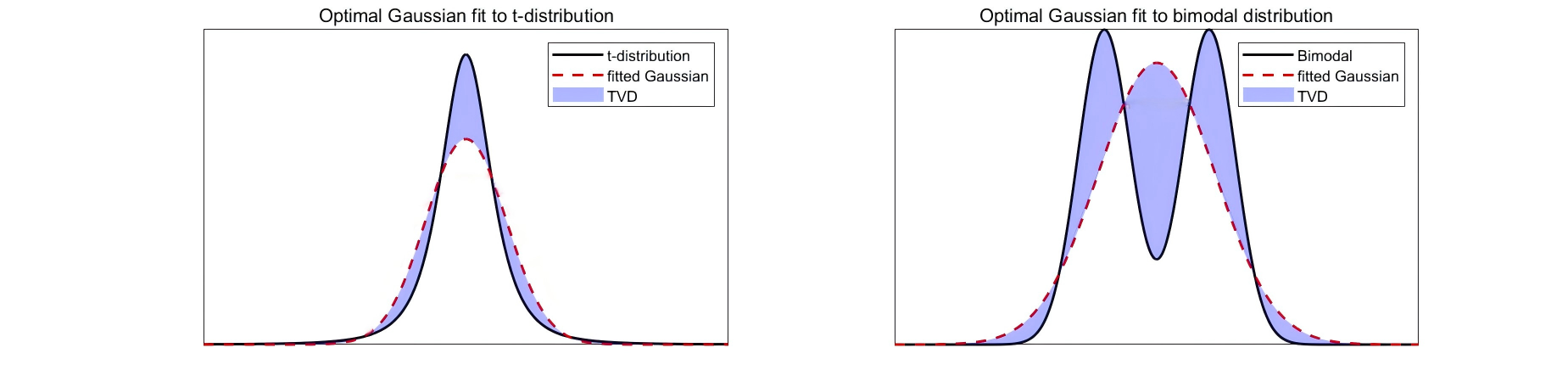}
    \caption{Illustration of the modeling-capacity limitation of GP surrogates. Even with infinite data, 
the best Gaussian approximation cannot match a $t$-distribution or a bimodal distribution, leaving a persistent TVD gap 
(blue shaded region).}
    \label{fig:tvd_gap}
\end{figure}

Therefore, in this section, we establish the no-regret property of the \( \infty \)-GP-TS algorithm in two steps. 
First, in Section~\ref{section_uni_conver}, we focus on the $\infty$-GP model itself and establish its broad convergence, showing that it can approximate a wide range of reward distributions, in contrast to GP models.
Second, in Section~\ref{section_new_regret_frame}, we introduce a new regret analysis framework for TS that links the regret directly to the TVD between the surrogate posterior and the true reward distribution.

Before proceeding, we briefly review and introduce several notations. Let
 $f^\ast(y|x)$ denote the true reward density at decision $x$ and let $f_n(y|x):= f(y|x,\mathcal{H}_n)$ denote the surrogate model’s posterior predictive distribution given $\mathcal{H}_n$. Let $\mu^\ast(x):=\mathbb{E}_{y\sim f^\ast(\cdot|x)}{y}$ be the true expected reward and $\mu^\ast_{\mathrm{max}}:= \max_{x\in\mathcal{X}}\mu^\ast(x)<\infty$.
Suppose that at the $i$-th iteration, the adopted acquisition action is $\pi^i\in\mathcal{X}$, then the instantaneous regret is $r_i:=\mu^\ast_{\mathrm{max}}-\mathbb{E}(y(\pi^i))$ and the cumulative regret is 
$\mathcal{R}_n:= \sum_{i=1}^nr_i$. 
The expectation is over the stochastic reward and the randomness of the acquisition policy. 
A sublinear growth of cumulative regret is referred to as $\textit{no-regret}$, indicating that the BO algorithm is asymptotically effective. $L_1(\mathbb{R}) $ denotes the space of all integrable functions on $\mathbb{R}$.

\subsection{Broad Convergence of $\infty$-GP Model}\label{section_uni_conver}

If we repeatedly sample from a fixed input \( x_0 \), the predictive distribution of \( \infty \)-GP surrogate model \( f_n(\cdot | x_0) \) can be shown to converge almost surely to the true conditional distribution \( f^\ast(\cdot | x_0) \), i.e.,
$$
\delta_{\mathrm{TV}}(f_n(\cdot | x_0), f^\ast(\cdot | x_0)) \to 0 \quad \text{as } n \to \infty.
$$
However, in BO, different inputs \( x \in \mathcal{X} \) are visited according to a distribution induced by the acquisition policy. As a result, the \( \delta_{\mathrm{TV}} \)-convergence holds in a visitation-weighted sense over \( \mathcal{X} \). Let \( \pi(x) \) denote the (discounted) visitation distribution induced by the sequence of acquisition policies \( (\pi^1, \pi^2, \dots) \), representing the frequency of sampling each \( x \). This distribution is used solely for theoretical analysis and does not need to be computed in practice. 
\begin{assumption}\label{ass_guangzhi}
The true reward density \( f^\ast(y|x) \) satisfies the following:
\begin{itemize}
    \item[$\mathrm{(a)}$]  There exists \( \eta > 0 \) such that
    \(
    \int_{\mathcal{X}} \int_{\mathbb{R}} |y|^{2+\eta} f^\ast(y|x)  dy  \pi(x)   dx < \infty.
    \)
    
    \item[$\mathrm{(b)}$] For all \( x \in \mathcal{X} \), \( 0 < f^\ast(y|x) < \infty \).
    
    \item[$\mathrm{(c)}$] 
    \(
    \left| \int_{\mathcal{X}} \int_{\mathbb{R}} f^\ast(y|x) \log f^\ast(y|x)   dy   \pi(x)   dx \right| < \infty,
    \)
    \(
    \left| \int_{\mathcal{X}} \int_{\mathbb{R}} f^\ast(y|x) \log \frac{f^\ast(y|x)}{\inf_{t \in [y-1, y+1]} f^\ast(t|x)}   dy   \pi(x)   dx \right| < \infty.
    \)
    
    \item[$\mathrm{(d)}$]  The mapping \( (x, y) \mapsto f^\ast(y|x) \) is jointly continuous on \( \mathcal{X} \times \mathbb{R} \), and the \( x \mapsto f^\ast(\cdot|x) \) is continuous as a function from \( \mathcal{X} \) to \( L_1(\mathbb{R}) \).
\end{itemize}
\end{assumption}
\begin{theorem}[Convergence of the \( \infty \)-GP Model]\label{theo_uni_converge}
Suppose Assumption~\ref{ass_guangzhi} holds. Then, as \( n \to \infty \),
\begin{equation}\label{eq_theo_uni_conve}
    \int_{\mathcal{X}} \delta_{\mathrm{TV}}(f_n(\cdot | x), f^\ast(\cdot | x))   \pi(x)   dx \to 0 \quad \text{almost surely}.
\end{equation}
\end{theorem}
The proof is provided in Appendix~\ref{sec_ec_proof_theo1}.
The convergence is almost sure with respect to data \( \mathcal{H}_\infty = \{(x_i, y(x_i))\}_{i=1}^\infty \) drawn i.i.d. from \( \pi(x) f^\ast(y|x) \). However, the visitation distribution $\pi(x)$ induced by vanilla TS may concentrate in certain regions of \( \mathcal{X} \) \citep{basu2017adaptive}. Consequently, some regions of $\mathcal{X}$ may remain inadequately explored, and the surrogate model over these underexplored regions may fail to converge to the true model. Such insufficient exploration may lead to a failure of convergence in BO, especially under general reward settings. This motivates our earlier introduction of the $\zeta^n$-greedy mechanism in TS, which solves this issue by enforcing exploration over $\mathcal{X}$.

The following Remark~\ref{remark_1} shows that Assumption~\ref{ass_guangzhi} is mild, and that the $\infty$-GP can approximate and converge to a significantly broader class of reward distributions than those allowed by the classical GP Assumption~\ref{ass_0}.
This flexibility provides the theoretical basis for the no-regret guarantees of the $\infty$-GP-TS developed in the next subsection.

\begin{remark}[Mildness of the Assumption~\ref{ass_guangzhi}]\label{remark_1}
    Assumption~\ref{ass_guangzhi} is very mild and strictly weaker than the classical Assumption \ref{ass_0} commonly imposed in GP-based BO. To facilitate a clear comparison with Assumption~\ref{ass_0}, consider the following broad class of reward distributions that satisfy Assumption~\ref{ass_guangzhi}:
\begin{equation}
    y(x) = \underbrace{\mu^\ast(x)}_{\text{Continuous}} + \underbrace{\epsilon^\ast(x)}_{\sum_{j=1}^{J} \omega_j(x) \psi_j},
\end{equation}
where \( \mu^\ast(x) \) is merely continuous without any additional smoothness assumptions, and \( \epsilon^\ast(x) \) is a mixture of distributions:
\(
\epsilon^\ast(x) \sim \sum_{j=1}^{J} \omega_j(x) \psi_j,
\)
where \( J \geq 1 \) and \( \omega_j(x) \in [0,1] \) are continuous input-dependent weights with \( \sum_{j=1}^{J} \omega_j(x) = 1 \). \( \psi_j \) may represent a wide variety of distribution types, including (sub-)Gaussian distributions, \( t \)-distributions with degrees of freedom \( > 2 \), as well as exponential or sub-Weibull distributions~\citep{subweibull} when the reward is non-negative. 
 Moreover, Assumption~\ref{ass_guangzhi} allows for heteroscedastic noise, i.e., noise whose distribution varies with input. 

\end{remark}

\subsection{Regret Analysis for General Reward Distributions}\label{section_new_regret_frame}
In this subsection, we establish a new regret analysis framework for TS, which links the regret directly to the TVD between the true reward distribution and the surrogate model used by the decision maker.

\subsubsection{A Central Lemma.}\label{sec_central_lemma} In this paper, we focus on asymptotic decay rates as $n\to\infty$, disregarding constant factors. Let $\precsim$ denote that the decay rate of the left-hand side is asymptotically smaller than or equal to that of the right-hand side as $n\to\infty$.

\begin{assumption}\label{ass_tail0}
    $\forall x\in\mathcal X$, there exist constants $a_1,a_2,a_3,\gamma>0$ such that for all $|y|>a_3$, $\max\big\{f^\ast(y\mid x),f_n(y\mid x)\big\}\le\ a_1\exp\big(-a_2|y|^{\gamma}\big).$
\end{assumption}

\begin{lemma}[Regret Bound via TVD]\label{lemma_central}

At iteration $n+1$, suppose that $\mathcal{X}$ is finite and Assumption \ref{ass_tail0} holds. 
Then, up to a logarithmic factor, the instantaneous regret is upper bounded by
\begin{equation}\label{lemma_1_1}
r_{n+1}\precsim \mu^\ast_{\max}  \delta_{\mathrm{TV}}\big(f_n(\cdot \mid \mathcal{X}), f^\ast(\cdot \mid \mathcal{X})\big),
\end{equation}
where $f_n(\cdot \mid \mathcal{X}) := \{f(\cdot \mid x, \mathcal{H}_n): x \in \mathcal{X}\}$ denotes the joint predictive distribution over rewards at each $x\in\mathcal{X}$ under surrogate model $f$, and $f^\ast(\cdot \mid \mathcal{X}) := \{f^\ast(\cdot \mid x): x \in \mathcal{X}\}$ denotes the joint distribution under the true reward model. The bound holds almost surely under the distribution of $\mathcal{H}_n$.
\end{lemma}
This lemma requires Assumption~\ref{ass_tail0} on the tail behavior.
In Remark~\ref{remark_3}, we show that this requirement is mild and remains strictly weaker than Assumption~\ref{ass_0} commonly imposed in classical BO.
Moreover, by applying the triangle inequality, the lemma also enables us to bound the regret gap induced by using different surrogate models, which is utilized in Corollary~\ref{theo_truncate}. The proof of Lemma~\ref{lemma_central} is provided in Appendix~\ref{sec_ec_proof_lemma1}.

\begin{remark}\label{remark2}
Strictly speaking, Eq.~\eqref{lemma_1_1} in Lemma~\ref{lemma_central} should be stated as $r_{n+1}\precsim \mu^\ast_{\max}  \mathcal{E}_n$, where $\mathcal{E}_n$ is the
\emph{posterior contraction rate} defined in Appendix~\ref{sec_ec_proof_theo1},
satisfying $\mathcal{E}_n\ge\delta_{\mathrm{TV}}\big(f_n(\cdot|\mathcal{X}), f^\ast(\cdot|\mathcal{X})\big)$ almost surely.
Thus, using $\delta_{\mathrm{TV}}(f_n(\cdot|\mathcal{X}), f^\ast(\cdot|\mathcal{X}))$ in this lemma slightly overstates the bound. 
Nevertheless, since all provable bounds on $\delta_{\mathrm{TV}}(f_n(\cdot|\mathcal{X}), f^\ast(\cdot|\mathcal{X}))$ in this work are ultimately derived using the same rate $\mathcal{E}_n$,  we unify the notation and express the results directly in terms of $\delta_{\mathrm{TV}}(f_n(\cdot|\mathcal{X}), f^\ast(\cdot|\mathcal{X}))$ throughout the main text for simplicity. A rigorous treatment is provided in the appendix.
\end{remark}

\subsubsection{A Discretization-then-Decoupling Proof Technique}

\paragraph{\textbf{Step 1: Decision Discretization.}}
The central Lemma~\ref{lemma_central} applies only when the decision set $\mathcal{X}$ is finite. Therefore, we discretize the compact decision space $\mathcal{X} \subset \mathbb{R}^d$ into a finite grid $\mathcal{X}_n = \{x^n_1, \ldots, x^n_{|\mathcal{X}_n|}\}$, where $|\mathcal{X}_n| = C_2 n^{\lambda_2}$ for some constants $C_2 > 0$ and $0 < \lambda_2 < 1$. As $n$ increases, the grid becomes increasingly dense, with points placed uniformly along each coordinate axis.
For analytical purposes, we consider a simplified variant of TS, referred to as TS(D) (\textit{discrete-TS}), in which the algorithm selects the closest point in $\mathcal{X}_n$ instead of the exact maximizer over $\mathcal{X}$. This transforms the BO problem into a special case of MAB problem with an increasing number of arms. Note that, the surrogate model’s posterior is still updated over the full domain $\mathcal{X}$ in TS(D).
TS(D) is a conceptual tool introduced solely for theoretical analysis.  Moreover, since TS(D) restricts the decision space, it may incur additional regret relative to the original TS. To ensure that this discretization does not alter the no-regret property, we impose a Lipschitz continuity condition on the objective function. 
\begin{assumption}\label{ass_lip_conti}
$\mu^\ast(x)$ is Lipschitz continuous over $\mathcal{X}$.
\end{assumption}
 Under Assumption~\ref{ass_lip_conti}, it is easy to show that the regret gap between TS and TS(D) is bounded by $\mathcal{O}(n^{1 - \lambda_2/d}) = o(n)$ (see Lemma~\ref{prop_TS-TS(D)} in the Appendix~\ref{sec:ec:prepare}). Therefore, to establish the no-regret property of the original TS, it suffices to prove that TS(D) achieves no-regret. This discretization argument is standard in the BO literature for handling continuous decision spaces \citep{srinivas2010gaussian,srinivas2012information,chowdhury2017kernelized}.

\paragraph{\textbf{Step 2: Sampling Decoupling.}}

After discretization, according to Lemma~\ref{lemma_central}, the remaining task is to develop a bound for $\delta_{\mathrm{TV}}\left( f_n(\cdot|\mathcal{X}_n), f^\ast(\cdot|\mathcal{X}_n) \right)$,
i.e., the TVD between the surrogate and true joint reward distributions over decision set $\mathcal{X}_n$.
While Theorem~\ref{theo_uni_converge} ensures posterior consistency, characterizing the concrete convergence rate of this joint distribution remains challenging, as the size $|\mathcal{X}_n| $ increases with $n$. This becomes a \textit{high-dimensional statistical inference problem} \citep{wainwright2019high}, where the dimension of the random variable grows with the sample size. Due to the curse of dimensionality, existing posterior contraction results, which is developed under fixed-dimensional settings, are no longer applicable.

To address this barrier, we consider a further simplified variant of TS(D) using {sampling decoupling}, termed as \textit{independent-discrete-TS} (TS(D)-I). In TS(D)-I, instead of drawing samples from the joint posterior distribution over \(\mathcal{X}_n\), we independently draw samples from the marginal posterior at each \(x_i^n \in \mathcal{X}_n\), \(i = 1, \dots, |\mathcal{X}_n|\).  
Notably, this decoupling affects only the sampling stage, and posterior updates remain based on the full joint model.
Let $\tilde{f}$ denote the decoupled $\infty$-GP model used in TS(D)-I, and let $\tilde{f}_n$ be its corresponding posterior predictive distribution, i.e., 
$\tilde{f}_n(\cdot|\mathcal{X}_n)=\prod_{i=1}^{|\mathcal{X}_n|}f_n(\cdot|x_i^n).$

This decoupled model can be seen as a degenerate approximation to the original, as it ignores correlations across inputs. Consequently, its posterior convergence is slower, i.e., 
$\delta_{\mathrm{TV}}\left({f}_n(\cdot|\mathcal{X}_n), {f}^\ast(\cdot|\mathcal{X}_n)\right)\precsim\delta_{\mathrm{TV}}\left(\tilde{f}_n(\cdot|\mathcal{X}_n), {f}^\ast(\cdot|\mathcal{X}_n)\right)$, implying that TS(D)-I provides a conservative upper bound on the regret of TS(D) (see the first inequality of Eq.~\eqref{eq:lemma2} in Lemma~\ref{lemma_indepelization}).
More importantly,
 using TS(D)-I enables us to bypass the curse of dimensionality:
after sampling decoupling, as shown in the second inequality of Eq.~\eqref{eq:lemma2} in Lemma \ref{lemma_indepelization}, the TVD between the joint distributions over \( \mathcal{X}_n \) can be bounded by a sum of marginal TVDs at each point \( x_i^n \in \mathcal{X}_n \). This reduces the analysis to univariate posterior convergence at each point, allowing us to directly apply well-established convergence rate results from the literature. The proof of Lemma~\ref{lemma_indepelization} is provided in Appendix~\ref{sec:ec:prepare}.
\begin{lemma}[Regret bound via Sampling Decoupling] \label{lemma_indepelization}
  Under the same assumption as in Lemma~\ref{lemma_central}, the instantaneous regret of using TS(D) at $(n+1)$-th iteration is given by
\begin{equation}\label{eq:lemma2}
    r_{n+1}\precsim \underbrace{\mu_{\mathrm{max}}^\ast\delta_{\mathrm{TV}}\left(\tilde{f}_n(\cdot|\mathcal{X}_n), {f}^\ast(\cdot|\mathcal{X}_n)\right)}_{\mathrm{regret\ of\ TS(D)-I}} \precsim \mu_{\mathrm{max}}^\ast\sum_{i=1}^{|\mathcal{X}_n|} \delta_{\mathrm{TV}}\left(\tilde{f}_n(\cdot|x_i^n), {f}^\ast(\cdot|x_i^n)\right).
\end{equation}
\end{lemma}



Finally, leveraging Lemma \ref{lemma_central} and Lemma \ref{lemma_indepelization}, we present the regret bounds in Theorem \ref{theo_3}. The proof is provided in Appendix~\ref{sec:ec:maintheorem}.
For any $\alpha>0$, $\lambda_0\geq 0$ and any non-negative function $L$ on $\mathbb{R}$, we define the locally H$\ddot{o}$lder class $\mathcal{C}^{\alpha,L,\lambda_0}(\mathbb{R})$ as the set of all functions $f:\mathbb{R}\mapsto \mathbb{R}$ that have finite $k$-th derivative $f^{(k)}$ for all $k\leq\lfloor\alpha\rfloor$, and satisfy $f^{(\lfloor \alpha \rfloor)}(y+h)-f^{(\lfloor \alpha \rfloor)}(y)\leq L(y)e^{\lambda_0h^2}|h|^{\alpha-\lfloor \alpha \rfloor},\ y,h\in\mathbb{R}.$

\begin{assumption}\label{ass_double_choice}
 $\forall x\in\mathcal{X}$, the density $f^\ast(\cdot|x)\in\mathcal{C}^{\alpha,L,\lambda_0}$ for some $\alpha>0$, $\lambda_0\geq 0$ and a non-negative function $L$ on $\mathbb{R}$ and satisfy 
    $\mathbb{E}_{Y\sim f^\ast(\cdot|x)}\big(\frac{|f^{\ast(k)}(Y|x)|}{f^\ast(Y|x)}\big)^{\frac{2\alpha+\eta}{k}}<\infty$ for any $k\leq\lfloor\alpha\rfloor$ and
    $\mathbb{E}_{Y\sim f^\ast(\cdot|x)}\big(\frac{L(Y)}{f^\ast(Y|x)}\big)^{\frac{2\alpha+\eta}{\alpha}} <\infty$ for some $\eta>0$. 
\end{assumption}

\begin{theorem}\label{theo_3}
  Suppose Assumptions \ref{ass_tail0} and \ref{ass_lip_conti} hold. Set  $\zeta^n=C_1n^{-\lambda_1}$ and \( |\mathcal{X}_n| =C_2 n^{\lambda_2} \), with $0<\lambda_1,\lambda_2<1$ and $\lambda_2 < \frac{\alpha(1 - \lambda_1)}{3\alpha + 2}$. Under different surrogate models, the following regret bounds hold:
  \begin{itemize}
\item[(a)] If TS(D) is used with the \(\infty\)-GP model as the surrogate and if Assumption \ref{ass_double_choice} holds, then
    \begin{equation}\label{eq_regret_1}
           \mathcal{R}^{\infty-\mathrm{GP}-\mathrm{TS(D)}}_{n+1}=\mathcal{O}\Big(\max\{n^{\frac{\alpha\lambda_1+(3\alpha+2)\lambda_2+\alpha+2}{2(1+\alpha)}}(\mathrm{log}\ n)^\psi,n^{1-\lambda_1}\}\Big)= o(n)
       \end{equation}
as $n\to\infty$, where $\psi>\frac{2(1+\frac{1}{\lambda}+\frac{1}{\alpha})+1}{2+\frac{2}{\alpha}}$.

 \item[(b)] If TS(D) is used with a GP model as the surrogate, the regret satisfies $\mathcal{R}^{\mathrm{GP-TS(D)}}_n = \mathcal{O}(n).$
\end{itemize}
\end{theorem}

Note that the discretization step is introduced solely as a conceptual tool for theoretical analysis. Thus, the discretization parameter $\lambda_2$ neither needs to be specified nor implemented in practice, and it may be taken to be arbitrarily small for analytical purposes.
Therefore, according to Theorem~\ref{theo_3}, when using the $\infty$-GP surrogate model and choosing $\lambda_2 < \frac{\alpha(1 - \lambda_1)}{3\alpha + 2}$, the regret of TS(D) is $o(n)$. 
In contrast, when using a classical GP surrogate, TS can only guarantee an $\mathcal{O}(n)$ regret bound under the general reward setting considered in this paper. As well documented in the literature \citep{snoek2014input,ray2019bayesian}, linear regret indeed arises for GP-based BO when the reward exhibits heavy-tailed noise or non-stationarity. Our empirical results in Section~\ref{sec_exp} further confirm this phenomenon. In conclusion, Theorem~2 demonstrates the superiority of the $\infty$-GP surrogate over the standard GP model.


\begin{remark}[Mildness of Assumption]\label{remark_3}
Thanks to the decoupling argument, we impose no structural assumptions on the objective function $\mu^\ast(x)$ beyond Lipschitz continuity (i.e., Assumption \ref{ass_lip_conti}), which is substantially weaker than the classical Assumption~\ref{ass_0}. In contrast, some regularity conditions are required on the stochastic noise $\epsilon^\ast(x)$. Specifically, Assumption \ref{ass_tail0} imposes a sub-Weibull \citep{subweibull} tail condition, while Assumption~\ref{ass_double_choice} imposes a local smoothness condition on the conditional density $f^\ast(\cdot \mid x)$.
Despite these requirements, the assumptions on the noises remain weaker than those in Assumption~\ref{ass_0}. 
To illustrate, consider the following flexible class of rewards that satisfy Assumption~\ref{ass_double_choice} (see Appendix~\ref{app:verify_assumptions} for verification):
$$
    y(x) = \underbrace{\mu^\ast(x)}_{\text{Lipschitz}} + \underbrace{\epsilon^\ast(x)}_{\sum_{j=1}^{J} \omega_j(x)  \psi_j},
$$
where $\omega_j(x)$ are continuous, input-dependent mixture weights, and each $\psi_j$ is a base distribution, including Gaussian distributions; when the rewards are restricted to be positive, $\psi_j$ also includes distributions such as Gamma, Weibull (shape parameter $>1$), and Exponential distributions.
Moreover, Assumption~\ref{ass_tail0} is satisfied by the predictive distribution $f_n$ induced by the $\infty$-GP model, since the tails of Dirichlet process mixtures of Gaussians are known to exhibit exponential decay \citep{doss1982tails}.

\end{remark}


\textcolor{blue}{}

\section{Computational Considerations and Scalability}\label{section_7}

One of the primary concerns in the recent BO literature is the computational scalability of algorithms as \(n\) increases. In this section, we discuss the memory and computational complexity of our method, and show that, thanks to the structural properties of the 
$\infty$-GP model and the use of a customized fast Gibbs sampling algorithm, only minimal additional memory and computational cost are required compared to classical GP models, specifically, increasing the complexity from $\mathcal{O}(n^3)$ to at most $\mathcal{O}(n^3 \log n)$.

\subsection{Scalable Memory Usage: Logarithmic Growth in the Number of Surfaces}

Although the $\infty$-GP model is theoretically a mixture of infinitely many GP surfaces, only a finite number (i.e., $K_n + 1$) of surfaces are realized up to the $n$-th iteration. Consequently, as shown in Algorithm~\ref{alg:thompson_sampling} and Eq.~\eqref{eq_urm_finite_3}, the primary memory cost in the $\infty$-GP-TS algorithm arises from storing the $K_n \times n$ matrix $\bm{\xi}_{1:n}(x_{1:n})$, which records the values of the realized surfaces at the evaluated input points.
A key advantage of the $\infty$-GP model is that the number $K_n$ of realized surfaces grows slowly with $n$.

\begin{proposition}[Slow Growth in the Number of Realized Surfaces]\label{prop_sparse}
As $n \to \infty$, the expected number of realized surfaces satisfies
$$\mathbb{E}(K_n) \sim \nu \log\left(\frac{n}{\nu}\right).$$
\end{proposition}
Proposition~\ref{prop_sparse} implies that the memory requirement grows only logarithmically with $n$. Thus, the total memory cost is only $\mathcal{O}(\log n)$ times that of a standard GP model. 
For example, when $\nu = 1$, the model can accommodate up to $10^5$ observations while maintaining, on average, fewer than five distinct surfaces. 
The proof of Proposition~\ref{prop_sparse} is provided in Appendix~\ref{sec:ec:proof_prop2+co1}.

\subsection{Truncated Gibbs Sampling}

The computational complexity of TS under both the $\infty$-GP and classical GP surrogate models is determined by two main components.

(i) \textbf{The complexity of kriging}, that is, computing the posterior distribution in Eq.~\eqref{eq_kriging_BO}. The computational bottleneck lies in computing the inverse of the kernel matrix $\Sigma_0(x_{1:n}, x_{1:n})^{-1}$, which incurs a cost of $\mathcal{O}(n^3)$. In the $\infty$-GP model, kriging must be performed separately for each realized surface.  
 However, a notable benefit of the $\infty$-GP model is that all surfaces share the same $\Sigma_0(x_{1:n}, x_{1:n})^{-1}$ matrix, allowing this expensive matrix inversion to be computed only once. As a result, the overall kriging complexity of $\infty$-GP remains identical to that of a standard GP model.

(ii) \textbf{The complexity of hyperparameter learning.}
In the classical GP model, hyperparameters (e.g., the length-scale and noise variance) are inferred using maximum likelihood estimation (MLE) or fully Bayesian methods such as slice sampling \citep{murray2010slice}, both of which incur an $\mathcal{O}(n^{3})$ computational cost \citep{snoek2012practical}.
In the $\infty$-GP model, the posterior distribution of the hyperparameters 
$[\Theta, \bm{\xi}_{1:n}(x_{1:n}), \bm{z}_{1:n} \mid \mathcal{H}_n]$
is inferred via Gibbs sampling, which   involves iteratively sampling subsets of the variables $\{ \Theta, \bm{\xi}_{1:n}(x_{1:n}), \bm{z}_{1:n} \}$ conditioned on the others.
 In particular, vanilla Gibbs sampling requires sequentially drawing from $\big[ \xi^{(z_i)} \mid \{ \xi^{(z_k)} : k \neq i \} \big]$ for each $i = 1, \dots, n$. This results in a computational complexity of $\mathcal{O}(n^4)$.
 
To alleviate this cost, we adopt a truncation-based approximation inspired by \cite{ishwaran2002approximate}. Specifically, we approximate the distribution $G_x = \sum_{l=1}^\infty w_l \delta_{\xi^{(l)}(x)}$ by truncating it to $L_n$ components:
\begin{equation}\label{eq_trucate}
    G_x = \sum_{l=1}^{\infty} w_l \xi^{(l)}(x) \approx \sum_{l=1}^{L_n} w_l \xi^{(l)}(x),
\end{equation}
where $L_n$ is an upper bound on the number of surfaces.
After truncation, the entire table $\bm{\xi}_{1:n}$ can be sampled jointly, which reduces the computational complexity of the first step to $\mathcal{O}(L_n n^3)$. The pseudocode is provided in Algorithm~\ref{algo_gibbs_sampling} of Appendix~\ref{SEC:EC:GIBBS}.

However, the truncation approximation might introduce additional regret.
A natural question arises: how should one select $L_n$ to ensure that the no-regret guarantee of $\infty$-GP-TS remains valid? 
The following Corollary~\ref{theo_truncate} answers this question by leveraging Lemma~\ref{lemma_central}, which bounds the additional regret through the TVD between the predictive distributions of the truncated model and the full model.

\begin{corollary}[Additional Regret Due to Truncation]\label{theo_truncate}
Under the same setting as in Theorem \ref{theo_3}, let $\mathcal{R}_{n}^{\infty-\mathrm{GP}-\mathrm{TS(D)}}$ denote the regret bound of the original $\infty$-GP-TS(D) established in Eq.~\eqref{eq_regret_1}, and let $\mathcal{R}^{L_n-\mathrm{GP}-\mathrm{TS(D)}}_{n}$ denote the regret when using the $L_n$-truncation approximation, as defined in Eq.~\eqref{eq_trucate}. Then the additional regret incurred by truncation is bounded by
$$
\mathcal{R}^{L_n-\mathrm{GP}-\mathrm{TS(D)}}_{n}-\mathcal{R}_{n}^{\infty-\mathrm{GP}-\mathrm{TS(D)}} \precsim {\mu^\ast_{\max} \exp\big( -\frac{L_n-1}{\nu} \big) n^2}
$$
\end{corollary}

The proof is provided in Appendix~\ref{sec:ec:proof_prop2+co1}.
Corollary~\ref{theo_truncate} implies that the regret induced by truncation decays exponentially in $L_n$, providing clear guidance on how to choose the truncation level. Specifically, choosing
\begin{equation}\label{eq_chocie_of_L}
    L_n = (\nu + \kappa) \log n+1
\end{equation}
for any $\kappa > 0$ ensures that the additional regret caused by truncation is $o(n)$, thereby preserving the sublinear regret growth. If $\kappa > \nu$, the additional regret becomes $o(1)$, meaning that truncation incurs no impact asymptotically. In practice, as the optimization proceeds, we infer and update the posterior of $\nu$, and set the $L_n$ accordingly.
Under this choice, the overall complexity of the Gibbs sampling algorithm (Algorithm~\ref{algo_gibbs_sampling}) becomes $$\mathcal{O}(L_nn^3)=\mathcal{O}( n^3\log n).$$

In summary, with the truncated Gibbs sampling algorithm, the total computational complexity of using the $\infty$-GP surrogate model is $\mathcal{O}(n^3\log n )$, introducing only a logarithmic overhead relative to the standard GP complexity of $\mathcal{O}(n^{3})$. This additional cost is often negligible in applications such as the neural network hyperparameter optimization task considered in Section~\ref{sec_exp}, where the cost of each function evaluation far exceeds the computational overhead of the BO procedure itself.

\begin{remark}
In this subsection, we focus on scalability with respect to the number of observations $n$, rather than the ambient dimension $d$ of the decision space. Since the $\infty$-GP model is defined as a mixture over infinitely many GPs, each individual surface inherits the structural properties of a standard GP. As a result, well-established large-scale GP techniques (“large-scale” refers to both large $n$ and large $d$) such as sparse GPs \citep{ranganathan2010online}, Gaussian Markov random fields \citep{l2019gaussian}, random feature approximation \citep{randomfeature} and methods for high-dimensional Bayesian optimization \citep{kriging_ding, snoek2015scalable}, remain applicable within the $\infty$-GP framework. 
\end{remark}

.


 \section{Experiments}\label{sec_exp}
In this section, we evaluate the proposed $\infty$-GP-TS algorithm on a set of benchmark tasks, with a particular focus on BO problems involving ill-conditioned reward distributions. We consider two challenging settings: (i) \textbf{non-stationary} rewards, which are widely recognized as a major challenge when applying BO in machine learning tasks, where classical GP-based methods often fail~\citep{snoek2014input}; and (ii) rewards with \textbf{heavy-tailed} observation noise, which violates the sub-Gaussian noise assumption in classical BO. 
Section~\ref{sec_EXP_TEST} presents results on several synthetic test functions. 
Section~\ref{SEC_EXP_HT} investigates the heavy-tailed setting through a portfolio optimization task. Section~\ref{SEC_EXP_NS} evaluates the algorithm under non-stationary conditions using hyperparameter tuning and neural architecture search tasks.

We compare our method against the following baselines:
\begin{itemize}
    \item BO methods built upon the standard GP surrogate model, including \textbf{GP-TS}~\citep{chowdhury2017kernelized}, \textbf{GP-UCB}~\citep{srinivas2009gaussian}, \textbf{GP-EI}~\citep{ament2023unexpected} and \textbf{GP-KG}~\citep{wu2016parallel}.
    \item \textbf{Warped GP-UCB/TS/EI/KG}:  
    In the non-stationary tasks, all GP-based baselines are augmented with the input warping technique proposed by~\citet{snoek2014input}, which transforms the input space of GP using a Beta cumulative distribution function to better model non-stationarity.
    \item \textbf{Truncated GP-UCB}:
In the heavy-tailed tasks, the classical GP-UCB is modified by applying the truncation strategy from~\citet{ray2019bayesian}, a method shown to be effective for BO under heavy-tailed observation noise.
\end{itemize}

  We set the exploration coefficient in Algorithm~\ref{alg:thompson_sampling} to $\zeta^n = n^{-1/2}$. For a fair comparison, all baseline algorithms are implemented with the same $\zeta^n$-greedy strategy. 
  In Algorithm~\ref{algo_gibbs_sampling}, we run the Gibbs sampler for $B = 500$ iterations, with initialization details provided in Appendix~\ref{SEC:EC:GIBBS}.
For the GP hyperparameters estimation in the baseline methods, we follow the fully Bayesian treatment recommended by~\citet{snoek2012practical} and~\citet{murray2010slice}. Specifically, we perform multiple MCMC runs and use the average of the sampled values as the final hyperparameter estimates. Each MCMC run is also executed for $B = 500$ iterations.

\subsection{Synthetic Test Functions}\label{sec_EXP_TEST}

We evaluate our method on three widely used synthetic benchmark functions: Ackley, Rosenbrock, and StybTang~\citep{xu2024standard}. To construct more challenging scenarios, we consider heavy-tailed and non-stationary variants of each function. 
In the heavy-tailed (HT) setting,  all the functions are corrupted by Weibull-distributed noises. In the non-stationary (NS) setting, each function is modulated by a trigonometric-exponential factor of the form $f_{\mathrm{NS}}(x) = \left(1 + \alpha \sin(x) e^x\right) \cdot f(x),$
which introduces non-stationarity across the domain. Gaussian observation noises are added to all evaluations.

Figure~\ref{fig:six_images} presents the cumulative regret over 100 optimization iterations for different algorithms on the synthetic benchmark functions.
Overall, \(\infty\)-GP-TS consistently outperforms classical GP-based baselines in both HT and NS settings. As predicted by Theorem~\ref{theo_3}~(a), the cumulative regret of \(\infty\)-GP-TS exhibits a sublinear trend and nearly plateaus after 60 iterations on the Rosenbrock and StybTang functions, suggesting convergence to a near-optimal solution. 
In contrast, standard GP-based methods (e.g., GP-EI and GP-KG) display nearly linear cumulative regret on ill-conditioned and complex objectives such as StybTang, Rosenbrock, and Ackley-NS, aligning with the theoretical insights provided in Section \ref{Section_6}.
Additionally, TS with either GP or $\infty$-GP surrogates generally performs well across all acquisition strategies in the majority of the experiments.
 This may be attributed to the exploration behavior inherent in TS, which prevents the search process from getting stuck in suboptimal regions and facilitates global exploration in those complex reward landscapes, as discussed in Section~\ref{section_uni_conver}. 
 However, this benefit comes at the cost of higher variance in regret, particularly when using the $\infty$-GP surrogate. Moreover, as shown in the third row of Figure~\ref{fig:six_images}, when the reward landscapes are relatively simple (without HT or NS effects), $\infty$-GP-TS and GP-TS exhibit nearly identical performance. This is because the $\infty$-GP adaptively adjusts its model complexity based on the observed data, and in simple settings, a single surface is sufficient to model the underlying function.
 
\begin{figure}[h]
    \centering

    \begin{subfigure}[b]{0.3\textwidth}
        \includegraphics[width=\linewidth]{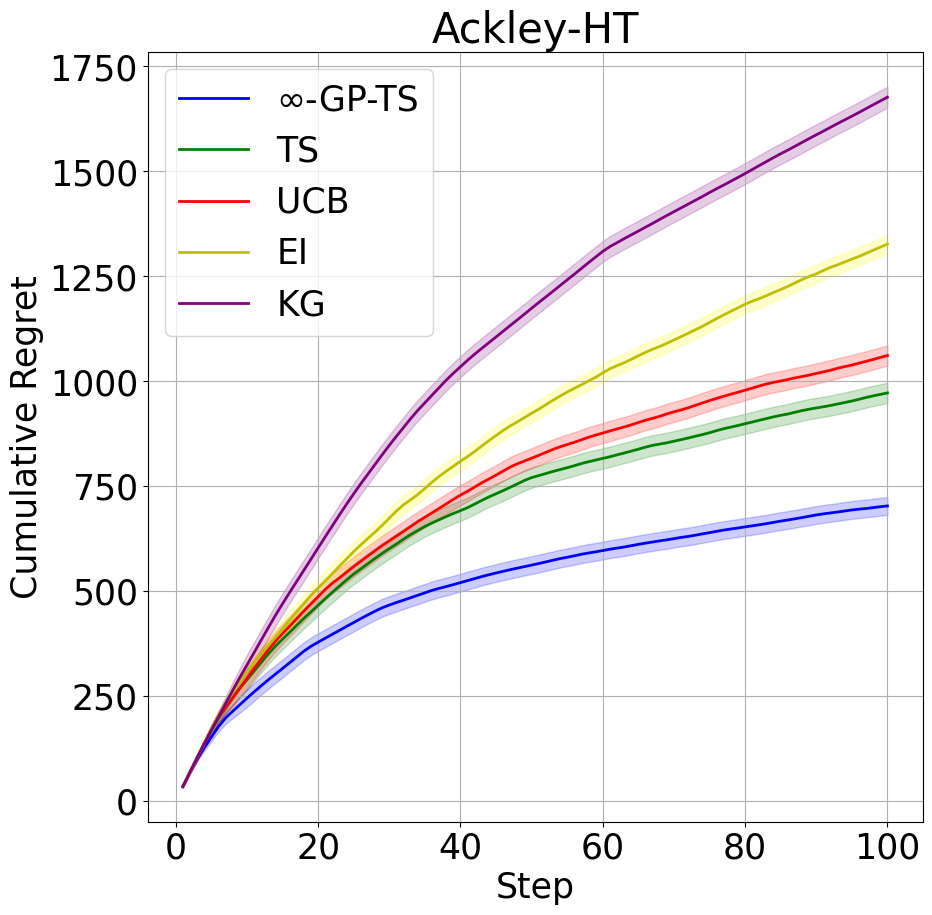}
    \end{subfigure}
    \hfill
    \begin{subfigure}[b]{0.3\textwidth}
        \includegraphics[width=\linewidth]{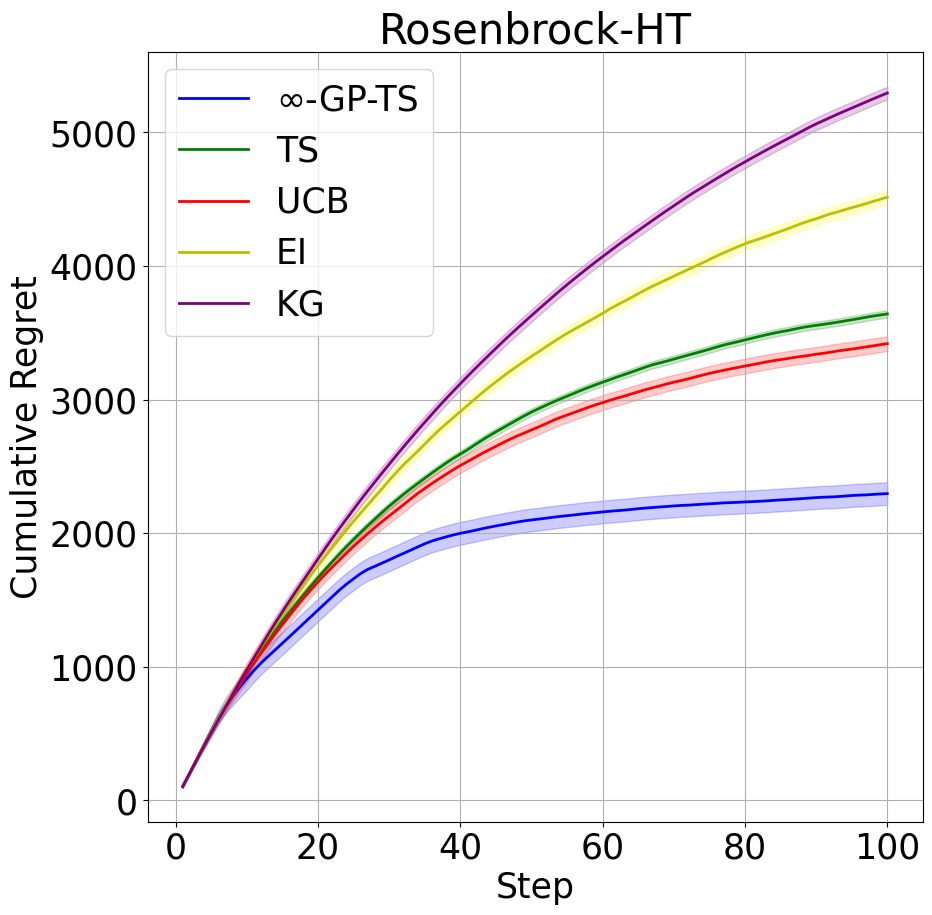}
    \end{subfigure}
    \hfill
    \begin{subfigure}[b]{0.3\textwidth}
        \includegraphics[width=\linewidth]{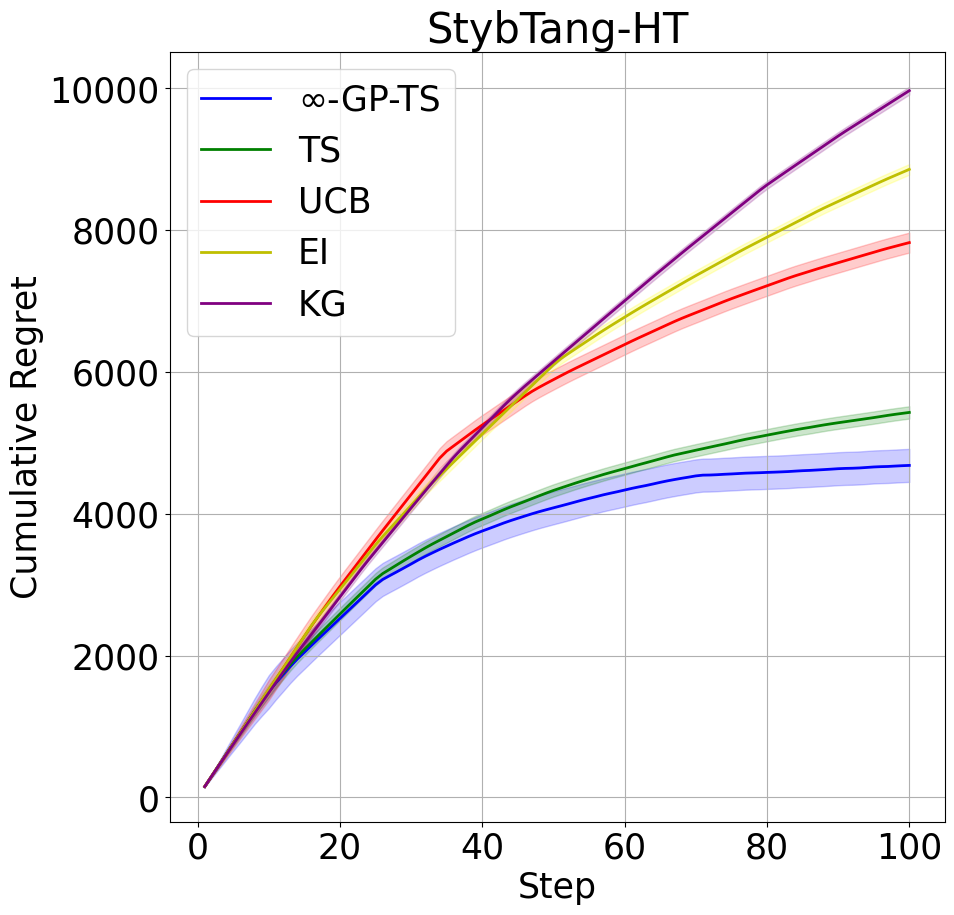}
    \end{subfigure}

    \vspace{0.5cm} 

    \begin{subfigure}[b]{0.3\textwidth}
        \includegraphics[width=\linewidth]{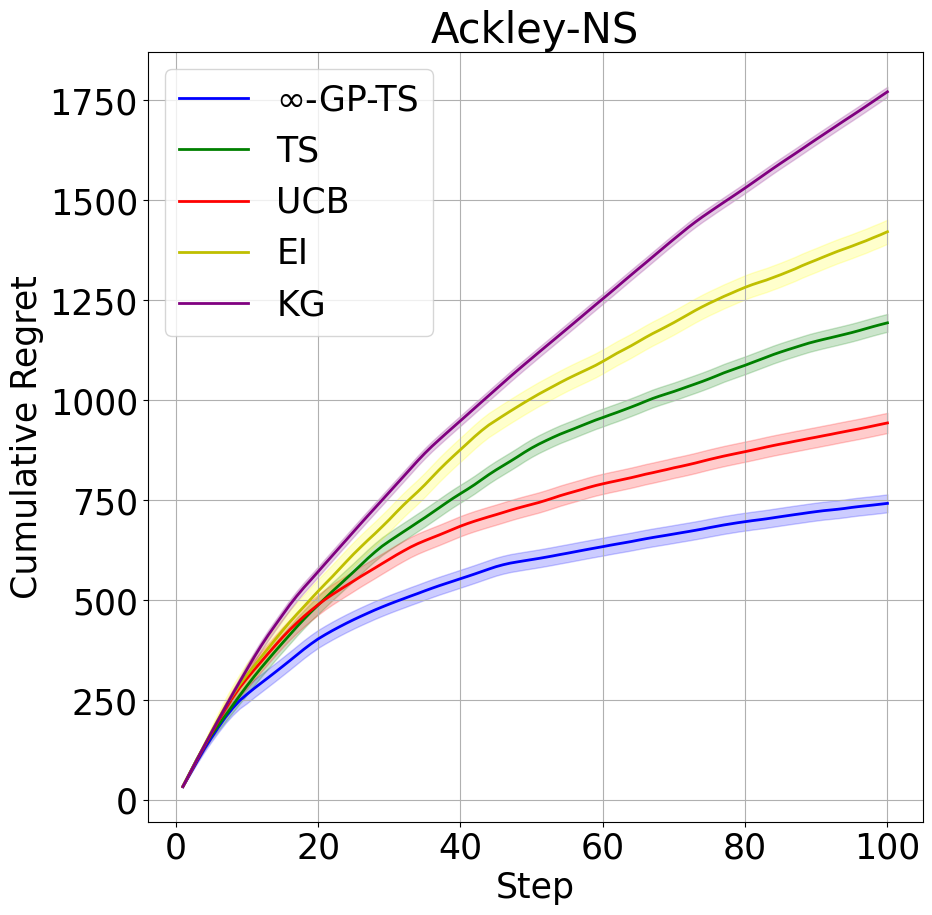}
    \end{subfigure}
    \hfill
    \begin{subfigure}[b]{0.3\textwidth}
        \includegraphics[width=\linewidth]{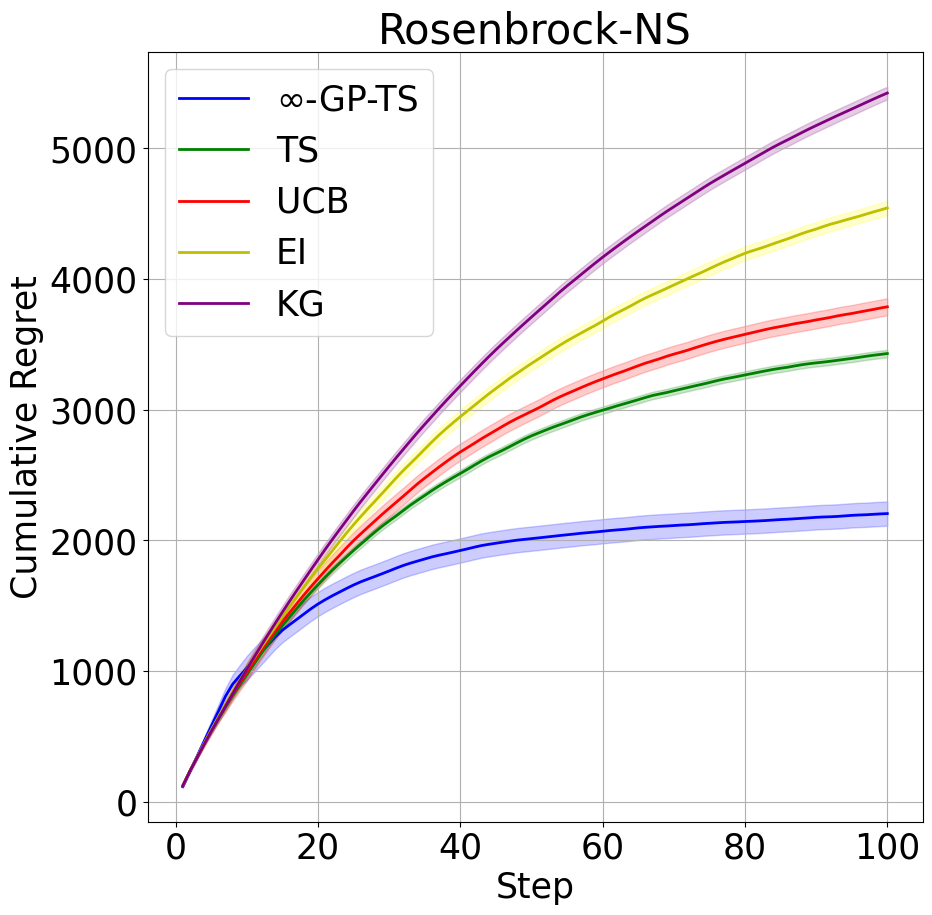}
    \end{subfigure}
    \hfill
    \begin{subfigure}[b]{0.3\textwidth}
        \includegraphics[width=\linewidth]{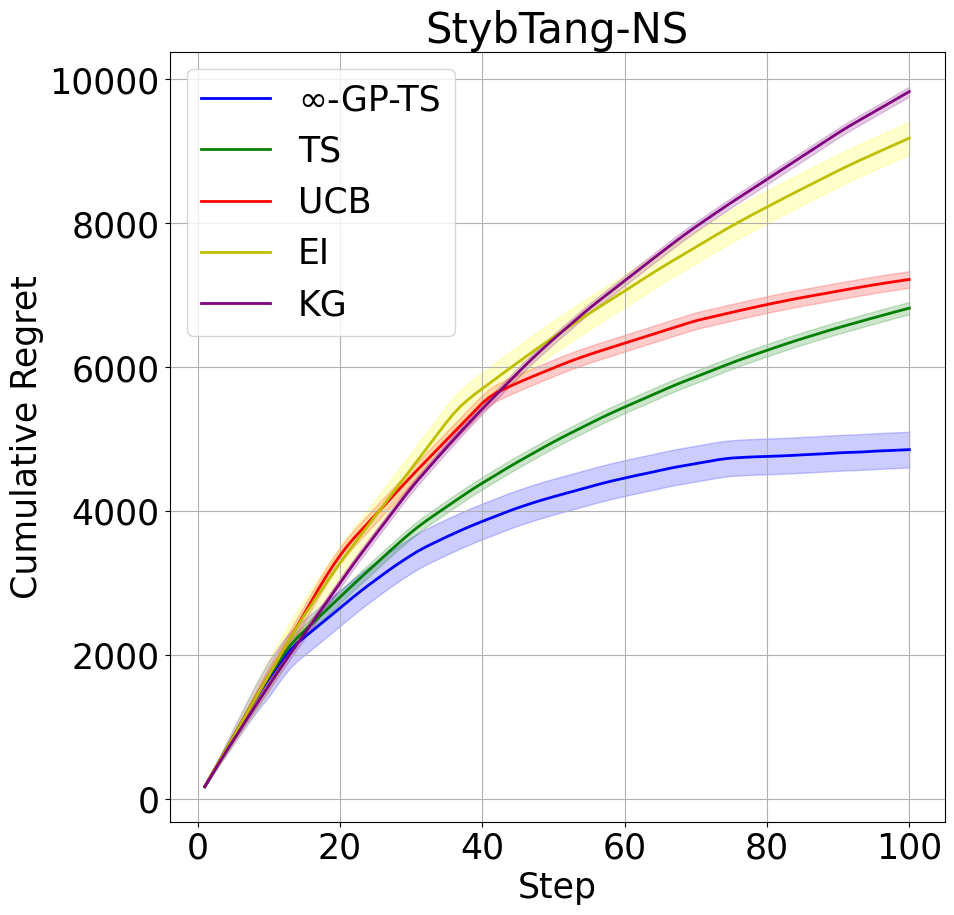}
    \end{subfigure}
    \begin{subfigure}[b]{0.3\textwidth}
        \includegraphics[width=\linewidth]{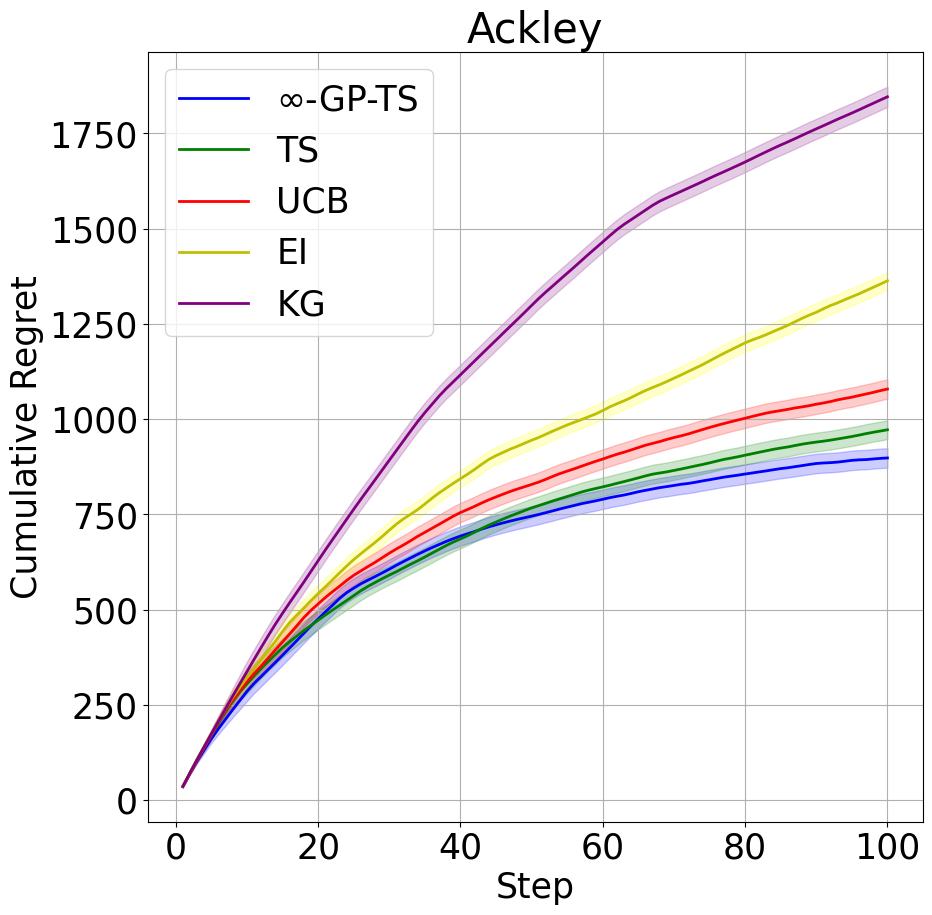}
    \end{subfigure}
    \hfill
    \begin{subfigure}[b]{0.3\textwidth}
        \includegraphics[width=\linewidth]{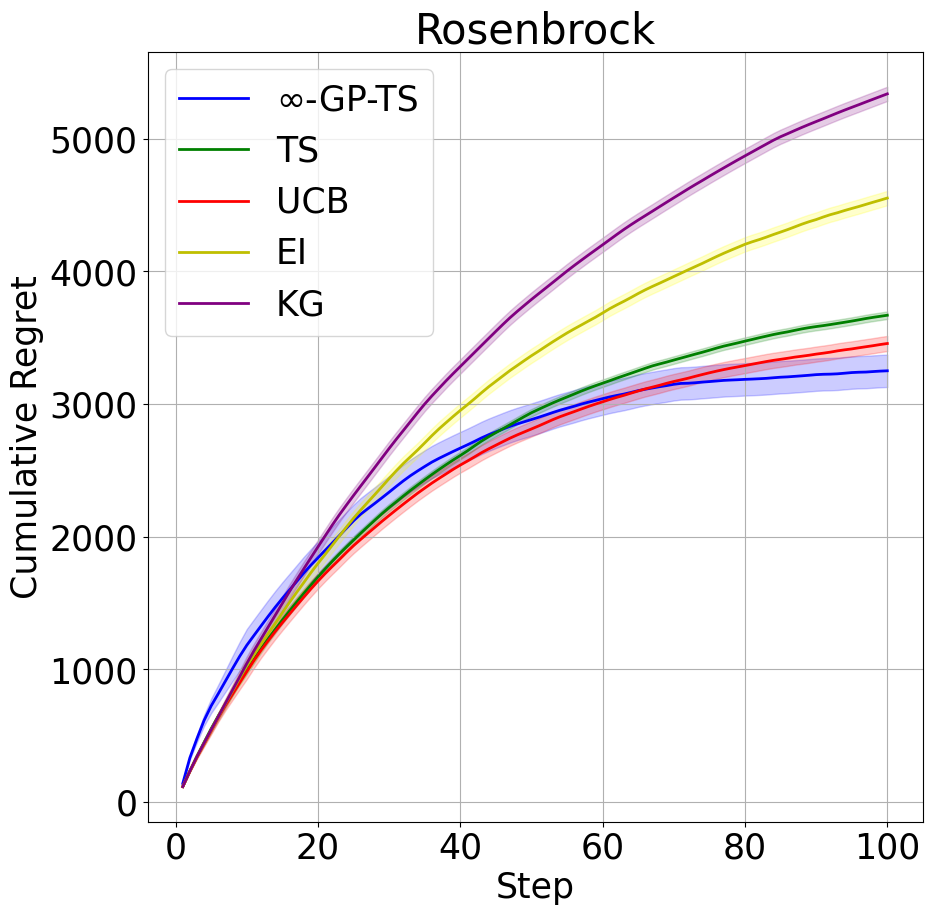}
    \end{subfigure}
    \hfill
    \begin{subfigure}[b]{0.3\textwidth}
        \includegraphics[width=\linewidth]{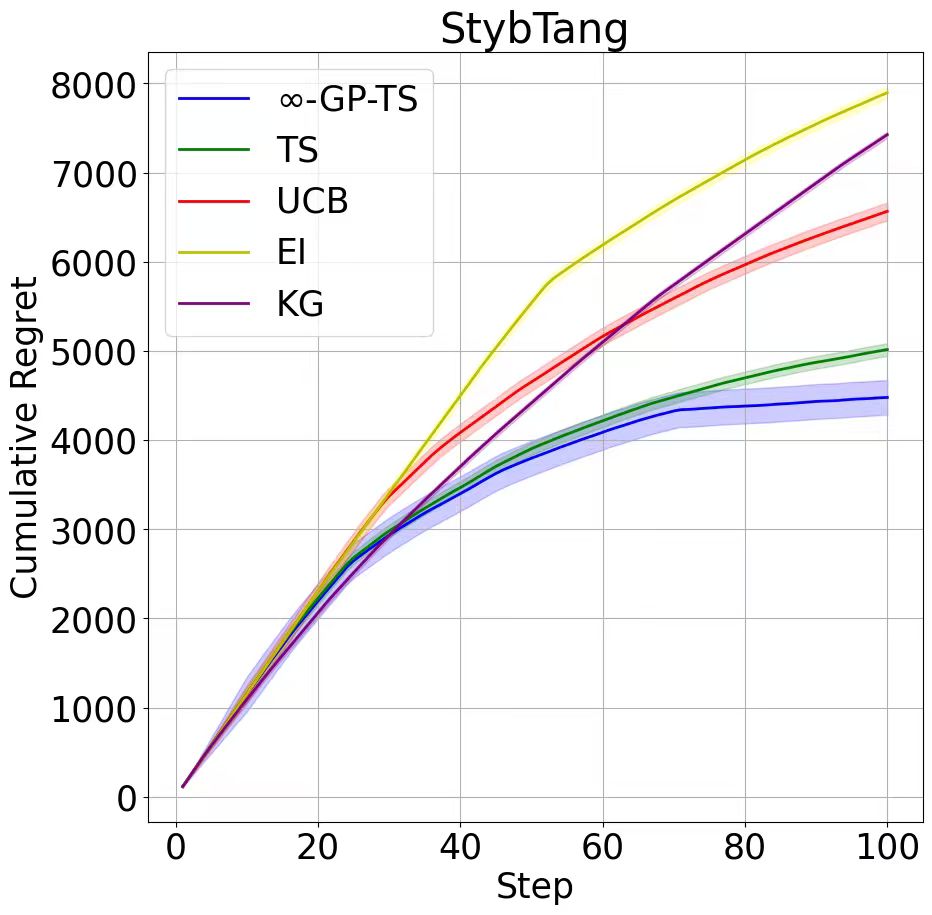}
    \end{subfigure}

    \vspace{0.5cm} 
    \caption{Cumulative regret comparison across BO algorithms}
    \label{fig:six_images}
\end{figure}

Table~\ref{table1:time} reports the wall-clock time of different algorithms, including only the time spent by the BO algorithm itself and excluding the time consumed by function evaluations.
The computational time of $\infty$-GP-TS is comparable to that of the GP-based baselines. Its runtime is often on par with GP-EI and is significantly lower than that of GP-KG, as the latter relies on simulation to estimate the acquisition function. Compared with GP-TS and GP-UCB, $\infty$-GP-TS incurs slightly more computation time, but the difference is not substantial.
There are two main reasons why $\infty$-GP-TS remains computationally efficient.
{First}, the use of the truncated Gibbs sampling Algorithm \ref{algo_gibbs_sampling} introduces minimal additional computational overhead, as analyzed in Section~\ref{section_7}. Compared to GP, the extra cost of using the $\infty$-GP mainly arises from maintaining $L_n$ latent surfaces. However, $L_n$ grows very slowly with $n$, as empirically confirmed in the next paragraph.
{Second}, although $\infty$-GP involves Gibbs sampling, the baseline methods also require hyperparameter estimation through a fully Bayesian approach, which relies on repeated MCMC. Even when MLE is used as an alternative, the associated gradient-based optimization still incurs comparable computational complexity~\citep{snoek2012practical}.
Notably, $\infty$-GP-TS requires only a single Gibbs sampling run per optimization iteration, further reducing its computational burden.

\begin{table}[h]
\centering
\small
\renewcommand{\arraystretch}{0.85}
\setlength{\tabcolsep}{4pt}
\begin{tabular}{l|ccccc}
    \toprule 
    \diagbox[height=3em,width=6em]{Obj.}{Alg.} & GP-TS & GP-UCB & GP-EI & GP-KG & $\infty$-GP-TS \\
    \midrule
    Ackley-HT       & 2773 & 2649 & 3017 & 3388 & 2813 \\
    Ackley-NS       & 2798 & 2535 & 3001 & 3450 & 2912 \\
    Rosenbrock-HT   & 2782 & 2737 & 3210 & 3309 & 2858 \\
    Rosenbrock-NS   & 2689 & 2712 & 2803 & 3383 & 2853 \\
    StybTang-HT     & 2790 & 2619 & 3107 & 3007 & 2901 \\
    StybTang-NS     & 2805 & 2720 & 3030 & 3378 & 3098 \\
    \bottomrule 
\end{tabular}
\caption{Computation time (seconds) for different BO algorithms across test functions.}\label{table1:time}
\end{table}

To better understand the dynamics of the $\infty$-GP model and how it adapts its complexity during the optimization process, Figure~\ref{fig:surface_over_time} illustrates the evolution of the weights $w_l$ associated with each latent surface $\xi^{(l)}$ (see Eq.~\eqref{equation_infi_surface}) in the Ackley-NS task. The optimization is run for 100 iterations, and the weights are updated every 10 iterations. We set the truncation level to $L_n=4$ throughout the optimization. Darker blocks in Figure~\ref{fig:surface_over_time} indicate higher weights.
As predicted in Proposition~\ref{prop_sparse}, the number of surfaces with non-negligible weights increases slowly. During the first 60 iterations, only Surface 1 and Surface 4 carry significant weight, with Surface 1 dominating with a weight close to 1. Weights of surfaces 2 and 3 remain nearly zero. Between iterations 60 and 100, Surface 4 becomes more active, and Surface 2 starts to accumulate some weight, whereas Surface 3 remains negligible throughout.
By iteration 100, the optimization has nearly converged, and only three surfaces (1, 2, and 4) exhibit non-trivial weights. This suggests that two or three surfaces are sufficient to effectively approximate the $\infty$-GP model. This observation supports the use of the truncation technique and demonstrates the scalability of the $\infty$-GP model. In practice, the required number of surfaces \( L_n \) is much smaller than the conservative theoretical choice in Eq.~(\ref{eq_chocie_of_L}), which is designed to ensure worst-case regret guarantees.

\begin{figure}
    \centering
    \includegraphics[width=1\linewidth]{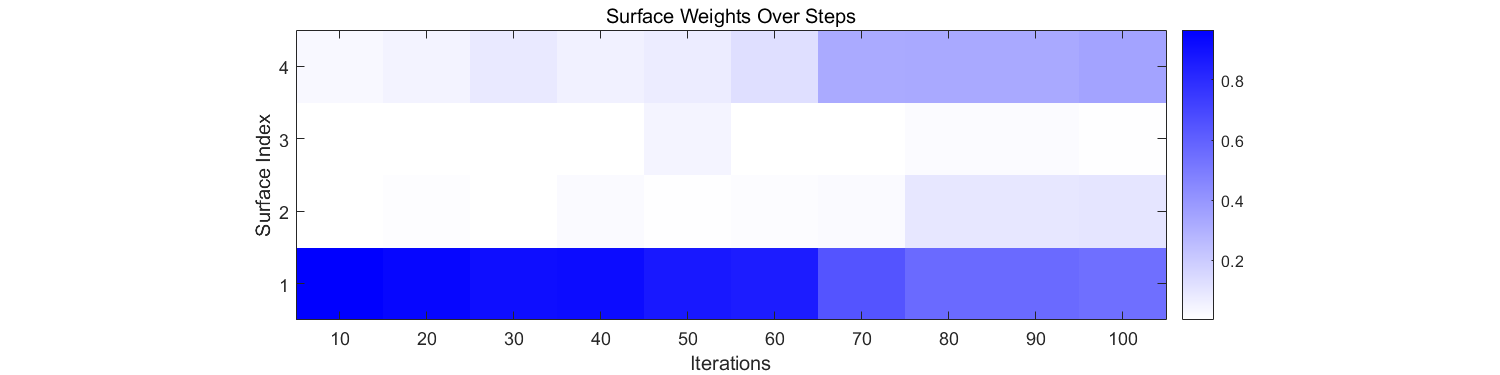}
    \caption{Weights of each surfaces over optimization of Ackley-NS.}
    \label{fig:surface_over_time}
\end{figure}

\subsection{Heavy-Tail Reward: Portfolio Optimization}\label{SEC_EXP_HT}

Heavy-tailed rewards are common in financial applications. We evaluate the performance of our algorithm on the \textit{Portfolio Optimization} task introduced by~\citet{cakmak2020bayesian}, which aims to optimize the hyperparameters of a trading strategy, including risk aversion, trade aversion, and holding cost multipliers. Unlike~\citet{cakmak2020bayesian}, where the objective is to minimize a risk measure, our goal is to maximize the expected reward.
The reward signals are generated using CVXPortfolio~\citep{CVXPort}, a realistic and computationally expensive financial market simulator.

The cumulative regret is shown in Figure~\ref{fig:three_images_row}(a). All methods exhibit sublinear cumulative regret on the Portfolio task. Among them, $\infty$-GP-TS performs significantly better than the others after 60 iterations. However, it does not outperform all baselines during the early iterations. This may be because $\infty$-GP-TS encourages more exploration in the initial phase: it explores not only the uncertainty in function values but also the uncertainty in the reward model itself.
Furthermore, the truncation technique enables Truncated-GP-UCB to achieve better performance than GP-TS and GP-KG. The wall-clock time of the BO algorithm (excluding function evaluation times) is reported in Table~\ref{table3:time}.

\subsection{Nonstationary Rewards: BO in Machine Learning}\label{SEC_EXP_NS}

To evaluate the performance of our method under nonstationary rewards, we consider two representative BO tasks in machine learning: hyperparameter tuning and neural architecture search.

First, we consider the multilayer perceptron (MLP) task from the \textbf{HPOBench} benchmark~\citep{eggensperger2021hpobench}, a widely used collection of hyperparameter optimization problems for machine learning models. Our goal is to optimize the MLP's hyperparameters to maximize its performance.
The problem involves five decision variables: two hyperparameters that determine the depth and width of the network, and three others that control the batch size, regularization strength, and the initial learning rate for the Adam optimizer. The MLP model is trained on the MNIST dataset.

The second task is \textbf{NASBench201}~\citep{nasbench201}, a neural architecture search (NAS) task on the CIFAR-100 dataset. The goal is to identify the best-performing neural network architecture for the given task.
We formulate the NAS-Bench-201 task as a 30-dimensional continuous optimization problem. The search space is defined as a cell structure with 6 directed edges, each associated with 5 candidate operations: skip-connect, zeroize, $1{\times}1$ convolution, $3{\times}3$ convolution, or $3{\times}3$ average pooling.
The decision variable is a continuous vector $x \in [0,1]^{30}$, divided into 6 groups of 5 elements, each group corresponding to one edge. Each group of the continuous decision variable represents soft logits over the 5 operations, and the final architecture is determined by taking the \texttt{argmax} within each group to select the operation.
The accuracy of the resulting architecture is then evaluated on the validation dataset.

\begin{figure}[H]
    \centering
    \begin{subfigure}[b]{0.3\textwidth}
        \centering
        \includegraphics[width=\linewidth]{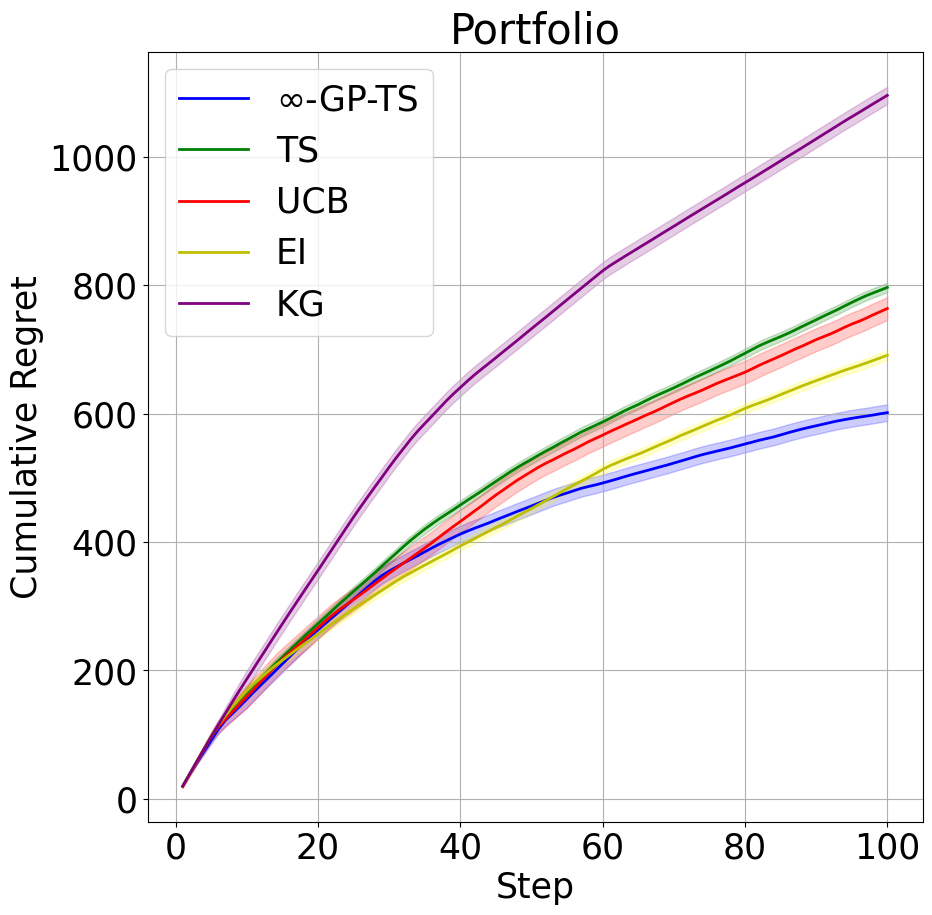}
        \caption{}
    \end{subfigure}
    \hfill
    \begin{subfigure}[b]{0.3\textwidth}
        \centering
        \includegraphics[width=\linewidth]{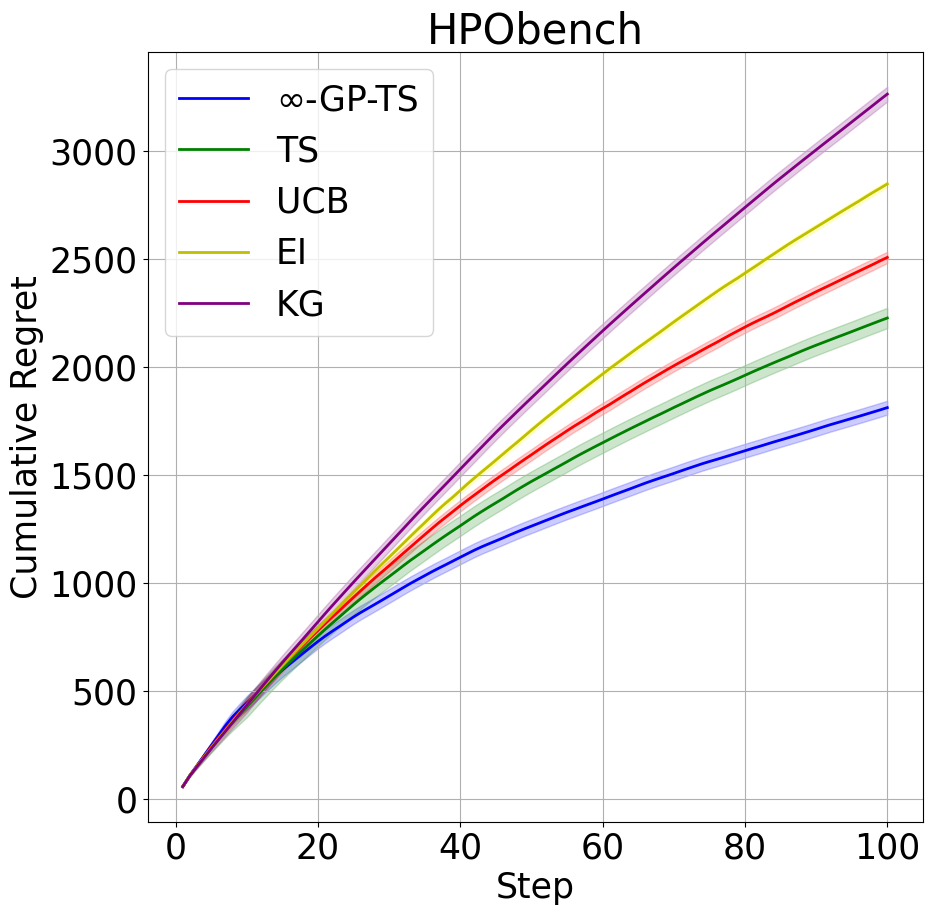}
        \caption{}
    \end{subfigure}
    \hfill
    \begin{subfigure}[b]{0.3\textwidth}
        \centering
        \includegraphics[width=\linewidth]{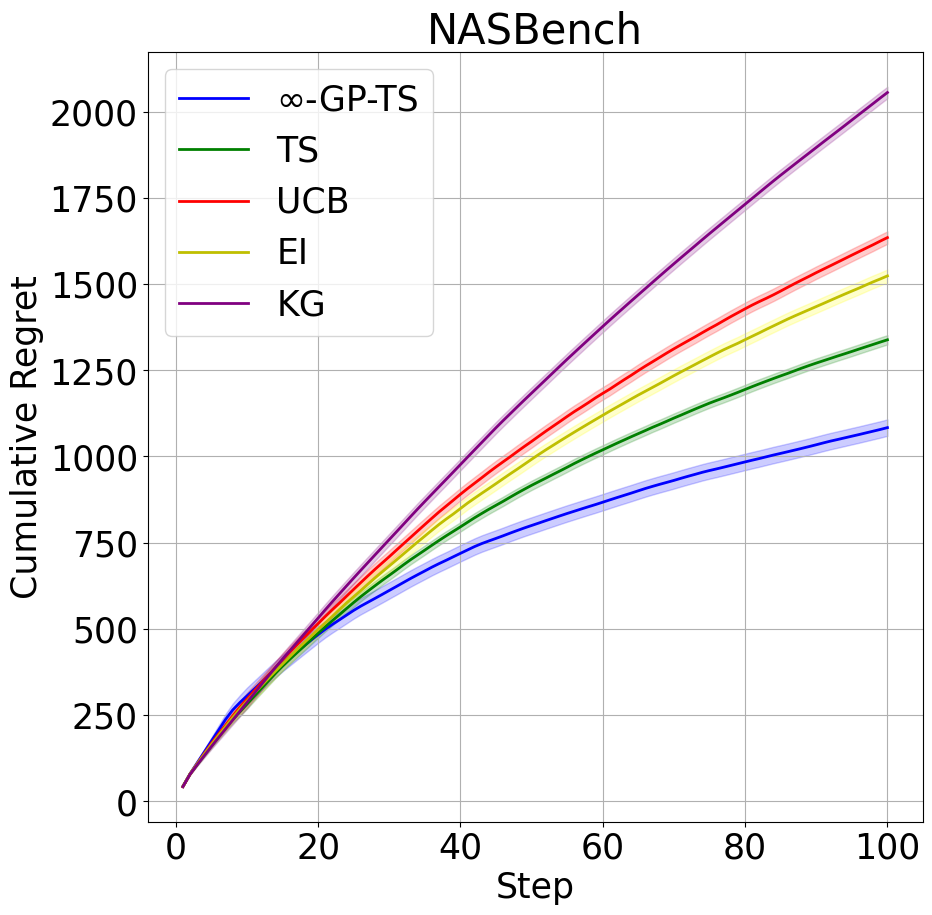}
        \caption{}
    \end{subfigure}

  \caption{Cumulative regret over time. (a) Portfolio Optimization, (b) HPOBench, and (c) NAS-Bench-201.}
    \label{fig:three_images_row}
\end{figure}

The results of \textbf{HPOBench} and \textbf{NASBench201} are shown in Figure~\ref{fig:three_images_row}(b) and Figure~\ref{fig:three_images_row}(c), respectively. Despite incorporating the input warping technique, traditional GP-based methods still perform poorly under nonstationary and complex reward landscapes in these machine learning tasks. In particular, {GP-KG} and {GP-EI} exhibit nearly linear growth in cumulative regret, indicating that GP fails to effectively model and learn these ill-conditioned reward distributions, further confirming the limitations of GPs discussed in Section~\ref{sec_limit_of_convential}.
In contrast, $\infty$-GP-TS can accommodate a wide range of reward distributions, as analyzed in Theorem~\ref{theo_uni_converge}, and achieves sublinear cumulative regret. Moreover, the wall-clock time is reported in Table~\ref{table3:time}.

\begin{table}[h]
\centering
\small
\renewcommand{\arraystretch}{0.85}
\setlength{\tabcolsep}{4pt}
\begin{tabular}{l|ccccc}
    \toprule
    \diagbox[height=3em,width=6em]{Obj.}{Alg.} & GP-TS & GP-UCB & GP-EI & GP-KG & $\infty$-GP-TS \\
    \midrule
    Portfolio       & 3510 & 3309 & 3807 & 3790 & 3411 \\
    HPOBench       & 3337 & 3300 & 3492 &3703 & 3502 \\
    NASBench   & 3920 & 3689 & 3802 & 4202 & 4009 \\
    \bottomrule
\end{tabular}
\caption{Computation time (seconds) for different algorithms across test functions.}\label{table3:time}
\end{table}

\section{Concluding Remarks}

In this paper, we propose a new surrogate model, the $\infty$-GP. Unlike classical GP models that capture only \emph{value uncertainty}, the $\infty$-GP additionally captures \emph{model uncertainty} by placing a prior over the distributional space. This enables $\infty$-GP-TS to explore and exploit not only within the function-value space but also within the space of reward models. As a result, the $\infty$-GP enjoys broad posterior convergence guarantees and achieves, for the first time, no-regret guarantee under a general reward setting.

One limitation of this work is that we focus on the scalability of $\infty$-GP with respect to the number of observations~$n$, rather than the ambient dimension~$d$ of the decision space. An interesting direction for future research is to investigate whether recent large-scale GP techniques can be incorporated into the $\infty$-GP framework to further improve scalability in high-dimensional settings.

\bibliographystyle{informs2014} 
\bibliography{ref} 
\ECSwitch


\ECHead{Appendix}
\renewcommand{\thesection}{EC.\arabic{section}}
\setcounter{section}{0}

\makeatletter
\renewcommand{\theHsection}{EC.\arabic{section}}
\makeatother

\section{Details of Gibbs sampling}\label{SEC:EC:GIBBS}
\begin{algorithm}
\caption{Truncated Gibbs Sampling for $\{\Theta,\bm{\xi_{1:n}}(x_{1:n}),\bm{z}_{1:n}\}$}\label{algo_gibbs_sampling}
\small
\begin{algorithmic}[1]
\State \textbf{Input:} Initialization of $\{\Theta,\bm{\xi_{1:n}}(x_{1:n}),\bm{z}_{1:n}\}$, $\mathcal{H}_n$, iteration number $B$, truncation number $L$
\State \textbf{Output:} Samples $[\hat{\Theta},\bm{\hat{\xi}_{1:n}}(x_{1:n}),\bm{\hat{z}}_{1:n}|\mathcal{H}_n]$, where \(\Theta = \{\beta, \nu, \tau, \sigma^2, \phi\}\).
\For{each iteration \(b = 1, \dots, B\)}
\State \textbf{Step 1: Sampling \(\bm{\xi_{1:n}}(x_{1:n})\):}
\For{each \(l = 1, \dots, L\)}
    \State Compute \(\mathcal{I}^l(\bm{z}_{1:n}) = \text{diag}(\mathbb{I}_{\{z_1 = l\}}, \dots, \mathbb{I}_{\{z_n = l\}})\).
     Sample \(\xi^{(l)}(x_{1:n})\) from the distribution $\xi^{(l)}(x_{1:n}) \sim \mathcal{N}\left( \frac{1}{\tau^2} \Lambda^l \mathcal{I}^l \left( y(x_{1:n}) - x_{1:n}^\top \beta \right), \Lambda^l \right)$,
    where \(\Lambda^l = \left(\Sigma_0^{-1} + \tau^{-2}\mathcal{I}^l\right)^{-1}\) and \(\Sigma_0 = \sigma^2 \rho\).
\EndFor

\State \textbf{Step 2: Sampling  Weights \((w_1, \dots, w_L)\):}
\For{each \(l = 1, \dots, L\)}
    \State Draw \(V_l \sim \text{Beta}(1 + M_l, \nu + \sum_{j=l+1}^{L} M_j)\), where \(M_j = \#\{i : z_i = j\}\). Compute weights by $w_1 = V_1,\ w_l = (1 - V_1)(1 - V_2) \dots (1 - V_{l-1}) V_l,\ w_L = 1 - \sum_{l=1}^{L-1} w_l$
\EndFor

\State \textbf{Step 3: Sampling Latent Labels \((z_1, \dots, z_n)\):}
\For{each observation \(x_i\)}
    \State Sample \(z_i\) from \(\{1, \dots, L\}\) with probabilities proportional to $P(z_i = j) \propto w_j \exp \left\{ -\frac{1}{2\tau^2} \left( y(x_i) - x_i^\top \beta - \xi_j(x_i) \right)^2 \right\}.$
\EndFor

\State \textbf{Step 4: Sampling Hyper-Parameters \(\Theta = \{\beta, \nu, \tau, \sigma^2, \phi\}\):}
\State \textbf{(4a) Sampling \(\nu\):}  \(\nu \sim \mathrm{Gamma}(a_\nu + n-1, b_\nu - \log(w_L))\)
\State \textbf{(4b) Sampling \(\beta\):}  \(\beta \sim \mathcal{N}(\beta_0^n, \Sigma_\beta^n)\), where
$$\Sigma_\beta^n = \left( \Sigma_\beta^{-1} + \tau^{-2} x_{1:n}^\top x_{1:n} \right)^{-1},\  \beta_0^n = \Sigma_\beta^n \left( \Sigma_\beta^{-1} \beta_0 + \tau^{-2} X_{1:n}^\top \bar{y}(x_{1:n}) \right),\ \bar{y}(x_i) = y(x_i) - \xi^{(z_i)}(x_i).$$ 

\State \textbf{(4c) Sampling \(\tau^2\):}  \(\tau^2 \sim \mathrm{InvGamma}(a_\tau^n, b_\tau^n)\), where $a_\tau^n = a_\tau + \frac{n}{2}, \quad b_\tau^n = b_\tau + \frac{1}{2} \sum_{i=1}^{n} \left( \bar{y}(x_i) - x_i^\top \beta \right)^2$.

\State \textbf{(4d) Sampling \(\sigma^2\):} \(\sigma^2 \sim \mathrm{IGamma}(a_\sigma^n, b_\sigma^n)\), where $a_\sigma^n = a_\sigma + \frac{nL}{2}, \quad b_\sigma^n = b_\sigma + \frac{1}{2} \sum_{l=1}^{L} \left( \xi^{(l)}(x_{1:n}) \right)^\top \rho_{\bm{\phi}}^{-1}(x_{1:n}, x_{1:n}) \xi^{(l)}(x_{1:n}).$
\State \textbf{(4e) Sampling \(\phi\):} If the kernel is anisotropic, see (4e) in Section \ref{sec_ec_gibbs}. If the kernel is isotropic: \(\phi\) is sampled from a grid of values \(\{\phi_1, \phi_2, \dots, \phi_M\}\) with probabilities proportional to $$P(\phi_m) \propto \frac{1}{\left[\text{det}(\rho_{\phi_m}(x_{1:n}, x_{1:n}))\right]^{L/2}} \exp \left( -\frac{1}{2 \sigma^2} \sum_{l=1}^{L} \left( \xi^{(l)}(x_{1:n}) \right)^\top \rho_{\phi_m}^{-1}(x_{1:n}, x_{1:n}) \xi^{(l)}(x_{1:n}) \right).$$
\EndFor
\end{algorithmic}
\end{algorithm}
\subsection{Truncated Gibbs Sampling}\label{sec_ec_gibbs}
The pseudocode of our proposed truncated Gibbs sampling method is summarized in Algorithm~\ref{algo_gibbs_sampling}. We provide additional details and explanations below.

\textbf{Step 1: Given $\Theta$, $\bm{z}_{1:n}$ and data $\{x_i,y(x_i)\}_{i=1,\dots,n}$, drawing the table 
$\bm{\xi_{1:n}}(x_{1:n})$.}

 $\mathcal{I}^l(\bm{z}_{1:n})=\text{diag}(\mathbb{I}_{\{z_1=l\}},\cdots,\mathbb{I}_{\{z_n=l\}})$ is the diagonal matrix whose $i$-th entry is equal to 1 if observation $x_i$ lies on the $l$-th surface $\xi^{(l)}$ ($l=1,\cdots,L$). Let $\Sigma_0$ be the covariance matrix induced by the kernel $\sigma^2\rho_{\bm{\phi}}(x,x')=\sigma^2\exp\left( -\sum_{k=1}^d \phi_k (x^{(k)} - x'^{(k)})^2 \right)$ evaluated at $\{x_1,\cdots,x_n\}$.
Then the distribution of $\{\xi^{(l)}(x_{1:n})\}_{l=1}^L$ is given by
\begin{equation}
    \begin{aligned}
        \xi^{(l)}(x_{1:n})&\propto \text{exp}\big\{-\frac{1}{2}\xi^{(l)}(x_{1:n})^T\Sigma^{-1}_0(x_{1:n},x_{1:n})\xi^{(l)}(x_{1:n})\big\}\\
        &\times\text{exp} \big\{ -\frac{1}{2\tau^2}\big(y(x_{1:n})-x_{1:n}^T\beta-\xi^{(l)}(x_{1:n})\big)^T\mathcal{I}^l(\bm{z}_{1:n})\big(y(x_{1:n})-x_{1:n}^T\beta-\xi^{(l)}(x_{1:n})\big)\big\},
    \end{aligned}
\end{equation}
where the first term in the RHS is the prior for $\{\xi^{(l)}(x_{1:n})\}_{l=1}^L$ and the second term is the likelihood.
Then, with $\Lambda^l=\big(\Sigma_0^{-1}+\frac{\mathcal{I}^l}{\tau^2}\big)^{-1}$, we have 
\begin{equation}
    \begin{aligned}
       \xi^{(l)}(x_{1:n})\sim N(\frac{1}{\tau^2}\Lambda^l\mathcal{I}^l(y(x_{1:n})-x_{1:n}^T\beta),\Lambda^l)
    \end{aligned}
\end{equation}


\textbf{Step 2: Drawing $(w_1,\cdots,w_L)$}: according to \citeEC{ishwaran2001gibbs}, $\forall l=1,\cdots,L$, 
first drawing $V_l\sim {Beta}(1+M_l,\nu+\sum_{j=l+1}^{L}M_j)$, where $M_j=\#\{i:z_i=j\}$ is the number of observations lies on the $j$-th surface. $w_1=V_1$, and $w_k=(1-V_1)(1-V_2)\cdots (1-V_{l-1})V_l$, for $l=2,\cdots,L-1$ and $w_L=1-\sum_{l=1}^{L-1}w_l$.

\textbf{Step 3: Drawing $(z_1,\cdots,z_n)$}: 
drawing $z_i$ from $\{1,\cdots,L\}$ with probabilities proportional to 
$$P(z_i=j)=w_j\text{exp}\{-\frac{1}{2\tau^2}\big(y(x_i)-x_i\beta-\xi^{(j)}(x_i))^2\}.$$

\textbf{Step 4: Drawing $\Theta=\{\beta,\nu,\tau^2,\sigma^2,\phi\}$ from $p(\Theta)\prod_{i=1}^n f(\mathcal{H}_n|\Theta)$}

(4a): repeat Step 2 and then draw $\nu$ from its posterior ${Gamma}(a_\nu+n-1,b_\nu-log(w_L))$.

(4b): draw $\beta$ from its posterior, given by $N(\beta_0^n,\Sigma_\beta^n)$, where
$$\Sigma_\beta^n=(\Sigma_\beta^{-1}+\tau^{-2}x_{1:n}^Tx_{1:n})^{-1},$$
$$\beta_0^n=\Sigma_\beta^n(\Sigma_\beta^{-1}\beta_0+\tau^{-2}x^T_{1:n}\bar{y}(x_{1:n}))$$
$$\bar{y}(x_i)=y(x_i)-\xi^{(z_i)}(x_i).$$


(4c) draw 
$\tau^2$ from its posterior distribution, given by $IGamma(a_\tau^n,b_\tau^n)$, where 
$a_\tau^n=a_\tau+\frac{n}{2}$ and $b_\tau^n=b_\tau+\frac{1}{2}(\bar{y}(x_{1:n})-x_{1:n}\beta)^T(\bar{y}(x_{1:n})-x_{1:n}\beta)$ 

(4d) draw $\sigma^2$ from its posterior distribution, given by $IG(a_\sigma^n,b_\sigma^n)$, where $a_\sigma^n=a_\sigma+\frac{nL}{2}$ and $b^n_\sigma=b_\sigma+\frac{1}{2}\sum_{l=1}^L(\xi^{(l)}(x_{1:n}))^T\rho^{-1}_\phi(x_{1:n},x_{1:n})\xi^{(l)}(x_{1:n})$.

(4e) Updating the kernel hyperparameters.
Let $[\phi]$ be the prior for $\phi$. Then the posterior is given by
$$
\frac{[\phi]}{\bigl[\det\bigl(\rho_{\bm{\phi}}(x_{1:n},x_{1:n})\bigr)\bigr]^{L/2}}
\exp \left(
-\frac{\sum_{l=1}^L \bigl(\xi^{(l)}(x_{1:n})\bigr)^{\top}
\rho_{\bm{\phi}}^{-1}(x_{1:n},x_{1:n}) \xi^{(l)}(x_{1:n})}{2\sigma^2}
\right).
$$
If the kernel is isotropic, that is, all components in $\bm{\phi}=(\phi_1,\dots,\phi_d)$ are identical and equal to a common value $\phi$. To sample from the posterior of $\phi$, \citeEC{gelfand2005bayesian} propose an efficient and numerically stable approach: specify a discrete uniform prior $P(\phi=\phi_m)=1/M$, where $\phi_m$, $m=1,\dots,M$, form a grid over $(0,b_\phi]$. Then we sample $\phi_m$ with probabilities proportional to
$$
\frac{1}{\bigl[\det\bigl(\rho_{\phi_m}(x_{1:n},x_{1:n})\bigr)\bigr]^{L/2}}
\exp \left(
-\frac{\sum_{l=1}^L \bigl(\xi^{(l)}(x_{1:n})\bigr)^{\top}
\rho_{\phi_m}^{-1}(x_{1:n},x_{1:n}) \xi^{(l)}(x_{1:n})}{2\sigma^2}
\right).
$$
If the anisotropic kernel is adopted, to estimate the dimension-specific kernel hyperparameters $\bm{\phi}=(\phi_1,\dots,\phi_d)$, we adopt a maximum likelihood approach conditional on the current posterior samples of the latent GP surfaces. Given the posterior realizations $\xi^{(1)}(x),\dots,\xi^{(L)}(x)$ at the evaluated inputs $x_1,\dots,x_n$, we define
$$
[\rho_{\bm{\phi}}]_{ij} = \exp \left(-\sum_{k=1}^d \phi_k (x_i^{(k)} - x_j^{(k)})^2\right).
$$
Then the negative log-likelihood to be minimized is
$$
\mathcal{L}(\bm{\phi}) =
\frac{L}{2}\log\det \rho_{\bm{\phi}} +
\frac{1}{2\sigma^2} \sum_{l=1}^L \xi^{(l)\top} \rho_{\bm{\phi}}^{-1} \xi^{(l)}.
$$
The minimizer $\bm{\phi}^\ast$ (obtained via gradient-based methods) of $\mathcal{L}(\bm{\phi})$ is used as the plug-in estimate for the kernel parameters in subsequent sampling steps.

\subsection{Prior Specification Detail}\label{SEC:PRIOR SPECIFICATION}
\begin{enumerate} 
    \item[$\bullet$] $\tau^2$ and $\sigma^2$: we set $a_\tau=a_\sigma=2$ and thus the mean of the prior is $b_\tau$ and $b_\sigma$. Therefore, prior
information about the variance (for example, a rough estimation of the variance) can be incorporated.
    \item[$\bullet$] $\phi$ has a uniform prior on $(0,b_\phi]$. If the distance between $x_1$ and $x_2$, $\|x_1-x_2\|>\frac{3}{\phi}$, their correlation will decrease to less than 0.05. Therefore, we set $\frac{3}{b_\phi}=d\ \mathrm{max}\|x_1-x_2\|$ with a small $d$ (e.g. 0.01).
\end{enumerate}

\section{Proof of Eq.~\eqref{eq_urm_finite} and Proposition \ref{prop_core_decompose}}\label{sec:ec:prop1}


Generally, the predictive reward distribution for any unexplored solution $x_{n+1}$ is given by  
\begin{equation}
    \begin{aligned}\label{decompose}
        &[y(x_{n+1})|\mathcal{H}_n]=\int\textcolor{blue}{\underbrace{[y(x_{n+1})|\Theta^{(1)},\xi^{(z_{n+1})}]}_{A}}\textcolor{red}{[\Theta^{(1)},\xi^{(z_{n+1})}|\mathcal{H}_n]}\\
        &=\int\int\textcolor{blue}{\underbrace{[y(x_{n+1})|\Theta^{(1)},\xi^{(z_{n+1})}]}_{A}}\textcolor{red}{\underbrace{[\xi^{(z_{n+1})}|\bm{\xi}_{1:n},\bm{z}_{1:n},\Theta^{(2)}]}_{B}\underbrace{[\Theta,\bm{\xi_{1:n}},\bm{z}_{1:n}|\mathcal{H}_n]}_{C}}.
    \end{aligned}
\end{equation}
In the first line of (\ref{decompose}), conditioned on the new surface \( \xi^{(z_{n+1})} (x_{n+1})\) and first-layer hyper parameter $\Theta^{(1)}=(\beta,\tau^2)$, the distribution \( \textcolor{blue}{[y(x_{n+1})|\Theta^{(1)},\xi^{(z_{n+1})}]} \) is $\mathcal{N}(x_{n+1}^\top \beta + \xi^{(z_{n+1})}(x_{n+1}), \tau^2)$. 
In the second line, leveraging the urn scheme, the posterior distribution of the surface \( \textcolor{red}{[\Theta^{(1)}, \xi^{(z_{n+1})} | \mathcal{H}_n]} \) is decomposed into Term B and Term C. Term B represents the prediction of the new surface \( \xi^{(z_{n+1})} \) based on the previously realized surfaces, as shown in (\ref{urn_scheme}). Term C represents the posterior distribution of the hyperparameters and realized surfaces, which requires inference through MCMC.

Next we derive the specific expression of Term B in (\ref{decompose}),  i.e., the Eq.~\eqref{eq_urm_finite}.
Notice that, when performing computations involving stochastic processes over $\mathcal{X}$, in practice, we are working with their finite-dimensional distributions, i.e., the joint distribution over a finite subset of $\mathcal{X}$. 
 Specifically, Term B, \( \textcolor{red}{[\xi^{(z_{n+1})}|\bm{\xi}_{1:n},\bm{z}_{1:n},\Theta^{(2)}]} \), becomes \( \textcolor{red}{[\xi^{(z_{n+1})}(x_{n+1})|\bm{\xi}_{1:n}(x_{1:n}),\bm{z}_{1:n},\Theta^{(2)}]} \), where $\bm{\xi}_{1:n}(x_{1:n})=\{\xi^{(j)}(x_i)\}_{i=1,\cdots,n \ j=1,\cdots,K_n}$ represents the \(K_n \times n\) table of values on the realized surfaces corresponding to the evaluated inputs \(x_{1:n}\).
Then, as discussed earlier, the Term  B can be decomposed as 
\begin{equation}
\begin{aligned}
&\textcolor{red}{   [\xi^{(z_{n+1})}(x_{n+1})|\bm{\xi}_{1:n}(x_{1:n}),\bm{z}_{1:n},\Theta^{(2)}]}\\
&=\int \underbrace{[\xi^{(z_{n+1})}(x_{n+1})|\bm{\xi}_{1:n}(x_{n+1}),\bm{z}_{1:n},\Theta^{(2)}]}_{\text{Step (a): Urn scheme}\ (\ref{urn_scheme})}\underbrace{[\bm{\xi}_{1:n}(x_{n+1})|\bm{\xi}_{1:n}(x_{1:n}),\Theta^{(2)}]}_{\text{Step (b): kriging}},
\end{aligned}
\end{equation}

where the Urn scheme term follows (\ref{urn_scheme}) by evaluating $\xi^{(z_{n+1})}$ and $\bm{\xi}_{1:n}$ at $x_{n+1}$. The kriging term is given by (recall that $\bm{\xi}_{1:n}=\{\xi^{(j)}\}_{j=1}^{K_n}$)
\begin{equation}\label{eq_kriging}
    \underbrace{[{\xi}^{(j)}(x_{n+1})|\bm{\xi}_{1:n}(x_{1:n}),\Theta^{(2)}]}_{\text{kriging}}\sim \mathcal{GP}(\mu_n^{(j)}(x_{n+1}),\sigma_n^2(x_{n+1})),
\end{equation}
 where $ \mu_n^{(j)}(x_{n+1})=\Sigma_0(x_{n+1},x_{1:n})\Sigma_0(x_{1:n},x_{1:n})^{-1}\xi^{(j)}(x_{1:n}),$ and $\sigma_n^2(x_{n+1})$ is given by (\ref{kriging_variance}). 
    Therefore, the term B is given by \begin{equation}
\begin{aligned}\label{eq_urm_finite_2}
  \textcolor{red}{ [\xi^{(z_{n+1})}(x_{n+1})|\bm{\xi}_{1:n}(x_{1:n}),\bm{z}_{1:n},\Theta^{(2)}]}\sim\frac{\nu}{\nu+n}G_0+\sum_{j=1}^{K_n}\frac{n_j}{\nu+n}\mathcal{GP}(\mu_n^{(j)}(x_{n+1}),\sigma_n^2(x_{n+1})). 
\end{aligned}
\end{equation}

\halmos

\section{Proof of Theorem~\ref{theo_uni_converge}: Convergence Analysis of $\infty$-GP} \label{sec_ec_proof_theo1}
Throughout this appendix, to analyze the convergence of the $\infty$-GP model using Bayesian nonparametric theory, it is essential to adopt the perspective that the $\infty$-GP model essentially defines a prior $\Pi_{\mathcal{X}}$ on the space of conditional densities $\mathcal{F}=\{f(\cdot \mid x) : x \in \mathcal{X}\}$, which is ``a distribution over distributions". 
Specifically, this prior $\Pi_{\mathcal{X}}$ is construct by first defining the conditional density as a Gaussian kernel mixture $$f(y\mid x)=\sum_{l=1}^\infty w_l \varphi_{x^\top\beta+\xi^{(l)}(x),\tau^2}(y),$$ where $\varphi_{\mu,\Sigma}$ is the standard Gaussian kernel with mean vector $\mu$ and covariance matrix $\Sigma$,
 and then placing priors on $w_l$, $\tau$, $\beta$ and $\xi^{(l)}$. Moreover, without loss of generality, we set the compact $\mathcal{X}\subset\mathbb{R}^d$ to $[0,1]^d$.

Then let us clarify the definitions of different convergence ``rates''.
Given
\(\mathcal H_n\), the posterior is
\(\Pi_{ \mathcal X}( \cdot\mid\mathcal H_n)\).
Let $\delta$ be any metric on~$\mathcal F$
(e.g.\ Hellinger, TVD, or the
$\pi$-integrated TVD used in this paper).
A sequence $\{\mathcal{E}_n\}_{n\ge1}$ is called a \emph{contraction rate}
if there exists $M>0$ such that
$$\Pi_{\mathcal X}(
     f:\delta(f,f^{\ast})>M\mathcal{E}_n
     |\mathcal H_n
  )
  \to 0$$ almost surely under $\mathbb{P}_{\mathcal{H}_n}$, where $\mathbb{P}_{\mathcal{H}_n}$ denotes the product measure induced by sampling $\mathcal{H}_n = \{(x_i, y_i)\}_{i=1}^n$ i.i.d. from $\pi(x) f^\ast(y \mid x)$.
Moreover, it is important to note that,  the predictive distribution $f_n$ of $\infty$-GP model, defined in Eq.\eqref{decompose2}, is essentially the posterior expectation (which is still a distribution) over $\Pi_\mathcal{X}$,  i.e., $f_n=\int f d\Pi_\mathcal{X}(f|\mathcal{H}_n)$. According to \citetEC[Theorem 2.5 and page.~507]{ghosal2000convergence}, under TVD metric, convergence rate of $\delta_{\mathrm{TV}}(f_n,f^\ast)$ is at least as fast as the contraction rate $\mathcal{E}_n$ (almost surely under $\mathbb{P}_{\mathcal{H}_n}$), a result that holds for general statistical models
and posterior distributions. As noted in Remark~\ref{remark2}, all bounds on $\delta_{\mathrm{TV}}(f_n, f^\ast)$ in this paper are derived via the contraction rate $\mathcal{E}_n$. For notation simplicity, we use $\delta_{\mathrm{TV}}(f_n, f^\ast)$ to represent the contraction rate in the main text, with a slight abuse of notation.





We begin by stating the following lemma, which is a direct extension of Theorem 2 in \citeEC{ghosal1999posterior} to the setting of $\pi$-integrated TVD.
The proof is provided in \citetEC[Theorem 5.9]{pati2013posterior}. Let $\delta_{\pi,\mathrm{TVD}}(f,f')$ denote the $\pi$-integrated TVD metric, defined as $\int_{\mathcal{X}} \delta_{\mathrm{TV}}(f(\cdot | x), f'(\cdot | x))   \pi(x)   dx$. 
\begin{lemma}[\citeEC{ghosal1999posterior}]\label{lemma_ghosal1999}
Suppose the true conditional density $f^\ast$ belongs to the Kullback-Leibler support of the prior $\Pi_\mathcal{X}$, i.e.,
\(
f^\ast \in \operatorname{KL}(\Pi_\mathcal{X}),
\)
and there exists a sequence of subsets (sieves) $\mathcal{F}_n$ of the space of conditional densities over $\mathcal{X}$ such that:
 the covering entropy satisfies
    \begin{equation}\label{eq_ghosal1999_1}
         \log N(\mathcal{E}, \mathcal{F}_n, \|\cdot\|_{1}) = o(n),
    \end{equation}
 \begin{equation}\label{eq_ghosal1999_2}
        \Pi_\mathcal{X}(\mathcal{F}_n^c) \leq c_1 e^{-n c_2},
    \quad \text{for some constants } c_1, c_2 > 0,
    \end{equation}

then the posterior distribution $\Pi_\mathcal{X}(\cdot \mid \mathcal{H}_n)$ is strongly consistent in the $\pi$-integrated TVD neighborhood of $f^\ast$, i.e., for every $\mathcal{E} > 0$,
$$\Pi_\mathcal{X}\left( f : \delta_{\pi,\mathrm{TVD}}(f,f^\ast) > \mathcal{E}  \middle|  \mathcal{H}_n \right) \to 0$$ almost surely as  $ n \to \infty.$
\end{lemma}
Definitions of the Kullback–Leibler (KL) support, covering number $N(\mathcal{E}, \mathcal{F}_n, \|\cdot\|_{1})$ (i.e., the minimal number of $\mathcal{E}$-balls required to cover a function class), and covering entropy $\log N(\mathcal{E}, \mathcal{F}_n, \|\cdot\|_{1})$ (i.e., the logarithm of the covering number) can be found in \citetEC{ghosal1999posterior} and other standard references on Bayesian nonparametrics. Moreover, the condition ``the true conditional density $f^\ast$ lies in the KL support of $\Pi_\mathcal{X}$" is met for the $\infty$-GP model when the ground truth $f^\ast$ satisfies Assumption~\ref{ass_guangzhi}. In particular, it follows directly from Theorem 6.1 of \citeEC{pati2013posterior}.

As mentioned earlier, since the predictive distribution $f_n$ is the posterior expectation over $\Pi_\mathcal{X}$, according to Theorem 2.5 in \citeEC{ghosal2000convergence}, 
the strong consistency of the posterior distribution $\Pi_\mathcal{X}(\cdot|\mathcal{H}_n)$ in $\pi$-integrated TVD metric guarantees the convergence of $f_n$ to $f^\ast$ at the same rate and under the same metric. That is, $ \int_{\mathcal{X}} \delta_{\mathrm{TV}}(f_n(\cdot | x), f^\ast(\cdot | x))   \pi(x)   dx \to 0$ almost surely as $n\to\infty$.
Therefore, to establish Theorem~\ref{theo_uni_converge}, it suffices to apply Lemma~\ref{lemma_ghosal1999}. The key technical step is to construct a suitable sieve $\mathcal{F}_n$ that satisfies the covering entropy condition \eqref{eq_ghosal1999_1} and the prior mass condition \eqref{eq_ghosal1999_2}. Note that, our consistency proof technique is inspired by \citeEC{pati2013posterior}. However, several key differences arise due to the distinctive prior specification in the $\infty$-GP model and the introduction of truncation technique. As a result, both our sieve construction and the covering entropy bounds differ from those in \citeEC{pati2013posterior}.

We set sequences $a_n=\mathcal{O}(\sqrt{n}), 
h_n=e^n,
\ell_n=\mathcal{O}(\frac{1}{\sqrt{n}}),
m_n=\mathcal{O}(n^{\frac{1}{4}}),
M_n=\mathcal{O}(n^{\frac{1}{4}}),
r_n=\mathcal{O}(n^{\frac{1}{4d}})$. Here, note that since our goal is to establish that the covering entropy is $o(n)$, we use the notation $\overset{o(n)}{\precsim}$ to suppress all constants and lower-order terms that are independent of $n$ and do not affect the asymptotic order.

To begin with, to construct a sieve for $\xi^{(l)}$, let $\mathbb{H}_1^a$ denote a unit ball in the RKHS induced by the kernel $\sigma^2 e^{-a\|x-x'\|^2}$ and $\mathbb{B}_1$ is a unit ball in $\mathbb{C}[0,1]^d$.
Following the construction in the proof of Theorem 3.1 in \citeEC{van2009adaptive}, we define
$$
B_n
:=\Bigl(
       M_n\sqrt{\tfrac{r_n}{\delta_n}}\;\mathbb{H}_{1}^{r_n}
       +{\mathcal{E} \ell_n} \mathbb B_{1}
   \Bigr)
   \cup
   \Bigl(
       \bigcup_{a<\delta_n} M_n \mathbb{H}_{1}^{a}
       +{\mathcal{E} \ell_n} \mathbb B_{1}
   \Bigr),
$$
 where $\delta_n=c_3\mathcal{E}/M_n$ for some constant $c_3>0$.
Then, let
$$\Theta_n :=
\Bigl\{
  (\beta,\xi,\tau):
  \|\beta\|\le a_n,\;
  \xi\in B_n,\;
  \ell_n\le\tau\le h_n
\Bigr\}.$$
Finally, we define the sieve $\mathcal{F}_n$  over the space of conditional densities as
$$
\mathcal{F}_n :=
\{
  f:f(\cdot\mid x)=\sum_{l=1}^{\infty}w_l
     \varphi_{x^\top\beta+\xi^{(l)}(x),\tau^2}(\cdot)
  : (\beta,\xi^{(l)},\tau)\in\Theta_n,
    \sum_{l>m_n+1}w_l\le\mathcal{E}
\}.
$$

To bound the covering number of \( \mathcal{F}_n \) under $\ell_1$ distance, we first derive a bound on the \( \ell_1 \) distance between any two densities \( f_1, f_2 \in \mathcal{F}_n \). For any \( x \in \mathcal{X} \), we have: there exists a $c_{n}=\mathcal{O}(\frac{1}{\ell_n})$ such that
\begin{equation}\label{eq_covering_number1}
\|f_1(\cdot|x) - f_2(\cdot|x)\|_1 
\le 
\sum_{l \le m_n} |w_{l,1} - w_{l,2}|
+ c_{n} \sum_{l \le m_n}
d_{\mathrm{SS}}\big(
(\beta_1, \xi_1^{(l)}, \tau_1), 
(\beta_2, \xi_2^{(l)}, \tau_2)
\big)
+ 2\mathcal{E},
\end{equation}
where the single-site metric is defined by
\(
d_{\mathrm{SS}}\big((\beta_1, \xi_1, \tau_1), (\beta_2, \xi_2, \tau_2)\big)
:= \|\beta_1 - \beta_2\| + \|\xi_1 - \xi_2\|_\infty + |\tau_1 - \tau_2|.
\)
This inequality is obtained as follows: we first bound the \( \ell_1 \) distance between two densities by
\(
\|f_1 - f_2\|_1 
\le 
\sum_{l = 1}^{m_n} |w_{l,1} - w_{l,2}| 
+ \sum_{l = 1}^{m_n} \|\varphi_{\mu_1,\tau_1} - \varphi_{\mu_2,\tau_2}\|_1 
+ 2\mathcal{E},
\)
where the $2\mathcal{E}$ term accounts for the approximation error due to truncating the infinite mixture at \( m_n \) components. 
Next, we apply the Lipschitz continuity of the Gaussian kernel with respect to both the mean and the variance parameters. Specifically, since $\tau\geq \ell_n$,
we have 
\(
\|\varphi_{\mu_1, \tau_1} - \varphi_{\mu_2, \tau_2}\|_1 
\le 
c_n \left(|\mu_1 - \mu_2| + |\tau_1 - \tau_2|\right)
\) for some $c_n=\mathcal{O}(\frac{1}{\ell_n})$. Finally, 
with
\(
|\mu_1(x) - \mu_2(x)| 
\le 
\|\beta_1 - \beta_2\| \cdot \|x\| + \|\xi_1^{(l)} - \xi_2^{(l)}\|_\infty,
\)
we obtain \eqref{eq_covering_number1}.

Therefore, according to Eq.~\eqref{eq_covering_number1}, a $4\mathcal{E}$–cover of $\mathcal{F}_n$ can be constructed by separately covering the components as follows:  \\
\emph{(i)} the weight vector $\{w_1,\ldots,w_{m_n}\} \in \Delta_{m_n-1}$ with precision $\mathcal{E}$,  \\
\emph{(ii)} each GP path $\xi^{(l)}$ in $B_n$, and  \\
\emph{(iii)} each pair $(\beta, \tau)$,  
such that for each parameter triplet $(\beta, \xi^{(l)}, \tau)\in\Theta_n$, the covering accuracy is controlled within $\mathcal{E}/(c_{n} m_n)$.

For \emph{(i)}, the covering entropy of a $(m_n - 1)$-simplex $\Delta_{m_n-1}$ under the $\|\cdot\|_{1}$ metric is bounded by
$$
\log N\left(\mathcal{E}, \Delta_{m_n - 1}, \|\cdot\|_1\right)
\lesssim m_n \log\left(\frac{m_n}{\mathcal{E}}\right).
$$

For \emph{(ii)}, covering each $\xi^{(l)} \in B_n$ ($l=1,\cdots,m_n$) under the $\|\cdot\|_\infty$ norm, according to the proof of Theorem 3.1 in \citetEC{van2009adaptive}, yields
$$
\log N\Bigl(\frac{\mathcal{E}}{c_{n}m_n},B_n,\|\cdot\|_\infty\Bigr)
\overset{o(n)}{\precsim} 
 r_n^{d}
  \Bigl\{
      \log(
        \frac{M_n m_n \ell_n\sqrt{r_n/\delta_n}}
              {\mathcal{E} \ell_n})
  \Bigr\}^{d+1}
+ \log(\frac{M_n m_n \ell_n}{\mathcal{E} \ell_n}).
$$

For \emph{(iii)}, according to Lemma 4.1 in \citeEC{india}, the covering entropy of each $(\beta,\tau)$ is 
$
\overset{o(n)}{\precsim}\log\big[
  (a_n/\ell_n)^d + \log(h_n/\ell_n) + 1\big]$.

Summing the three contributions yields the following bound on the covering entropy (since there are $m_n$ such triplets $(\beta,\xi^{(l)},\tau)$, the covering entropies in \emph{(ii)} and \emph{(iii)} are multiplied by $m_n$):
\begin{equation}
    \begin{aligned}
      &  \log N\bigl(F_n,4\mathcal{E},\|\cdot\|_1\bigr)\\&\overset{o(n)}{\precsim} m_n r_n^{d}
  \Bigl\{
      \log(
        \frac{M_n m_n\sqrt{r_n/\delta_n}}
              {\mathcal{E} })
  \Bigr\}^{d+1}+m_n\log(\frac{M_nm_n }{\mathcal{E} })
\\
  & +m_n\log\big[
  (a_n/\ell_n)^d + \log(h_n/\ell_n) + 1\big]+ m_n\log\frac{m_n}{\mathcal{E}}=o(n).
    \end{aligned}
\end{equation}
Therefore, condition \eqref{eq_ghosal1999_1} is satisfied.

Moreover, notice that
\begin{equation}
    \begin{aligned}
        &\Pi_\mathcal{X}(\mathcal{F}_n^c)\leq m_n  \mathbb{P}(\Theta_n^c)+\mathbb{P}\left(\sum_{l=m_n+1}^\infty w_l > \mathcal{E}\right) \\
        &\leq m_n \left\{ \mathbb{P}(\|\beta\|>a_n) + \mathbb{P}(\tau\notin[\ell_n,h_n]) + \mathbb{P}(\xi\in B_n^c) \right\} + \mathbb{P}\left(\sum_{l=m_n+1}^\infty w_l > \mathcal{E}\right).
    \end{aligned}
\end{equation}
It is straightforward to verify that both $\mathbb{P}(\|\beta\|>a_n)$ and $\mathbb{P}(\tau\notin[\ell_n,h_n])$ are bounded by $c_4 e^{-c_5 n}$ for some constants $c_4, c_5 > 0$. Furthermore, by Eq. (5.3) of \citetEC{van2009adaptive}, the tail probability $\mathbb{P}(\xi\in B_n^c)$ is also exponentially small in $n$. For the final term, note that we employ a truncation technique in the $\infty$-GP model by setting all $w_l = 0$ for $l > L_n = \mathcal{O}( \log n)$, and we choose $m_n = \mathcal{O}(n^{1/4})$, which ensures that
$
\mathbb{P}\left(\sum_{l=m_n+1}^\infty w_l > \mathcal{E}\right) 
$ decreases faster than any exponential rate in $n$. Therefore, condition \eqref{eq_ghosal1999_2} is satisfied. The proof of Theorem~\ref{theo_uni_converge} is concluded.

\halmos

\section{Proof of Lemma \ref{lemma_central}} \label{sec_ec_proof_lemma1}

In Lemma \ref{lemma_central}, we assume that $\mathcal{X}$ is finite in order to ensure that the TVD is well defined over finite-dimensional distributions.
We first show the expected reward gap between two Thompson Sampling policies is controlled by the TVD between their induced action distributions. Let \( x_{n+1} \sim \pi^{n+1}_f \) and \( x_{n+1}' \sim \pi^{n+1}_{f'} \), where \( \pi^{n+1}_f \) and \( \pi^{n+1}_{f'} \) are the TS policies induced by surrogate models \( f \) and \( f' \) given $\mathcal{H}_{n}$, respectively. Then we have
\begin{equation}\label{eq:tv_bound}
\begin{aligned}
    &\mathbb{E}\left[ y(x_{n+1}) - y(x_{n+1}') \right] 
    =\mathbb{E}_{\mathcal{H}_{n}}\mathbb{E}_{\pi_f^{n+1},\pi^{n+1}_{f'}} \mathbb{E}_{y\sim f^\ast}\left[ y(x_{n+1}) - y(x_{n+1}') \right]  \\
    &=\mathbb{E}_{\mathcal{H}_{n}}\mathbb{E}_{\pi_f^{n+1},\pi^{n+1}_{f'}} \left[ \mu^\ast(x_{n+1}) - \mu^\ast(x_{n+1}') \right]  \\
    &= \mathbb{E}_{\mathcal{H}_{n}}\sum_{x \in \mathcal{X}} \sum_{x' \in \mathcal{X}} \left( \pi^{n+1}_f(x) \mu^\ast(x) - \pi^{n+1}_{f'}(x') \mu^\ast(x') \right)  \\
    &= \mathbb{E}_{\mathcal{H}_{n}}\sum_{x \in \mathcal{X}} \left( \pi^{n+1}_f(x) - \pi^{n+1}_{f'}(x) \right) \mu^\ast(x)  \\
    &\leq \mathbb{E}_{\mathcal{H}_{n}}\sum_{x\in\mathcal{X}} \left| \pi^{n+1}_f(x) - \pi^{n+1}_{f'}(x) \right| \cdot \left| \mu^\ast(x) \right| \leq 2\mu^\ast_{\max} \mathbb{E}_{\mathcal{H}_{n}} \delta_{\mathrm{TV}}\left( \pi^{n+1}_f, \pi^{n+1}_{f'} \right)
    . 
\end{aligned}
\end{equation}
The last line follows from Hölder's inequality and the definition of TVD. 

Next, we relate the TVD between policies to that between posterior distributions of rewards. Let \( g(\cdot) \) be the \texttt{argmax} function over a finite action set \( \mathcal{X} \). In case of ties, we define \( g(\cdot) \) to select the smallest index among the maximizers to ensure it is well-defined.
Each TS policy is induced by sampling from the posterior distribution over \( \mu^\ast(x) \) and applying the \texttt{argmax} rule. 
That is, each draw $x_{n+1}\sim\pi_f^{n+1}$ satisfies $x_{n+1}=g(\hat{\mu}_f^n(\cdot|\mathcal{X}))$, where $\hat{\mu}_f^n(\cdot|\mathcal{X})\sim \mu_f^n(\cdot|\mathcal{X})$ is a sample drawn from \( \mu^{n}_f(\cdot \mid \mathcal{X}):=[\mu^\ast(\cdot|\mathcal{X})|\mathcal{H}_n] \), i.e., the posterior distribution over the expected reward under surrogate model \( f \).

To bound the regret of using surrogate model $f$ using \eqref{eq:tv_bound}, we need to
replace $f'$ in \eqref{eq:tv_bound} by the true reward $f^\ast$. 
Since the ground-truth best policy $\pi^{n+1}_{f^\ast}=\delta_{x^\ast}$ is a Dirac located at fixed $x^\ast=g(\mu^\ast(x))$, it follows that
\begin{equation}\label{eq:tv-dirac}
\delta_{\mathrm{TV}} \left(\pi^{n+1}_f,\delta_{x^\ast}\right)
= 1-\pi^{n+1}_f(x^\ast).
\end{equation}
Let $\Delta:=\mu^\ast(x^\ast)-\max_{x\in\mathcal{X}\setminus\{ x^\ast\}}\mu^\ast(x)$. Then it is straightforward to verify that
\begin{equation}\label{eq:misselection}
1-\pi^{n+1}_f(x^\ast)\le\ 
\mathbb P \left(\sup_{x\in\mathcal X}\big|\hat\mu_f^n(x)-\mu^\ast(x)\big|\ge\frac{\Delta}{2}\ \Bigm|\ \mathcal H_n\right).
\end{equation}
Combining \eqref{eq:tv_bound}, \eqref{eq:tv-dirac} and \eqref{eq:misselection} and using Markov inequality, we obtain
\begin{equation}\label{eq:inst-regret-prob}
\mathbb E r_{n+1}\ \le\ 2\mu^\ast_{\max}\ \mathbb E_{\mathcal H_n} \left[1-\pi^{n+1}_f(x^\ast)\right] 
 \le \frac{4\mu^\ast_{\max}}{\Delta}\mathbb E_{\mathcal H_n} \left[\sup_{x\in\mathcal X}\big|\hat\mu_f^n(x)-\mu^\ast(x)\big|\right].
\end{equation}
Next, we connect the bound in \eqref{eq:inst-regret-prob} to the TVD between the surrogate model and the true reward model. This step requires a truncation argument: we first restrict the integrals to $|y|<B$ from some $B>0$, and then control the contribution of the tail outside $|y|>B$. According to Algorithm~\ref{alg:thompson_sampling}, a realization $\hat\mu_f^n$ drawn from $\mu_f^n$ can be equivalently written as $\hat\mu_f^n=T(\hat f^{n})$, where $\hat f^{n} \sim \Pi_{\mathcal X}(\cdot \mid \mathcal H_{n})$ is drawn from the posterior over $\mathcal{F}$ (the posterior $\Pi_\mathcal{X}(\cdot|\mathcal{H}_{n})$ is defined in Section~\ref{sec_ec_proof_theo1}), and $T: \mathcal{F} \to \mathbb{R}^{\mathcal X}$ is the mean operator defined by $T(f)(x) = \int y  \mathrm{d}f(y \mid x)$. 
For $B>0$, and $\forall x$, we have
\begin{equation}\label{eq:mean-decomp}
\begin{aligned}
  &  \big|\hat\mu_f^n(x)-\mu^\ast(x)\big|
=\big|\int y  d\hat f^n(y\mid x)-\int y  df^\ast(y\mid x)\big|
\\
&\le \big|\int y \mathbb{I}_{|y|\le B} (d\hat f^n(y\mid x)-  df^\ast(y\mid x))\big|+ T_B^n(x)+T_B^\ast(x)\\
&\le\ 2B \delta_{\mathrm{TV}} \big(\hat f^n(\cdot\mid x),f^\ast(\cdot\mid x)\big)+ T_B^n(x)+T_B^\ast(x),
\end{aligned}
\end{equation}
where the tail remainders
$$T_B^n(x):=\int (|y|-B)_+  df^n(y\mid x)$$ and $$T_B^\ast(x):=\int (|y|-B)_+  df^\ast(y\mid x).$$ Moreover, by the data processing inequality (DPI, which states that any measurable deterministic transformation cannot increase the TVD between two measures, as it only makes the distributions harder to distinguish), we have $\sup_{x\in\mathcal X}\delta_{\mathrm{TV}} \big(f^n(\cdot\mid x), f^\ast(\cdot\mid x)\big) \le\ \delta_{\mathrm{TV}} \big(f^n(\cdot\mid \mathcal X), f^\ast(\cdot\mid \mathcal X)\big).$ Then,
taking the supremum over $x$ in \eqref{eq:mean-decomp} gives
\begin{equation}\label{eq:mean-sup}
\sup_{x\in\mathcal X}\big|\hat\mu_f^n(x)-\mu^\ast(x)\big|
 \le 2B \delta_{\mathrm{TV}} \big(\hat f^n(\cdot\mid \mathcal X),f^\ast(\cdot\mid \mathcal X)\big)+\sup_x T_B^n(x) +\sup_x T_B^\ast(x).
\end{equation}

 Suppose that the contraction rate of $\Pi_\mathcal{X}$ (see the definition of contraction rate in Section \ref{sec_ec_proof_theo1}) is $\mathcal{E}_{n}$, i.e., \(
  \Pi_{\mathcal X}(
     f:\delta_{\mathrm{TV}}(f(\cdot\mid \mathcal X),f^{\ast}(\cdot\mid \mathcal X))>M\mathcal{E}_n
     |\mathcal H_n
  )
  \to 0
\) in $\mathbb{P}_{\mathcal{H}_n}$ probability for some $M>0$, then it follows immediately that $\mathbb E_{\mathcal H_n} \delta_{\mathrm{TV}} \big(\hat f^{n}(\cdot\mid \mathcal X),\ f^\ast(\cdot\mid \mathcal X)\big)= O(\mathcal{E}_n)$.

Next, we show that the remainder $\sup_x T_B^n(x) +\sup_x T_B^\ast(x)$ is negligible. With the tail assumption ``$\forall x\in\mathcal X$, there exist constants $a_1,a_2,a_3,\gamma>0$ such that for all $|y|>a_3$, $\max\big\{f^\ast(y\mid x),f_n(y\mid x)\big\}\le\ a_1\exp\big(-a_2|y|^{\gamma}\big)$", we have 
that there exist constants $C,c>0$ such that $$\sup_{x\in\mathcal X}T_B^n(x)+\sup_{x\in\mathcal X}T_B^\ast(x)\le C \exp \big(-c B^\gamma\big)$$ for all $B>a_3$. This follows by expressing the remainder as an integral of the tail probability: 
$$T_B^\ast(x)
=\int_B^\infty \int_{|y|>t} f^\ast(y\mid x)dydt
\le\int_B^\infty C' e^{-c' t^\gamma}dt
\le C e^{-c B^\gamma},$$
and the same bound holds for $T_B^n(x)$.

Therefore, taking conditional expectation of \eqref{eq:mean-sup} and then expectation over $\mathcal H_n$ yields
\begin{equation}\label{eq:mean-rate}
\mathbb E_{\mathcal H_n}\Big[\sup_{x\in\mathcal X}\big|\hat\mu_f^n(x)-\mu^\ast(x)\big|\Big]
\ \le\ 2B\cdot O(\mathcal{E}_n)\ +\ C e^{-c B^\gamma}.
\end{equation}
Choose $B_n=(\frac{1}{c} \log\frac{1}{\mathcal{E}_n})^{ 1/\gamma},$
then, up to logarithmic factors,  \eqref{eq:mean-rate} becomes
\begin{equation}\label{eq:mean-final}
\mathbb E_{\mathcal H_n}\Big[\sup_{x\in\mathcal X}\big|\hat\mu_f^n(x)-\mu^\ast(x)\big|\Big]
\ =\ O (\mathcal{E}_n ).
\end{equation}
 Recall that $\delta_{\mathrm{TV}}\left( f_n(\cdot \mid \mathcal{X}),\ f^\ast(\cdot \mid \mathcal{X}) \right)$ is almost surely bounded by the contraction rate $\mathcal{E}_n$, and as noted in Remark \ref{remark2}, throughout the main text we use $\delta_{\mathrm{TV}}\left( f_n(\cdot \mid \mathcal{X}),\ f^\ast(\cdot \mid \mathcal{X}) \right)$ to denote the contraction rate $\mathcal{E}_n$. Combing \eqref{eq:inst-regret-prob} and \eqref{eq:mean-final}, we have $$r_{n+1}\precsim \mu^\ast_{\max}  \delta_{\mathrm{TV}}\big(f_n(\cdot \mid \mathcal{X}), f^\ast(\cdot \mid \mathcal{X})\big).$$

\begin{remark}\label{remark_EC}
 Let $\mathcal{X}=\{x_1,\cdots,x_{|\mathcal{X}|}\}$. In the regret bound presented in Lemma~\ref{lemma_central}, we may safely replace the true reward distribution $f^\ast(\cdot \mid \mathcal{X})$ with its decoupled version $\tilde{f}^\ast(\cdot \mid \mathcal{X}) := \prod_{i=1}^{|\mathcal{X}|} f^\ast(\cdot \mid x_i)$, which discards any correlation among the inputs $\{x_i\}_{i=1}^{|\mathcal{X}|}$. This substitution is justified by the fact that  all steps in the proof depend only on the marginals $\{f^\ast(\cdot\mid x)\}_{x\in\mathcal{X}}$.
As a result, the regret bound remains valid when $f^\ast$ is replaced by $\tilde{f}^\ast$, which will be used in the proof of Lemma~\ref{lemma_indepelization}.
\end{remark}

\halmos

\section{Proof of Main Theorem \ref{theo_3}}\label{sec:ec:maintheorem}

Before presenting the main theorem, we introduce several preparatory results. The proofs of Lemmas~\ref{lemma_indepelization}, \ref{prop_TS-TS(D)},  and~\ref{thm:posterior_convergence} are deferred to later sections.

\begin{itemize}
\item \textbf{Central Lemma~\ref{lemma_central}}, which bounds the instantaneous regret in terms of the TVD between the predictive distribution and the true reward model. See the proof in Section \ref{sec_ec_proof_lemma1}.
\item \textbf{Lemma~\ref{prop_TS-TS(D)}}, which shows that the additional regret due to discretization is sublinear.
\item \textbf{Lemma~\ref{lemma_indepelization}}, which shows that the regret of TS(D) is bounded by that of TS(D)-I.
\item \textbf{Lemma~\ref{thm:posterior_convergence}}, which establishes the convergence rate of the $\infty$-GP model at a single input.
\end{itemize}

\proof{Proof:} By applying $\zeta^i$-greedy-TS policy at the $i$-th iteration, the expected sample size allocated to each input $x\in\mathcal{X}_n$ until the $n$-th iteration is lower bounded by (notice that we assume the total number of iterations $n$ is known in advance and that $\mathcal{X}_n$ is fixed prior to the run)
    \begin{equation}\label{ec_eq_egreedy_expolore}
        \mathbb{E}_{\mathcal{H}_n}(N^n_x)\geq \frac{\sum_{i=1}^n\zeta^i}{|\mathcal{X}_n|}.
    \end{equation}
Since $0<\lambda_1<1$ and $|\mathcal{X}_n|=C_2n^{\lambda_2}$, $\mathbb{E}(N^n_x)=\Omega(n^{1-\lambda_1-\lambda_2})$. Moreover, the $\zeta^n$ exploration term will incur additional regret:
\begin{equation}
    \begin{aligned}
        \Delta \mathcal{R}^{\zeta^n-\mathrm{greedy}}\leq \mu^\ast_{max}\sum_{i=1}^n\zeta^i=\mathcal{O}(n^{1-\lambda_1}).
    \end{aligned}
\end{equation}

(a) Next we discuss the $\delta_{\mathrm{TV}}$-convergence rate at each input $x\in\mathcal{X}_n$. If Assumption \ref{ass_double_choice} holds,  according to Lemma \ref{thm:posterior_convergence}, $\forall  x\in\mathcal{X}_n$, 
the convergence rate of $\infty$-GP 
model is given by \begin{equation}\label{eq_ec_tvd_lower}
\mathbb{E}_{\mathcal{H}_n}\big(\delta_{\mathrm{TV}}\left(f_n(\cdot|x), f^\ast(\cdot|x)\right)\big)\leq \mathbb{E}_{\mathcal{H}_n}\big\{\big(N^n_x\big)^{-\frac{\alpha}{2(1+\alpha)}}\big(\mathrm{log}\ N^n_x\big)^\psi\big\}=\mathcal{O}(n^{\frac{\alpha\lambda_1+\alpha\lambda_2-\alpha}{2(1+\alpha)}}(\mathrm{log}\ n)^\psi),
\end{equation}
where 
$$\psi>\frac{2(1+\frac{1}{\lambda}+\frac{1}{\alpha})+1}{2+\frac{2}{\alpha}}.$$

Although the bound in Lemma \ref{lemma_central} holds almost surely under $\mathbb{P}_{\mathcal{H}_n}$, we relax the term $\delta_{\mathrm{TV}}\left(f_n(\cdot|\mathcal{X}_n), f^\ast(\cdot|\mathcal{X}_n)\right)$ on the right-hand side of Lemma \ref{lemma_central} by replacing it with its expectation $\mathbb{E}_{\mathcal{H}_n} \left[ \delta_{\mathrm{TV}}\left(f_n(\cdot|\mathcal{X}_n), f^\ast(\cdot|\mathcal{X}_n)\right) \right]$. Without this relaxation, the right-hand side would be difficult to compute, as the contraction rate $\mathcal{E}_n$ depends on the specific sampling strategy $\pi(x)$. By adding the expectation, we are able to calculate its lower bound using Eq.\eqref{eq_ec_tvd_lower}.

Then with Lemma \ref{lemma_indepelization} and Eq.~\eqref{eq_ec_tvd_lower}, the instantaneous regret of using $\infty$-GP-TS(D) is given by 
\begin{equation}\label{eq_Theo3_1}
\begin{aligned}
 r^{\text{TS(D)-}\infty\text{-GP}}_{n+1}&\precsim  \mathbb{E}_{\mathcal{H}_n}[\delta_{\mathrm{TV}}\left(f_n(\cdot|\mathcal{X}_n), f^\ast(\cdot|\mathcal{X}_n)\right) ]\\&\precsim \mathbb{E}_{\mathcal{H}_n}[\delta_{\mathrm{TV}}\left(\tilde{f}_n(\cdot|\mathcal{X}_n), f^\ast(\cdot|\mathcal{X}_n)\right) ]\\&\precsim \mathbb{E}_{\mathcal{H}_n}[\sum_{i=1}^{|\mathcal{X}_n|} \delta_{\mathrm{TV}}\left(\tilde{f}_n(\cdot|x_i^n), {f}^\ast(\cdot|x_i^n)\right)]\\&\precsim \mathcal{O}(|\mathcal{X}_n|n^{\frac{\alpha\lambda_1+\alpha\lambda_2-\alpha}{2(1+\alpha)}}(\mathrm{log}\ n)^\psi)\\&= \mathcal{O}(n^{\frac{\alpha\lambda_1+(3\alpha+2)\lambda_2-\alpha}{2(1+\alpha)}}(\mathrm{log}\ n)^\psi).
\end{aligned}
\end{equation}

Then the cumulative regret
\begin{equation}\label{eq_proof_TSDIINFTYGP_REGRET}
    \begin{aligned}
        \mathcal{R}^{\text{TS(D)-}\infty\text{-GP}}_{n+1}&=r^{\text{TS(D)-}\infty\text{-GP}}_1+\sum_{i=2}^{n+1} r^{\text{TS(D)-}\infty\text{-GP}}_{i}\\
        &=r^{\text{TS(D)-}\infty\text{-GP}}_1+\mathcal{O}(\sum_{i=1}^ni^{\frac{\alpha\lambda_1+(3\alpha+2)\lambda_2-\alpha}{2(1+\alpha)}}(\mathrm{log}\ i)^\psi)\\
        &=\mathcal{O}(n^{\frac{\alpha\lambda_1+(3\alpha+2)\lambda_2+\alpha+2}{2(1+\alpha)}}(\mathrm{log}\ n)^\psi),
    \end{aligned}
\end{equation}
which is $o(n)$ if $\lambda_2<\frac{\alpha(1-\lambda_1)}{3\alpha+2}$.
The last ``$=$"  in (\ref{eq_proof_TSDIINFTYGP_REGRET}) is from Euler-Maclaurin formula.
Therefore, the regret is $\mathcal{O}(\max\{ n^{\frac{\alpha\lambda_1+(3\alpha+2)\lambda_2+\alpha+2}{2(1+\alpha)}}(\mathrm{log}\ n)^\psi,n^{1-\lambda_1}\})$, where the second term in the ``max" accounts for the regret caused by using $\zeta^n$-greedy policy.

(b) 
The proof for the regret bound in the classical GP model follows the same approach as in part (a). The only distinction is that, under the GP model, the posterior predictive distribution $f_n$ is a Gaussian distribution. According to the main theorem in \citeEC{1975gaussianerrors}, even if $f_n$ correctly identifies both the mean and variance of $f^\ast$ (which in practice is not achievable, so the actual regret will be even worse), the TVD between $f_n$ and $f^\ast$ can only guarantee a non-vanishing constant bound. This upper bound can be attained in some specific cases.
As a result, the regret of GP surrogate model can only guarantee an $\mathcal{O}(n)$ regret.

\halmos

\section{Preparatory Lemmas for Proving Theorem \ref{theo_3}}\label{sec:ec:prepare}

\subsection{Proof of Lemma \ref{lemma_indepelization}}
 Since we consider TS(D) policy in Lemma \ref{lemma_indepelization}, the decision set $\mathcal{X}=\{x_1,\cdots,x_{|\mathcal{X}|}\}$ is finite. 
In this subsection, first, we prove that, by sampling decoupling, the resulting 
decoupled $\infty$-GP model has lower convergence rate compared to the original model,  i.e., $\delta_{\mathrm{TV}}\left({f}_n(\cdot|\mathcal{X}), {f}^\ast(\cdot|\mathcal{X})\right)\leq\delta_{\mathrm{TV}}\left(\tilde{f}_n(\cdot|\mathcal{X}), {f}^\ast(\cdot|\mathcal{X})\right)$. Therefore, by the TVD-based regret bound derived in Lemma \ref{lemma_central}, the regret of the TS(D)-I policy is guaranteed to upper bound that of the TS(D) policy (i.e., the first inequality of Lemma \ref{lemma_indepelization}). Second, we demonstrate that the TVD-based regret bound of the TS(D)-I policy admits a decomposition of the form $\mu_{\mathrm{max}}^\ast\sum_{i=1}^{|\mathcal{X}|} \delta_{\mathrm{TV}}\left(\tilde{f}_n(\cdot|x_i), {f}^\ast(\cdot|x_i)\right)$, i.e., the second inequality of Lemma \ref{lemma_indepelization}.

Before proceeding, we recall a standard result from \citetEC{ghosal2000convergence}, which provides sufficient conditions under which the posterior distribution contracts at a given rate:

\begin{lemma}[\citeEC{ghosal2000convergence}]\label{lemma2}
Let $\delta$ be a metric on the space of densities (e.g., TVD). 
Suppose that there exist a positive constant $C$, a sequence of positive numbers $({\mathcal{E}}_n)_{n \geq 1}$ with $\mathcal{E}_n\to 0$ and $\lim_{n \to \infty} n {\mathcal{E}}_n^2 = \infty$, 
 and a sequence of measurable sets (sieves) $\mathcal{F}_n$ such that:
\begin{align}
    \log N(\mathcal{E}_n, \mathcal{F}_n, \delta) &\leq  n \mathcal{E}_n^2, \label{9} \\
    \Pi_{\mathcal{X}}(\mathcal{F}_n^c) &\leq  e^{-(C + 4) n {\mathcal{E}}_n^2}, \label{10} \\
    \Pi_{\mathcal{X}}\left( \mathcal{K}(f^\ast, {\mathcal{E}}_n) \right) &\geq  e^{-C n {\mathcal{E}}_n^2}, \label{11}
\end{align}
where the KL ball is defined as
$$\mathcal{K}(f^\ast, \mathcal{E}) := \left\{ f : 
\int f^\ast \log\left(f^\ast/{f}\right) < \mathcal{E}^2, \ 
\int f^\ast \log^2\left(f^\ast/{f}\right) < \mathcal{E}^2 \right\}.$$
Then the posterior distribution contracts at rate $\mathcal{E}_n$ around $f^\ast$, i.e.,
$$\Pi_\mathcal{X}\left( \left\{ f : \delta(f, f^\ast) > M \mathcal{E}_n \right\}  \middle|  \mathcal{H}_n \right) \to 0 $$  almost surely under $f^\ast,$
for some sufficiently large constant $M > 0$.
\end{lemma}

Specifically, the original $\infty$-GP model and its decoupling version are defined as follows. 
\paragraph{The Original $\infty$-GP Model.}
The conditional density is modeled as
$$
f^{(\mathrm{orig})}(y \mid x) = \sum_{l=1}^\infty w_l  \varphi_{x^\top \beta + \xi^{(l)}(x), \tau^2}(y),
$$
where $\{w_l,\beta,\tau\}$ are shared across all $x$, and each $\xi^{(l)}$ is a GP sample paths shared across $x$.

\paragraph{Decoupled $\infty$-GP Model.}
The conditional density is modeled as
$$
f^{(\mathrm{dec})}(y \mid x) = \sum_{l=1}^\infty w_l^{(x)}  \varphi_{x^\top \beta^{(x)} + \xi^{(l,x)}(x), \tau^{(x) 2}}(y),
$$
where for each $x \in \mathcal{X}$,
 $\{w_l^{(x)},\beta^{(x)}, \tau^{(x)}\}$ are independently drawn,
     $\{\xi^{(l,x)}\}_{l=1}^\infty$ are i.i.d.\ GP sample paths independently drawn for each $x$.

Same as in the proof of Theorem~\ref{theo_uni_converge}, the sieves for the original $\infty$-GP model and its decoupling version can be constructed as follows.
\paragraph{Sieve for the original model.}
All inputs \( x\in\mathcal{X}\) share the same atoms and weight vector.
Let
$$
\mathcal{F}_n^{(\mathrm{orig})}
=\Bigl\{
  f:\;
  f(y\mid x)=\sum_{l=1}^{\infty} w_l 
             \varphi_{x^{ \top}\beta+\xi^{(l)}(x), \tau^2}(y),
(\beta,\xi^{(l)},\tau)\in\Theta_n,\sum_{l>m_n}w_l\le\mathcal{E}, \forall x\in\mathcal{X}
\Bigr\}.
$$

\paragraph{Sieve for the fully decoupled model.}
The decoupled sieve is
\begin{equation}\label{sieve_decouple}
    \mathcal{F}_n^{(\mathrm{dec})}
=\Bigl\{
  f:
  f(y\mid x)=\sum_{l=1}^{\infty} w^{(x)}_l 
\varphi_{x^{\top}\beta^{(x)}+\xi^{(l,x)},\bigl(\tau^{(x)}\bigr)^2}(y),
\;
(\beta^{(x)},\xi^{(l,x)},\tau^{(x)})\in\Theta_n,\;
 \sum_{l>m_n}w^{(x)}_l\le\mathcal{E},
\forall x\in\mathcal{X}
\Bigr\},
\end{equation}
where each input \(x\) possesses its own independent $(\beta^{(x)},\xi^{(l,x)},\tau^{(x)})$ and
weights.

Suppose the sieve $\mathcal{F}_n^{(\mathrm{orig})}$ for the original $\infty$-GP model satisfies \eqref{9}, \eqref{10} and \eqref{11}.
We now verify the condition \eqref{9} for the sieve associated with the decoupled model. Suppose that 
$$
\log N(\mathcal{E}_n^{(\mathrm{orig})}, \mathcal{F}_n^{(\mathrm{orig})}, \delta) \leq  n (\mathcal{E}_n^{(\mathrm{orig})})^2
$$
holds. In the decoupled model, each $x \in \mathcal{X}$ has an independent set of parameters. Consequently, the sieve \eqref{sieve_decouple} for the decoupled model takes the form:
$
\mathcal{F}_n^{(\mathrm{dec})} = \prod_{x \in \mathcal{X}} \mathcal{F}^{(\mathrm{orig})}_{n,x},
$
where each $\mathcal{F}^{(\mathrm{orig})}_{n,x}$ is the sieve of the original model.
Hence, the total covering number satisfies
$$
\log N(\mathcal{E}^{(\mathrm{orig})}_n, \mathcal{F}_n^{(\mathrm{dec})}, \delta) \leq |\mathcal{X}| \cdot \log N(\mathcal{E}^{(\mathrm{orig})}_n, \mathcal{F}^{(\mathrm{orig})}_{n,x}, \delta)\leq |\mathcal{X}| \cdot  n (\mathcal{E}_n^{(\mathrm{orig})})^2.
$$
This indicates that the decoupled model requires a larger covering number to achieve the same accuracy $\mathcal{E}^{(\mathrm{orig})}_n$ as the original model. Consequently, in order for Condition~\eqref{9} to remain valid under the decoupled model, it must be satisfied with a $\mathcal{E}_n^{(\mathrm{dec})} \geq \mathcal{E}_n^{(\mathrm{orig})}$. 
Then,
notice that the posterior contraction rate is determined by the largest $\mathcal{E}_n$ for which all three conditions \eqref{9}, \eqref{10} and \eqref{11} in Lemma~\ref{lemma2} simultaneously hold. Therefore, for the decoupled model, conditions \eqref{9}, \eqref{10} and \eqref{11} in Lemma~\ref{lemma2} must be satisfied with a larger rate $\mathcal{E}_n^{(\mathrm{dec})} \geq \mathcal{E}_n^{(\mathrm{orig})}$ than the original model.
This in turn implies a slower posterior contraction rate for the decoupled model. Consequently, we have $$r_{n+1}\precsim\delta_{\mathrm{TV}}\left({f}_n(\cdot|\mathcal{X}), {f}^\ast(\cdot|\mathcal{X})\right)\precsim\delta_{\mathrm{TV}}\left(\tilde{f}_n(\cdot|\mathcal{X}), {f}^\ast(\cdot|\mathcal{X})\right).$$ This shows that the regret of TS(D) can be upper bounded by TS(D)-I, which corresponds to the first inequality in Lemma \ref{lemma_indepelization}.
As for the second inequality, note that $\tilde{f}_n$ is independent across different inputs in $\mathcal{X}$. By Remark~\ref{remark_EC}, we can replace $f^\ast$ with its decoupled version $\tilde{f}^\ast$ in the regret bound. Then, by a well-known result stating that the TVD between two independent product measures is upper bounded by the sum of the marginal TVDs, i.e., $$\delta_{\mathrm{TV}}\left(\tilde{f}_n(\cdot|\mathcal{X}_n), \tilde{f}^\ast(\cdot|\mathcal{X}_n)\right) \le \sum_{i=1}^{|\mathcal{X}_n|} \delta_{\mathrm{TV}}\left(\tilde{f}_n(\cdot|x_i^n), \tilde{f}^\ast(\cdot|x_i^n)\right)=\sum_{i=1}^{|\mathcal{X}_n|} \delta_{\mathrm{TV}}\left(\tilde{f}_n(\cdot|x_i^n), {f}^\ast(\cdot|x_i^n)\right),$$ the second inequality of Lemma \ref{lemma_indepelization} follows.
\halmos

\subsection{Lemma~\ref{prop_TS-TS(D)} and its proof}

\begin{lemma}\label{prop_TS-TS(D)}
  Suppose Assumption $\ref{ass_lip_conti}$ holds and $|\mathcal{X}_n|\sim\mathcal{O}(n^{\lambda_2})$, with $0<\lambda_2<1$. Then, the difference between the cumulative regrets of TS and TS(D) satisfies $$|\mathcal{R}^{\text{TS}}_n-\mathcal{R}^{\text{TS(D)}}_n|= \mathcal{O}(n^{1-\lambda_2/d})= o(n)$$ as $n\to\infty$, where $\mathcal{R}^{\text{TS}}_n$ and $\mathcal{R}^{\text{TS}}_n$ denote the cumulative regrets associated with TS and TS(D), respectively.
\end{lemma}

\proof{Proof of Lemma~\ref{prop_TS-TS(D)}}
We calculate the additional regret caused by discretization. Since $\mathcal{X}_n$ forms a uniform grid over $\mathcal{X}\in\mathbb{R}^d$ with cardinality $n^{\lambda_2}$, the maximum distance from any $x \in \mathcal{X}$ to its closest neighbor in $\mathcal{X}_n$ is 
\(
 \mathcal{O}(n^{-\lambda_2/d}).
\) Also, the distance between the optimal decision $x^\ast$ and its closest neighbor in $\mathcal{X}_n$ is also $ \mathcal{O}(n^{-\lambda_2/d})$. Therefore, recall that Assumption \ref{ass_lip_conti} holds, 
at each iteration $t$, the additional regret caused by discretization is bounded by $\Delta\mathcal{R}_n^{\mathrm{discrete}} \leq 2L' \cdot \mathcal{O}( n^{1 - \lambda_2 / d}) = \mathcal{O}(n^{1 - \lambda_2 / d})$, where $L'$ is the Lipschitz constant of $\mu^\ast$.

\halmos

\subsection{Lemma \ref{thm:posterior_convergence}: Single-Point Posterior Convergence Rate of $\infty$-GP}
When considering a fixed input $x_0 \in \mathcal{X}$, the $\infty$-GP model simplifies considerably: the baseline GP distribution at $x_0$ reduces to a univariate Gaussian distribution, and the complex structure of different inputs being realized on different latent surfaces disappears. As a result, the $\infty$-GP model degenerates into a ``Dirichlet process mixture of normal priors" considered in \citetEC{shen2013adaptive}. Therefore, the posterior contraction rate at a single point $x_0$ directly follows from Theorem 1 of \citetEC{shen2013adaptive}, by setting the ambient dimension $d = 1$ and $\kappa=2$. We restate the result in Lemma \ref{thm:posterior_convergence}.
Other classical posterior contraction results (see \citeEC{ghosal2001entropies} and \citeEC{ghosal2007posterior}) also provide rates for Dirichlet process mixtures. However, these results generally require stronger tail assumptions than those imposed in \citetEC{shen2013adaptive}.

\begin{lemma}[Theorem 1 in \citeEC{shen2013adaptive}]\label{thm:posterior_convergence}
Let \( f^\ast(y \mid x_{0}) \) be the true reward distribution at input $x_0$ satisfying Assumption~\ref{ass_double_choice}, $\Pi_{x_0}$ be the prior induced by $\infty$-GP, and let \( \Pi_{x_0}(\cdot \mid \mathcal{H}_n) \) denote the posterior distribution based on \( n \) i.i.d. evaluations at input $x_0$. Then there exists a sequence \( \mathcal{E}_n = n^{-\alpha/(2\alpha + 2)}(\log n)^t \) for some constant \( t > 0 \), such that for any \( M > 0 \),
$$
\lim_{n\to\infty}\Pi_{x_0}\left( \left\{ f : \delta_{\mathrm{TV}}(f,f^\ast) > M \mathcal{E}_n \right\}  \middle|  {\mathcal{H}_n}\right) \to 0 \quad \text{almost surely under } f^\ast.
$$
\end{lemma}

\section{Proof of Proposition~\ref{prop_sparse} and Corollary \ref{theo_truncate}} \label{sec:ec:proof_prop2+co1}

\proof{Proof of Proposition~\ref{prop_sparse}}
Under the stick-breaking construction of the Dirichlet process with concentration parameter \(\nu\), the surface assignment process is equivalent to a Pólya urn scheme. The probability that the \(i\)-th sample introduces a new surface is \( \frac{\nu}{\nu + i - 1} \). Therefore, the expected number of distinct surfaces among \(n\) samples is
$$
\mathbb{E}(K_n) = \sum_{i=1}^n \frac{\nu}{\nu + i - 1}.
$$
This sum admits the asymptotic approximation:
$$
\mathbb{E}(K_n) \sim \nu \log\left( \frac{n}{\nu} \right) \quad \text{as } n \to \infty,
$$
which completes the proof.

\proof{Proof of Corollary \ref{theo_truncate}}
In this theorem, we consider the TS(D) policy, with the finite decision set being $\mathcal{X}$. Recall that, in $\infty$-GP model, each reward $y(x_i)$ is generated by the hierarchical model:
$$
y(x_i) = \beta^\top x_i + \xi(x_i) + \epsilon_i, \quad \xi(x) \sim G_x = \sum_{l=1}^{\infty} w_l \delta_{\xi^{(l)}(x)}, \quad \epsilon_i \sim \mathcal{N}(0,\tau^2).
$$

Let $f_n(y \mid \mathcal{X})$ denote the joint predictive distribution of the full $\infty$-GP model, induced by the infinite mixing measure $G_x = \sum_{l=1}^\infty w_l \delta_{\xi^{(l)}(x)}$.
Let $f_n^{(L)}(y \mid \mathcal{X})$ denote the corresponding joint predictive distribution under the truncated model with finite mixing measure $G_x^{(L)} = \sum_{l=1}^{L} w_l \delta_{\xi^{(l)}(x)}$.
We denote by $Q^{(\infty)}$ and $Q^{(L)}$ the induced joint distributions over the surface assignments $z = (z_1, \dots, z_{|\mathcal{X}|}) \in [\infty]^{|\mathcal{X}|}$ and $[L]^{|\mathcal{X}|}$, under the full and truncated models, respectively.

Then,
\begin{equation}\label{eq_ec_9}
    \begin{aligned}
        \int \left| f_n^{(L)}(y \mid \mathcal{X}) - f_n(y \mid \mathcal{X}) \right| dy
&= \int \left| \int \prod_{i=1}^{|\mathcal{X}|} f(y(x_i) \mid \xi^{(z_i)}(x_i)) \left( dQ^{(L)}(z) - dQ^{(\infty)}(z) \right) \right| dy  \\
&\leq  \int \int \prod_{i=1}^{|\mathcal{X}|} f(y(x_i) \mid \xi^{(z_i)}(x_i))dy \left|\left( dQ^{(L)}(z) - dQ^{(\infty)}(z) \right) \right| \\
&= 2 \cdot \delta_{\mathrm{TV}} \left( Q^{(L)}, Q^{(\infty)} \right).
    \end{aligned}
\end{equation}

That is, the $\ell^1$ distance (equal to twice the TVD) between joint reward distributions is bounded by \( \delta_{\mathrm{TV}}(Q^{(L)}, Q^{(\infty)}) \). According to \citeEC{ishwaran2002approximate},
$$
\delta_{\mathrm{TV}}(Q^{(L)}, Q^{(\infty)})  \le 2 \cdot \mathbb{E}\left[ 1-\big(\sum_{l=1}^{L-1} w_l\big)^{|\mathcal{X}|} \right]\precsim2|\mathcal{X}|\exp(-\frac{L-1}{\nu}).
$$

Therefore, according to Eq.~\eqref{eq_ec_9}, $\delta_{\mathrm{TV}}(f_n^{(L)}(y \mid \mathcal{X}) , f_n(y \mid \mathcal{X}))\precsim |\mathcal{X}|\exp(-\frac{L-1}{\nu})$. Then, with the central Lemma~\ref{lemma_central} and the triangle inequality of TVD, the instantaneous regret of truncated $\infty$-GP-TS is given by
\begin{equation}
    \begin{aligned}
        {r}_{n+1}^{L-\mathrm{GP}-\mathrm{TS(D)}}&\precsim \mu^\ast_{\max}  \delta_{\mathrm{TV}}\big(f^{(L)}_n(\cdot \mid \mathcal{X}), f^\ast(\cdot \mid \mathcal{X})\big)\\
        &\precsim \mu^\ast_{\max}  \delta_{\mathrm{TV}}\big(f^{(L)}_n(\cdot \mid \mathcal{X}), f_n(\cdot \mid \mathcal{X})\big)+\mu^\ast_{\max} \delta_{\mathrm{TV}}\big(f_n(\cdot \mid \mathcal{X}), f^\ast(\cdot \mid \mathcal{X})\big)\\
        &\precsim \mu^\ast_{\max}|\mathcal{X}|\exp(-\frac{L-1}{\nu})+\mu^\ast_{\max} \delta_{\mathrm{TV}}\big(f_n(\cdot \mid \mathcal{X}), f^\ast(\cdot \mid \mathcal{X})\big)
    \end{aligned}
\end{equation}
Then we have
$$
\mathcal{R}^{L-\mathrm{GP}-\mathrm{TS(D)}}_{n}\precsim \mu^\ast_{\max} n |\mathcal{X}|\exp(-\frac{L-1}{\nu})+\mathcal{R}_n^{\infty-\mathrm{GP}-\mathrm{TS(D)}}\precsim {\mu^\ast_{\max}n^2 \exp\big( -\frac{L-1}{\nu} \big) }+\mathcal{R}_{n}^{\infty-\mathrm{GP}-\mathrm{TS(D)}},
$$
where the last inequality holds because in the setting of Theorem \ref{theo_3},
 $|\mathcal{X}|=\mathcal{O}(n^{\lambda_2})$ and $\lambda_2<1$.
\halmos

\section{Justification of the Required Assumptions}\label{app:verify_assumptions}

In this section, we discuss the assumptions imposed in Theorem~\ref{theo_uni_converge} and Theorem~\ref{theo_3}.
For Theorem~\ref{theo_uni_converge}, Remark~\ref{remark_1} provides a broad family of distributions that satisfy its required conditions, with a detailed justification given in Remark 5.1 of \citeEC{pati2013posterior}.
Therefore, our focus here is on verifying that the assumptions required by Theorem~\ref{theo_3} are indeed mild, and are satisfied by the broad class of reward distributions described in Remark~\ref{remark_3}.
Recall that for the deterministic component $\mu^\ast$, only Lipschitz continuity is required. 
The remaining assumptions---Assumption~\ref{ass_tail0} (tail behavior) and Assumption~\ref{ass_double_choice} (local smoothness and moment conditions)---concern only the stochastic noise term $\epsilon^\ast(x)$.

We show that the following conditional distributions $f^\ast(y\mid x)$ satisfy both assumptions:
\begin{itemize}
    \item Gaussian distribution,
    \item Gamma distribution,
    \item Weibull distribution with shape parameter $>1$,
    \item Exponential distribution,
    \item and any finite mixture of the above families.
\end{itemize}

The verification proceeds in two steps:
(1) tail conditions required by Assumption~\ref{ass_tail0}, and
(2) local smoothness and moment conditions required by Assumption~\ref{ass_double_choice}.

\subsection{Verification of the Tail Condition (Assumption~\ref{ass_tail0})}

This condition requires the density to satisfy a ``sub-Weibull" \citepEC{subweibull} tail bound:
\[
    f^\ast(y\mid x)\;\le\; a_1 \exp(-a_2 |y|^\gamma)
    \qquad \text{for some } \gamma>0.
\]
This condition is extremely mild and is satisfied by all distributions listed above:
\begin{itemize}
    \item Gaussian: $\gamma = 2$,
    \item Gamma: $\gamma = 1$,
    \item Exponential: $\gamma = 1$,
    \item Weibull: $\gamma >1$,
    \item Mixtures of the above families: dominated by the slowest-decaying component, hence also sub-Weibull.
\end{itemize}

Thus Assumption~\ref{ass_tail0} holds uniformly across the entire class.

\subsection{Verification of the Smoothness Condition (Assumption~\ref{ass_double_choice})}

Assumption~\ref{ass_double_choice} requires: first, local Hölder smoothness of $f^\ast(\cdot\mid x)$ up to order $\lfloor\alpha\rfloor$; second, finite moments of the derivative ratio $ \left|\frac{f^{\ast(k)}(Y\mid x)}{f^\ast(Y\mid x)}\right|$ and $ \left|\frac{L(Y)}{f^\ast(Y\mid x)}\right|.$
Below we verify these requirements explicitly for the aforementioned distribution families.


\textbf{Gaussian distribution:}
Let $f^\ast(y\mid x)
    = \frac{1}{\sqrt{2\pi}\sigma}
    \exp\!\left(-\frac{(y-\mu)^2}{2\sigma^2}\right)$.
The Gaussian density is infinitely differentiable on $\mathbb{R}$, and for each integer $k\ge1$,
$$
f^{\ast (k)}(y\mid x)
=  P_k(y)\, f^\ast(y\mid x),
$$
where $P_{k}(y)$ is a polynomial of degree $k$ whose coefficients depend on $\mu,\sigma$.  Therefore,
$$
\frac{\bigl|f^{\ast (k)}(y\mid x)\bigr|}{f^\ast(y\mid x)}
= |P_{k}(y)|,\quad \forall k\ge1,
$$
and
$$
\mathbb{E}_{Y\sim f^\ast}\Bigl[\Bigl|\frac{f^{\ast (k)}(Y\mid x)}{f^\ast(Y\mid x)}\Bigr|^{\frac{2\alpha+\eta}{k}}\Bigr]
= \mathbb{E}_{Y\sim\mathcal{N}(\mu,\sigma^2)}\bigl[ |P_{k}(Y)|^{\frac{2\alpha+\eta}{k}}\bigr]
<\infty,
$$
because $P_{k}(Y)$ is a polynomial and moments of any order are finite under a Gaussian distribution. 

Moreover, for any $\alpha>0$, one can choose
\(
L(y)\;=\;C(1+|y|^{\alpha+1})\exp(-\frac{(y-\mu)^2}{2\sigma^2})
\) for some $C$
so that $$f^{(\lfloor \alpha \rfloor)}(y+h)-f^{(\lfloor \alpha \rfloor)}(y)\leq L(y)e^{\lambda_0h^2}|h|^{\alpha-\lfloor \alpha \rfloor},\ y,h\in\mathbb{R},$$ for some $\lambda_0$. This follows directly by applying the mean value theorem.
Then, it is easy to verify that
$$\mathbb{E}_{Y\sim f^\ast}\Bigl[\Bigl|\frac{L(Y)}{f^\ast(Y\mid x)}\Bigr|^{\frac{2\alpha+\eta}{\alpha}}\Bigr]
<\infty$$
since this ratio is polynomial in $Y$.
Thus Gaussian distributions satisfy Assumption~\ref{ass_double_choice} for any $\alpha>0$.

\textbf{Gamma distribution}:
Let $f^\ast(y\mid x)
= \frac{1}{\Gamma(r)\,\theta^r}\, y^{r-1} 
\exp\!\Bigl(-\frac{y}{\theta}\Bigr),$ with $r>0$, $\theta>0$.
Fix $r>0$. For each integer $k\ge1$, its $k$-th derivative can be written in the form
\[
    f^{\ast (k)}(y\mid x)
    = y^{r-1-k}\,P_k(y)\,\exp\!\Bigl(-\frac{y}{\theta}\Bigr),
\]
where $P_k$ is a polynomial term. Hence
\[
    \frac{\bigl|f^{\ast (k)}(y\mid x)\bigr|}{f^\ast(y\mid x)}
    = y^{-k}\,\bigl|P_k(y)\bigr|,
    \qquad k=1,2,\dots,\lfloor\alpha\rfloor.
\]

Then
\[
\mathbb{E}_{Y\sim f^\ast(\cdot\mid x)}\!\left[
\left|\frac{f^{\ast (k)}(Y\mid x)}{f^\ast(Y\mid x)}\right|^{\frac{2\alpha+\eta}{k}}
\right]
=
\mathbb{E}_{Y\sim f^\ast(\cdot\mid x)}\bigl[\,|Y^{-k}P_k(Y)|^{\frac{2\alpha+\eta}{k}}\,\bigr]<\infty
\]
This is because, near $0$, $P_k(Y)$ behaves at most polynomially, whereas the Gamma density behaves like $y^{r-1}$.  
Hence the integrand near $0$ is dominated by $ y^{r-1 - (2\alpha+\eta)}.$
This is integrable at $0$ by choosing $\alpha$ and $\eta$ sufficiently small so that $2\alpha+\eta<r$.  
Integrability in the tail is ensured by the exponential factor $e^{-y/\theta}$.

Moreover, it is straightforward to verify that 
$f^\ast(\cdot\mid x)\in\mathcal{C}^{\alpha,L,\lambda_0}$ 
for some $\alpha>0$, $\lambda_0\ge0$, and $L(y)=C\,(1+|y|^{m})\,f^\ast(y\mid x)$, where $C>0$ is a constant.
Therefore, the moment condition
\[
\mathbb{E}_{Y\sim f^\ast(\cdot\mid x)}\left[
\left(\frac{L(Y)}{f^\ast(Y\mid x)}\right)^{\frac{2\alpha+\eta}{\alpha}}
\right]<\infty
\]
follows immediately.  
Therefore, the Gamma family satisfies Assumption~\ref{ass_double_choice}.



\textbf{ Exponential, Weibull, and their mixtures:} For the Exponential distributions, Weibull distributions (shape parameter $>1$), and any finite mixture of these families, the verification proceeds analogously.

For the Exponential density 
\(f^\ast(y\mid x)=\lambda e^{-\lambda y}\), the ratio \(|f^{\ast(k)}(y\mid x)|/f^\ast(y\mid x)\) is a constant.  
Hence all required moments are finite, and the Assumption~\ref{ass_double_choice}
are trivially satisfied.

For Weibull densities with shape parameter \(r>1\), $f^\ast(y\mid x)=\frac{r}{\lambda}
\left(\frac{y}{\lambda}\right)^{r-1}
\exp\!\left[-\left(\tfrac{y}{\lambda}\right)^r\right],$
the $k$-th derivative takes the form $f^{\ast(k)}(y\mid x)
= y^{\,r-1-k}\,R_k(y)\,
\exp\!\left[-\left(\tfrac{y}{\lambda}\right)^r\right],$
where \(R_k(y)\) is a polynomial.  
Thus the moment condition
\[
\mathbb{E}\!\left[
\left(\frac{|f^{\ast(k)}(Y\mid x)|}{f^\ast(Y\mid x)}\right)^{\frac{2\alpha+\eta}{k}}
\right]
<\infty
\]
holds whenever \(2\alpha+\eta < r\), exactly as in the Gamma distribution case.

Moreover, both the Exponential and Weibull distributions belong to $\mathcal{C}^{\alpha,L,\lambda_0}$ 
for some $\alpha>0$, $\lambda_0\ge0$, and
$L(y)=C\,(1+|y|^{m})\,f^\ast(y\mid x)$, with suitable constants $C,m>0$ depending only on $k$ and the distributional parameters. 
Therefore,
\[
\mathbb{E}\!\left[
\left(\frac{L(Y)}{f^\ast(Y\mid x)}\right)^{\frac{2\alpha+\eta}{\alpha}}
\right]
<\infty.
\]
In conclusion, Assumption~\ref{ass_double_choice} holds for the Exponential and Weibull distributions.

\textbf{Finite mixtures.}
Finally, any finite mixture
\[
f^\ast(y\mid x)=\sum_{j=1}^J w_j(x)\,\psi_j(y)
\]
of the above distributions also satisfies Assumption~\ref{ass_double_choice}, because both the required smoothness properties and the exponential tail bounds are preserved under finite convex combinations.

\bibliographystyleEC{informs2014} 
\bibliographyEC{ref_appendix}
\end{document}